\documentclass{article}

\usepackage{microtype}
\usepackage{graphicx}
\usepackage{subfigure}
\usepackage{booktabs} %

\usepackage{hyperref}

\newcommand{\alglinelabel}{%
  \addtocounter{ALC@line}{-1}%
  \refstepcounter{ALC@line}%
  \label%
}
\usepackage[accepted]{icml2025}

\usepackage{amsmath}
\usepackage{amssymb}
\usepackage{mathtools}
\usepackage{amsthm}
\usepackage{colortbl}
\usepackage{mmll}
\usepackage{tikz}
\usetikzlibrary{arrows,automata,positioning}
\usepackage{enumitem}
\usepackage[colorinlistoftodos, textsize=tiny, textwidth=2.2cm, backgroundcolor=blue!10, linecolor=magenta, bordercolor=magenta]{todonotes}

\def \algname {R2PVI}
\definecolor{lightblue}{rgb}{0.88,1,1}
\allowdisplaybreaks

\icmltitlerunning{Robust Offline Reinforcement Learning with Linearly Structured $f$-Divergence Regularization}

\begin{document}

\twocolumn[
\icmltitle{Robust Offline Reinforcement Learning with Linearly Structured \\$f$-Divergence Regularization}

\icmlsetsymbol{equal}{*}

\begin{icmlauthorlist}
\icmlauthor{Cheng Tang}{equal,sch}
\icmlauthor{Zhishuai Liu}{equal,sch2}
\icmlauthor{Pan Xu}{sch2}
\end{icmlauthorlist}

\icmlaffiliation{sch}{
University of Illinois Urbana-Champaign, work was done when Cheng was in Tsinghua University
}
\icmlaffiliation{sch2}{Duke University}

\icmlcorrespondingauthor{Pan Xu}{pan.xu@duke.edu}

\icmlkeywords{Machine Learning, ICML}

\vskip 0.3in]

\printAffiliationsAndNotice{\icmlEqualContribution} %

\begin{abstract}

The Robust Regularized Markov Decision Process (RRMDP) is proposed to learn policies robust to dynamics shifts by adding regularization to the transition dynamics in the value function. Existing methods mostly use unstructured regularization, potentially leading to conservative policies under unrealistic transitions. To address this limitation, we propose a novel framework, the $d$-rectangular linear RRMDP ($d$-RRMDP), which introduces latent structures into both transition kernels and regularization. We focus on offline reinforcement learning, where an agent learns policies from a precollected dataset in the nominal environment. We develop the Robust Regularized Pessimistic Value Iteration (R2PVI) algorithm that employs linear function approximation for robust policy learning in $d$-RRMDPs with $f$-divergence based regularization terms on transition kernels. We provide instance-dependent upper bounds on the suboptimality gap of R2PVI policies, demonstrating that these bounds are influenced by how well the dataset covers state-action spaces visited by the optimal robust policy under robustly admissible transitions. We establish information-theoretic lower bounds to verify that our algorithm is near-optimal. Finally, numerical experiments validate that R2PVI learns robust policies and exhibits superior computational efficiency compared to baseline methods.
\end{abstract}

\section{Introduction}
Offline reinforcement learning (RL) \citep{levine2020offline} facilitates policy learning from fixed datasets, eliminating the need for direct interaction with the environment. 
When the policy deployment environment differs from the one where the dataset was collected, robust policies that remain effective under the environment shift are required \citep{garcia2015comprehensive, packer2018assessing, zhang2020robust, wang2024return,guo2024off}. A widely adopted framework for learning such policies is the distributionally robust Markov decision process (DRMDP) \citep{iyengar2005robust, nilim2005robust}, which models dynamics changes as an uncertainty set around the nominal transition kernel. In this setup, an agent seeks policies performing well even in the worst-case environment within the uncertainty set. The most common design of uncertainty sets is the $(s,a)$-rectangularity \citep{iyengar2005robust, nilim2005robust}, which independently models uncertainty for each state-action pair. Although mathematically elegant, the $(s,a)$-rectangularity can result in overly conservative policies, especially when the state and action spaces are large. To address this issue, \citet{goyal2023robust} introduce the $r$-rectangular uncertainty set, which parameterizes transition kernels using latent factors. 
This concept has since been incorporated into $d$-rectangular linear DRMDPs ($d$-DRMDPs, \citet{ma2022distributionally}), extending its applicability to robust decision-making with linear function approximation.
Building on $d$-DRMDPs, recent works \citep{blanchet2024double, wang2024sample, liu2024minimax} propose provably efficient algorithms  that leverage function approximation for robust policy learning.

However, the $d$-DRMDP framework has several problems that remain unaddressed, which we summarize as follows.
{\it Theoretical Gaps:} Current understanding of $d$-DRMDPs is largely restricted to uncertainty sets defined by the Total Variation (TV) divergence \citep{liu2024distributionally,liu2024minimax}. For uncertainty sets defined by the Kullback-Leibler (KL) divergence, prior works \citep{ma2022distributionally, blanchet2024double} rely on additional regularity assumptions regarding the KL dual variable, which is hard to validate in practice. Moreover, the $\chi^2$-divergence defined uncertainty set has demonstrated effectiveness in certain empirical applications \citep{panaganti2022sample, xu2023improved} and has also been analyzed under the $(s,a)$-rectangularity \citep{shi2024curious}. Yet there are no theoretical results or efficient algorithms for $d$-DRMDPs.
{\it Practical challenges:} Existing practical algorithms \citep{ma2022distributionally, liu2024minimax, wang2024sample} depend on a dual optimization oracle (see Remark 4.2 in \citet{liu2024distributionally}) to estimate the robust value function. The computation complexity of these methods is proportional to the feature dimension $d$ and the planning horizon $H$. While heuristic methods like the Nelder-Mead algorithm \citep{nelder1965simplex} can approximate the oracle, they become computationally expensive when dealing with high-dimensional features (large $d$) and extended planning horizons (large $H$), which are common in real-world applications. These limitations raise an important question:
\begin{center}
\emph{Can we design efficient offline robust RL algorithms \\
using general $f$-divergence\footnote{The general $f$-divergence includes widely studied divergences such as Total Variation, Kullback-Leibler, and $\chi^2$ divergences.} uncertainty models \\
with linearly structured transitions?}
\end{center}
In this work, we provide a positive answer to this question. Inspired by the robust regularized MDP (RRMDP) framework with the $(s,a)$-rectangularity condition \citep{yang2023robust, zhang2020robust, panaganti2024model}, where the uncertainty set constraint in DRMDP is replaced by a regularization penalty term measuring the divergence between the nominal and perturbed dynamics, we propose the $d$-rectangular linear RRMDP ($d$-RRMDP) framework. Specifically, $d$-RRMDP replaces the $d$-rectangular uncertainty set in $d$-DRMDPs with a carefully designed penalty term that preserves the linear structure. The motivations are two folds: (1) it has been shown by \citet{yang2023robust} that the robust value function under the RRMDP is equivalent to that under the DRMDP with $(s,a)$-rectangularity as long as the regularizer is properly chosen; (2) removing the uncertainty set constraint simplifies the dual problem for certain divergences \citep{zhang2024soft}, potentially improving computational efficiency and facilitating theoretical analysis. 
We summarize our contributions as follows: 
\begin{itemize}[leftmargin=*]
    \item  We establish that key dynamic programming principles, including the robust Bellman equation and the existence of deterministic optimal robust policies, hold under the $d$-RRMDP framework. Additionally, we derive dual formulations of robust Q-functions with TV, KL and $\chi^2$ divergences defined regularization, highlighting their linear structures.
    \item We propose a computationally tractable meta-algorithm, Robust Regularized Pessimistic Value Iteration (R2PVI), for offline $d$-RRMDPs with general $f$-divergence regularization. For TV, KL, and $\chi^2$ divergences, we provide instance-dependent upper bounds on the suboptimality gap of policies learned by R2PVI, in a general form of $\beta \sup_{P \in \mathcal{U}^\lambda(P^0)} \sum_{h=1}^H \mathbb{E}^{\pi^\star, P} \big[\sum_{i=1}^d \| \phi_i(s, a) \mathbf{1}_i \|_{\Lambda_h^{-1}} \mid s_1 = s \big]$, 
    where $d$ is the feature dimension, $H$ is the horizon length, \(\bphi(s, a)\) is the feature mapping, $\lambda$ is the regularization parameter, and $\beta$ is a problem-dependent parameter whose specific form depends on the choice of the divergence (see \Cref{sec: Instance-Dependent Bound of R2PVI} for details). The set $\mathcal{U}^\lambda(P^0)$ is derived from our theoretical analysis, and it does not represent an uncertainty set in the conventional DRMDP framework. We further construct an information-theoretic lower bound, demonstrating that this instance-dependent uncertainty function is intrinsic.
    \item We conduct experiments in simulated environments, including a linear MDP setting \citep{liu2024distributionally} and the American Put Option environment \citep{tamar2014scaling}. Our findings show that:
	1.	The $d$-RRMDP framework yields equivalent robust policies as $d$-DRMDP with appropriately chosen regularization parameters.
	2.	R2PVI significantly improves algorithms designed for $d$-DRMDPs in terms of the computation complexity, and is comparable to algorithms designed for standard linear MDPs.
\end{itemize}

\textbf{Notations.}
In this paper, we denote $\Delta(\cS)$ as the probability distribution in the state space $\cS$. For any $H \in \NN$, $[H]$ represents the set $\{ 1,2,3,\cdots,H \}$. For a vector $\bv\in\RR^d$, we denote $v_i$ as the $i$-th element. For any function $V:\cS\rightarrow[0,H]$, we denote $V_{\min}= \min_{s\in\cS}V(s)$ and $V_{\max} = \max_{s\in\cS}V(s)$. For any distribution $\mu\in\Delta(\cS)$, we denote $\Var_{s\sim\mu}V(s)$ as the variance of the random variable $V(s)$ under $\mu$. For any two probability measures $P$ and $Q$ satisfying that $P$ is absolute continuous with respect to $Q$, the $f$-divergence is defined as $D_{f}(P \| Q)=\int_{\mathcal{S}}f(P(s)/Q(s))Q(s) \text{d} s$, where $f$ is a convex function on $\RR$ and differentiable on $\RR_+$ satisfying $f(1)=0$ and $f(t)=+\infty, \forall t<0$.
The Total Variation (TV) divergence, Kullback-Leibler (KL) divergence and Chi-Square ($\chi^2$) divergence between $P$ and $Q$ are defined by 
$f(x) = |x-1|/2, f(x) =x \log x, f(x) =(x-1)^2 $, respectively. Given a scalar $\alpha$, we denote $[V(s)]_\alpha = \min\{V(s), \alpha\}$.  Given an interval $I$, we define $[V(s)]_I$ as the result of clipping $V(s)$ to lie within the interval $I$. We denote $\mathbf{I}$ as the identity matrix and $\mathbf{1}_i\in\RR^d$ as the one-hot vector with the $i$-th element equals to one.

\section{Related Work}
\textbf{Distributionally Robust MDPs.} The seminal works of \citet{satia1973markovian,iyengar2005robust, nilim2005robust} proposed the framework of DRMDP. There are several lines of works studying DRMDPs under different settings. \citet{zhou2021finite,panaganti2022robust,panaganti2024bridging,shi2024distributionally, liu2025linear} studied the offline DRMDP assuming access to an offline dataset and provided sample complexity bounds under the coverage assumption on the offline dataset. \citet{liu2024distributionally,liu2024upper, lu2024distributionally} studied the online DRMDP where an agent learns robust policies by actively interacting with the nominal environment. \citet{blanchet2024double, panaganti2022robust} studied the DRMDP with general function approximation, they focused on the offline setting with the $(s,a)$-rectangularity assumption. \citet{ma2022distributionally, liu2024minimax, wang2024sample} studied the offline $d$-DRMDP, they proposed provably efficient and computationally tractable algorithms and provided sample complexity bounds under different kinds of coverage assumptions on the offline dataset.

\textbf{RRMDPs.} 
The work of \citet{yang2023robust, zhang2024soft} proposed the RRMDP, which can be regarded as a generalization of the DRMDP by substituting the uncertainty set constraint in DRMDP with the regularization term defined as the divergence between the perturbed model and the nominal model. In particular, \citet{yang2023robust} studied the tabular RRMDP and proposed a model-free algorithms assuming access to a simulator. \citet{zhang2024soft} studied the offline RRMDP, they established connections between RRMDPs with risk sensitive MDPs, and derived the policy gradient principle. Moreover, they studied general function approximation and proposed a computationally efficient algorithm, RFZI, for RRMDPs with KL-divergence defined regularization terms. \citet{zhang2024soft} firstly discovered that the duality of the robust value function has a closed expression under the KL-divergence. 
\citet{panaganti2024model} studied the offline RRMDP with regularization terms defined by the general $f$-divergence. They studied general function approximation and provided sample complexity results. They further proposed a hybrid algorithm, which learns robust policies with both historical data and interactive data collection, for RRMDPs with TV-divergence defined regularization term. Existing works focus on the $(s,a)$-rectangularity uncertainty regularization, which is different from ours.

\section{Problem Formulation}
\label{sec:Problem formulation}
In this section, we provide  preliminaries for RRMDPs. 

\paragraph{Markov decision process (MDP).} We first introduce the concept of MDPs, which is the basis of our settings. Specifically, we denote $\text{MDP}(\cS, \mathcal{A}, H,  P^0, r)$ as a finite horizon MDP, where $\mathcal{S}$ is the state space, $\mathcal{A}$ is the action space, $H$ is the  horizon length, $P^0 = \{P^0_h \}_{h = 1}^H$ are nominal transitional kernels, and the $r(s,a) \in [0,1]$ is the deterministic reward function assumed to be known in advance.
For any policy $\pi$, the value function and Q-function at time step $h$ are defined as $V_h^{\pi}(s) = \EE^{P^0} \big[\sum_{t = h}^H r_t(s_t, a_t)| s_h = s, \pi \big]$, and $Q^{\pi}_h(s,a) = \EE^{P^0} \big[ \sum_{t = h}^H r_t(s_t, a_t)| s_h = s, a_h=a, \pi\big]$.

\paragraph{Robust regularized MDP (RRMDP)} We define a finite horizon RRMDP as $\text{RRMDP}(\cS, \mathcal{A}, H,  P^0, r, \lambda, D, \mathcal{F})$, where $\lambda$ is the regularizer, $D$ is the probability divergence metric, and $\cF$ is the feasible set of all perturbed transition kernels. For any policy $\pi$, the robust regularized value function is defined as $V_h^{\pi, \lambda}(s) =  \inf_{P \in \mathcal{F}} \EE^{P} \big[\sum_{t = h}^H \big[r_t(s_t, a_t) + \lambda D(P_t(\cdot|s_t,a_t)\|P_t^0(\cdot|s_t,a_t))\big] \big| s_h = s, \pi \big]$ and robust Q-function as $Q^{\pi, \lambda}_h(s,a) =  \inf_{P \in \mathcal{F}} \EE^{P} \big[ \sum_{t = h}^H \big[r_t(s_t, a_t) +\lambda D(P_t(\cdot|s_t,a_t)\|P_t^0(\cdot|s_t,a_t))\big]  
 \big| s_h = s, a_h=a, \pi \big]$.

The RRMDP framework has been referred to by different names in the literature, including the penalized robust MDP \citep{yang2023robust}, the soft robust MDP \citep{zhang2024soft}, and the robust $\phi$-regularized MDP \citep{panaganti2024model}. For consistency, we adopt the term RRMDP in this work. In RRMDPs, the perturbed transition kernel class $\cF$ typically encompasses all possible kernels. However, for environments with large state-action spaces, $\cF$ may be overly broad, including transitions that are unrealistic or irrelevant. To address this, we introduce latent structures on transition kernels and design regularization terms that penalize changes in the latent structure, sharing similar ideas with the design of $r$-rectangular \citep{goyal2023robust}   and $d$-rectangular \citep{ma2022distributionally} uncertainty sets.

\paragraph{The $d$-rectangular linear RRMDP ($d$-RRMDP).} In this paper, we propose the novel $d$-RRMDP, which admits a linear structure of the feasible set and reward function. Specifically, a $d$-RRMDP is a RRMDP where the nominal environment $P^0$ is a special case of linear MDP with a simplex feature space \citep[Example 2.2]{jin2020provably}, and the feasible set $\cF$ involves kernels defined based on the linear structure of the nominal transition kernel. We make the following assumption on reward functions and transition kernels:
\begin{assumption}[\citet{jin2020provably}]
\label{assumption:linear MDP} Given a known state-action feature mapping $\bphi: \cS \times \cA \rightarrow \RR^d$ satisfying $\sum_{i=1}^d\phi_i(s,a) = 1, \phi_i(s,a) \geq 0$, we assume the reward function $\{r_h\}_{h=1}^H$ and the nominal transition kernels $\{P_h^0\}_{h=1}^H$ admit linear structures. Specifically, for all $(h,s,a)\in[H]\times\cS\times\cA$, we have $r_h(s,a) = \langle \bphi(s,a), \btheta_h \rangle$, and $P_h^0(\cdot |s,a) = \langle \bphi(s,a), \bmu_h^0(\cdot) \rangle$, 
where $\{ \btheta_h \}_{h=1}^H$ are known vectors with bounded norm $\| \btheta_h \|_2 \leq \sqrt{d}$ and $\bmu^0_h = (\mu^0_{h,1}, \mu^0_{h,2}, \cdots ,\mu^0_{h, d})$, $\mu^0_{h, i}(\cdot) \in \Delta(\cS), \forall i \in [d]$.
\end{assumption}
With \Cref{assumption:linear MDP}, the robust regularized value function and Q-function are defined as
\begin{align}
& V_h^{\pi, \lambda}(s) = \inf_{\bmu_t \in \Delta(\cS)^d,P_t=\la \bphi,\bmu_t \ra} \EE^{\{P_t\}_{t=h}^H} \bigg[\sum_{t = h}^H \big[r_t(s_t, a_t) \nonumber \\&\quad  + \lambda \la\bphi(s_t, a_t),\bD(\bmu_t||\bmu_t^0) \ra\big] \Big| s_h = s, \pi \bigg], \label{def:definition of d-rec value function}\\
& Q_h^{\pi, \lambda}(s,a) = \inf_{\bmu_t \in \Delta(\cS)^d,P_t=\la \bphi,\bmu_t \ra} \EE^{\{P_t\}_{t=h}^H} \bigg[\sum_{t = h}^H \big[ r_t(s_t,a_t) \nonumber
\\
&\quad  + \lambda \la\bphi(s_t, a_t),\bD(\bmu_t||\bmu_t^0) \ra\big] \Big| s_h = s, a_h = a, \pi \bigg], \nonumber
\end{align}
where $\bD(\bmu||\bmu^0) = [D(\mu_{i} \| \mu^0_{i})]_{i\in[d]}$.
In other words, we only consider perturbed kernels in the linear feasible set %
\begin{align*}
    \cF_{\text{L}} &= \big\{P=\{P_h\}_{h=1}^H | P_h(\cdot|s,a) = \la\bphi(s,a),\bmu_h(\cdot)\ra, \\
    &\bmu_h = (\mu_{h, 1}, \mu_{h, 2}, ... \mu_{h,d})^{\top}, \mu_{h, i}(\cdot) \in \Delta(\cS), \forall i \in [d\big\}.
\end{align*}
The optimal robust regularized value function and Q-function are defined as:
\begin{align}\label{def:optimal value function}
    \begin{split}
    V_h^{\star, \lambda}(s) &= \sup_{\pi}V_h^{\pi, \lambda}(s), \\
    Q_h^{\star, \lambda}(s,a) &= \sup_{\pi}Q_h^{\pi. \lambda}(s,a). 
    \end{split}
\end{align}
Based on \eqref{def:optimal value function}, the optimal robust policy is defined as the policy that achieves the optimal robust regularized value function,  
$\pi^{\star,\lambda} = \argmax_{\pi}V_1^{\pi, \lambda}(s),~\forall s\in\cS$.

\paragraph{Dynamic programming principles for $d$-RRMDPs} For completeness, we first show that the dynamic programming principles \citep{sutton2018reinforcement} hold for $d$-RRMDPs.
\begin{proposition}
\label{prop:regularized Robust Bellman equation}
(Robust Regularized Bellman Equation) Under the $d$-rectangular linear RRMDP, for any policy $\pi$ and any $(h,s,a) \in [H]\times \cS \times \cA$, we have
\begin{align}
& Q_{h}^{\pi, \lambda}(s, a) = r_h(s, a) + \inf_{\bmu_h \in \Delta(\cS)^d,P_h=\la \bphi,\bmu_h \ra}\big[ \nonumber \\
&\quad \EE_{s'\sim P_h(\cdot|s,a)}\big[V_{h+1}^{\pi, \lambda}(s')\big]+ \lambda \la\bphi(s, a),\bD(\bmu_h||\bmu_h^0) \ra\big], \notag
\\
&V_h^{\pi, \lambda}(s) = \EE_{a \sim \pi(\cdot |s)}\big[Q_{h}^{\pi, \lambda}(s, a)\big]. \label{regularized robust bellman equation: 2}
\end{align}
\end{proposition}
Next, we show that the optimal robust policy is deterministic and stationary. Hence, we can restrict the policy class $\Pi$ to the deterministic and stationary one. 
\begin{proposition}
\label{prop:existence of optimal value function}
 Under the $d$-rectangular linear RRMDP, 
there exists a deterministic and stationary policy $\pi^\star$, such that for any $(h,s,a) \in [H]\times \cS \times \cA$, $V^{\pi^\star, \lambda}_h(s) = V^{\star, \lambda}_{h}(s)$, and $ Q^{\pi^\star, \lambda}_h(s, a) = Q^{\star, \lambda}_h(s,a)$.
\end{proposition}
With \Cref{prop:regularized Robust Bellman equation} and \Cref{prop:existence of optimal value function},  we can derive the following robust regularized Bellman optimality equation:
\begin{align}
Q_{h}^{\star, \lambda}(s, a) & = r_h(s, a) + \inf_{\bmu_h \in \Delta(\cS)^d,P_h=\la \bphi,\bmu_h \ra}\big[\nonumber \\&\hspace{-2em} \EE_{s'\sim P_h(\cdot|s,a)}\big[V_{h+1}^{\star, \lambda}(s')\big] 
+ \lambda \la\bphi(s, a),\bD(\bmu_h||\bmu_h^0) \ra\big], \nonumber
\\
V_h^{\star, \lambda}(s) &= \max_{a\in\cA}Q_{h}^{\star, \lambda}(s, a).  \label{eq: robust regularized Bellman optimality equation - 2}
\end{align}
A direct consequence of \eqref{eq: robust regularized Bellman optimality equation - 2} is the optimal policy $\pi^{\star, \lambda}=\{\pi_h^{\star, \lambda}\}_{h=1}^H$ is the greedy policy with respect to the optimal robust Q-functions $\{Q^{\star,\lambda}_h\}_{h=1}^{H}$. Thus, in order to estimate $\pi^{\star, \lambda}$, it suffices to estimate $Q^{\star,\lambda}_h, \forall h\in[H]$.  

\paragraph{Offline dataset and learning goal.}
An agent works with an offline dataset $\cD$ with $K$ i.i.d. trajectories collected from the nominal environment by a behavior policy $\pi^b$. Specifically, for the $\tau$-th trajectory $\{(s_h^{\tau}, 
 a_h^{\tau}, r_h^{\tau})\}_{h=1}^H$, we have
 $a_h^{\tau}\sim \pi_h^b(\cdot|s_h^{\tau})$, $r_h^{\tau}=r_h(s_h^{\tau},a_h^{\tau})$, and $s_{h+1}^{\tau}\sim P^0_h(\cdot|s_h^{\tau},a_h^{\tau})$ for any $h\in[H]$.
The agent aims to learn the optimal robust policy $\pi^{\star}$ from the offline dataset $\cD$. %
Given a learned policy $\hat{\pi}$, we evaluate $\hat{\pi}$ by the suboptimality gap defined as follows
\begin{align}
\label{eq:definition of suboptimality}
    \text{SubOpt}(\hat{\pi}, s_1, \lambda) := V_1^{\star, \lambda}(s_1) - V_1^{\hat{\pi}, \lambda}(s_1).
\end{align}
\section{Robust Regularized Pessimistic Value Iteration (R2PVI)}
\label{sec:Robust Regularized Pessimistic Value Iteration}
In this section, we first develop a meta-algorithm for $d$-RRMDPs with general $f$-divergence defined regularization. To instantiate the meta-algorithm under specific $f$-divergences, we provide exact dual formulations of Q-functions with TV, KL and $\chi^2$-divergence defined regularization, respectively. 

  We first show that robust Q-functions admit linear representations under $d$-RRMDPs.
\begin{proposition}
\label{prop:linear of bellman operator}
 Under \Cref{assumption:linear MDP}, for any tuple $(\pi, s, a, h)$, we have $Q_h^{\pi, \lambda}(s,a) = \langle \bphi(s,a), \btheta_h + \bw_h^{\pi, \lambda} \rangle$, 
where $\bw_h^{\pi, \lambda} = \big(w_{h, 1}^{\pi, \lambda}, w_{h,2}^{\pi, \lambda},\cdots, w_{h,d}^{\pi, \lambda}\big)^\top \in \RR^d$, 
and $w_{h, i}^{\pi, \lambda} := \inf_{\mu_{h,i} \in \Delta(\cS)} \big[\EE^{\mu_{h,i}}[V_{h+1}^{\pi, \lambda}(s)] + \lambda D(\mu_{h,i} \| \mu_{h, i}^0)\big]$. 
\end{proposition}
The linear representation of the robust Q-function enables linear function approximation for parameter estimation. The definition of parameter $\bw_h^\lambda$ involves a regularized optimization. For any function $V:\cS\rightarrow \RR$, the dual formulation of the regularized optimization problem \citep{yang2023robust} is:
\begin{align*}
    &\inf_{\mu \in \Delta(S)}\EE_{s\sim\mu}V(s) + \lambda D_{f}(\mu\|\mu^0) \\
   &= \sup_{\alpha \in R}\Big[- \lambda \EE_{s \sim \mu^0}\Big[f^*\Big(\frac{\alpha - V(s)}{\lambda}\Big)\Big] + \alpha\Big],
\end{align*}
where $f^*$ is the conjugate function of $f$. We propose to estimate $w_h^\lambda$ through the ridge regression. We define the intermediate variable $w_{h,i}^{\lambda}(\alpha) := \EE_{s\sim\mu_i^0}[f^*(\frac{\alpha-V(s)}{\lambda})]$ and obtain an estimation $\hat{w}_{h,i}^{\lambda}(\alpha):=\big[\argmin_{\bw \in R^d}\sum_{\tau = 1}^K(f^*(\frac{\alpha-\hat{V}_{h+1}^\lambda(s)}{\lambda}) - \bphi(s_h^{\tau}, a_h^{\tau})^{\top}\bw)^2 + \lambda\|\bw \|_2^2 \big]^i=  \big[\bLambda_h^{-1}[\sum_{\tau=1}^K\bphi(s_h^{\tau}, a_h^{\tau})f^*(\frac{\alpha-\hat{V}_{h+1}^\lambda(s)}{\lambda})]\big]^i.$
We then estimate $w_{h,i}^{\lambda}$ by $\hat{w}_{h,i}^{\lambda} = \sup_{\alpha\in \RR} \{-\lambda\hat{w}_{h,i}^{\lambda}(\alpha)+\alpha\}$. Leveraging \Cref{prop:regularized Robust Bellman equation} and the pessimism principle \citep{jin2021pessimism} developed to take account for the distribution shift arising from the offline dataset, we propose the meta-algorithm in \Cref{alg:for f divergence}.
\begin{algorithm}[ht]
    \caption{ \algname\ under general $f$-divergence
    \label{alg:for f divergence}}
    \begin{algorithmic}[1]
    \REQUIRE{
        Dataset $\cD$, Regularizer $\lambda > 0$}
        \STATE init $\hat{V}_{H+1}^{\lambda}(\cdot)=0$
        \FOR {episode $h=H, \cdots, 1$}{
            \STATE Compute $\bLambda_h \leftarrow \sum_{\tau = 1}^K \bphi(s_h^{\tau}, a_h^{\tau})\bphi(s_h^{\tau}, a_h^{\tau})^{\top} + \gamma \mathbf{I}$
            \STATE  $\hat{w}_{h,i}^{\lambda}(\alpha) \leftarrow \big[\bLambda_h^{-1}[\sum_{\tau=1}^K\bphi(s_h^{\tau}, a_h^{\tau})f^*(\frac{\alpha-\hat{V}_{h+1}^\lambda(s)}{\lambda})]\big]^i$ \\ \hfill $\triangleright$ {\color{blue} Duality Estimation for general $f$-divergence}
            \label{algline: general f regression}
            \STATE $\hat{w}_{h,i}^{\lambda} \leftarrow \sup_{\alpha \in \RR} \{-\lambda\hat{w}_{h,i}^\lambda (\alpha)+\alpha\}$
            \label{algline: general f supremum}
            \STATE  Construct the penalty $\Gamma_h(\cdot, \cdot)$.  
            \STATE Estimate $\hat{Q}_h^{\lambda}(\cdot, \cdot) \leftarrow \min \{\langle \bphi(\cdot,\cdot), \btheta_h + \hat{\bw}_h^{\lambda} \rangle - \Gamma_h(\cdot, \cdot), H-h+1 \}^+$.
            \STATE Construct $\hat{\pi}_h(\cdot | \cdot) \leftarrow \argmax_{\pi_h}\la\hat{Q}_h^{\lambda}(\cdot, \cdot), \hat{\pi}_h(\cdot| \cdot)\ra_{\mathcal{A}}$ and $\hat{V}_h^{\lambda}(\cdot) \leftarrow\la\hat{Q}_h^{\lambda}(\cdot, \cdot), \hat{\pi}_h(\cdot|\cdot)\ra_{\mathcal{A}}$.
        }\ENDFOR
    \end{algorithmic}  
\end{algorithm}
\begin{remark}
    We emphasize that this general framework may encounter numerical challenges when computing the supremum over $\alpha$, especially depending on the choice of the divergence function $f$. In particular, the smoothness and curvature of the conjugate function $f^*$ can significantly affect the stability and efficiency of the optimization. For instance, some divergences lead to non-smooth or non-strongly convex conjugates, making the maximization problem harder to solve accurately. Therefore, while this framework is general, we highlight that divergence-specific algorithm designs are necessary to ensure tractability and numerical stability. 
\end{remark}

Next, we instantiate the $f$-divergence with TV, KL and $\chi^2$-divergences respectively, and specify the estimation procedure corresponding to different divergences.

\subsection{R2PVI with the TV-Divergence}
\label{sec:R2PVI With The TV-Divergence}
In this section, we show how to get the estimation in Line
\ref{algline: general f regression} and Line \ref{algline: general f supremum} of \Cref{alg:for f divergence} for TV divergence defined regularization.
We first present the following duality result.
\begin{proposition}
\label{prop:the regularized duality under TV divergence} Given any probability measure $\mu^0\in\Delta(\cS)$ and value function $V:\cS\rightarrow[0,H]$, if the distance $D$ is chosen as the TV-divergence, the dual formulation of the original regularized optimization problem is formed as: $\inf_{\mu \in \Delta (S)}\EE_{s\sim\mu}V(s) + \lambda D_{\text{TV}}(\mu\| \mu^0)=\EE_{s\sim\mu^0}[V(s)]_{V_{\min} +\lambda}.$
\end{proposition}
\begin{remark}
    We compare the duality of the regularized problem in \Cref{prop:the regularized duality under TV divergence} with the duality of the constraint problem  in DRMDPs with TV-divergence defined uncertainty sets \citep{shi2024distributionally}: $\inf_{P \in \mathcal{U}_{\text{TV}}^{\rho}(P^0)} \EE^{P} V(s)= \max_{\alpha \in [V_{\min}, V_{\max}]} \big\{ \EE^{P^0}[V(s)]_{\alpha}
    - \rho(\alpha - \min_{s'}[V(s')]_{\alpha}) \big\}$.
The former has a {\it closed form}, while the later involves an optimization over the dual variable $\alpha$.
We show later this distinction makes R2PVI much more computationally efficient compared to algorithms designed for DRMDPs. 
\end{remark}
Next, we present the parameter estimation procedure. Given an estimated robust value function $\hat{V}^{\lambda}_{h+1}$, we denote $\alpha_{h+1} = \min_{s'} \hat{V}^{\lambda}_{h+1}(s') + \lambda$. By the linear representation in \Cref{prop:linear of bellman operator}, the duality for TV-divergence in \Cref{prop:the regularized duality under TV divergence} and the linearly structured nominal kernel in \Cref{assumption:linear MDP}, we estimate the parameter $\bw_h^\lambda$ as follows
\begin{align}
\label{TV-estimation}
    \begin{split}
      \hat{\bw}_h^{\lambda} & =\argmin_{\bw \in R^d} \sum_{\tau=1}^K \big([\hat{V}_{h+1}^{\lambda}(s_{h+1}^\tau)]_{\alpha_{h+1}}  \\
    &\qquad- \bphi(s_h^\tau, a_h^{\tau})^{\top}\bw\big)^2 + \gamma \| \bw \|_2^2,       
    \end{split}
\end{align}
where $\gamma$ is the regularizer in the ridge regression.
\begin{remark}
    Thanks to the closed form expression of the duality for TV in \Cref{prop:the regularized duality under TV divergence}, R2PVI does not need the dual optimization oracle as the DRPVI algorithm  proposed for the $d$-DRMDP \citep[see their equation (4.4) and Algorithm 1 for more details]{liu2024minimax}. DRPVI needs to solve the dual optimization oracle separately for each dimension in each iteration, which is not necessary in our algorithm.%
\end{remark}

\subsection{R2PVI with the KL-Divergence}
Similar to the TV-divergence, we next derive the estimation in Line \ref{algline: general f regression} and Line \ref{algline: general f supremum} of \Cref{alg:for f divergence} for KL divergence defined regularization. We first present the duality result.
\begin{proposition}
\label{prop:the regularized duality under KL divergence}\citep[Example 1]{zhang2024soft}  Given any probability measure $\mu^0\in\Delta(\cS)$ and value function $V:\cS\rightarrow[0,H]$, if the probability divergence $D$ is chosen as the KL-divergence, then the dual formulation of the original regularized optimization problem is: $\inf_{\mu \in \Delta(S)}\EE_{s\sim\mu}V(s) + \lambda D_{\text{KL}}(\mu\|\mu^0) = -\lambda \log \EE_{s\sim \mu^0}\big[e^{-{V(s)}/{\lambda}}\big].$
\label{eq: KL robust bellman equation}
\end{proposition}
The duality of KL also has a closed form. we will shown in \Cref{sec:main theoretical results} and \Cref{sec:experiment} the closed form solution will reduce the computational cost and also ease the theoretical analysis. 
Next, we present the parameter estimation procedure. According to the linear representation of Q-functions in \Cref{prop:linear of bellman operator}, the duality for KL-divergence in \Cref{prop:the regularized duality under KL divergence} and the linearly structured nominal kernel in \Cref{assumption:linear MDP}, we estimate the parameter $\bw_h^\lambda$ by a two-step procedure. Given an estimated robust value function $\hat{V}^{\lambda}_{h+1}$, we first estimate $\EE_{s\sim\bmu^0}e^{-\hat{V}^{\lambda}_{h+1}(s)/\lambda}$ by 
$ \hat{\bw}'_h = \argmin_{\bw \in \RR^d} \sum_{\tau=1}^K \big(e^{-{\hat{V}_{h+1}^{\lambda}(s^\tau_{h+1})}/{\lambda}} - \bphi(s_h^\tau, a_h^{\tau})^{\top}\bw \big)^2 + \gamma \| \bw \|_2^2$. 
Then we take a log-transformation to get an estimation of $\bw_h^\lambda$:
\begin{align}
    \hat{\bw}^\lambda_h= -\lambda \log \max\{\hat{\bw}'_h, e^{-H/\lambda}\}. \label{KL-estimation}
\end{align}
Note that the max operator is to ensure the ridge-regression estimator is well-defined to take the log-transformation, and  $e^{-H/\lambda}$ is the lower bound on $\EE_{s\sim\bmu^0}e^{-\hat{V}^{\lambda}_{h+1}(s)/\lambda}$.
\begin{remark}
The algorithm proposed by \citet{ma2022distributionally} relies on dual optimization oracles under DRMDPs with KL divergence defined uncertainty sets, while our algorithm takes advantages of the closed-form duality solution. Their algorithm also relies on an additional value shift technique to guarantee the estimated parameter is well-defined to take the log-transformation, while our algorithm does not. 
\end{remark}
\subsection{R2PVI with the $\chi^2$-Divergence}
It remains to derive the estimation in Line \ref{algline: general f regression} and Line \ref{algline: general f supremum} of \Cref{alg:for f divergence} for $\chi^2$-divergence defined regularization. We first present a result on the duality of the $\chi^2$-divergence.
\begin{proposition}
\label{prop:the regularized duality under x2 divergence} Given any probability measure $\mu^0\in\Delta(\cS)$ and value function $V:\cS\rightarrow[0,H]$, if $D$ is chosen as the $\chi^2$-divergence, the dual formulation of the original regularized optimization problem is:
\begin{align}
& \inf_{\mu \in \Delta(S)} \EE_{s\sim\mu}V(s) + \lambda D_{\chi^2}(\mu\| \mu^0)= \label{eq: x2 robust bellman equation}
\\
& \sup_{\alpha \in [V_{\min}, V_{\max}]}\big\{ \EE_{s\sim\mu^0}[V(s)]_{\alpha} - \frac{1}{4\lambda}\Var_{s\sim\mu^0}[V(s)]_{\alpha} \big\}.\nonumber    
\end{align}
\end{proposition}
Next, we present the parameter estimation procedure. According to the linear representation of the Q-function in \Cref{prop:linear of bellman operator}, the duality for $\chi^2$-divergence in \Cref{prop:the regularized duality under x2 divergence} and the linear structure of the nominal kernel in \Cref{assumption:linear MDP}, we estimate the parameter $\bw_h^\lambda$ as follows.
First, we propose a new method motivated by the variance estimation in \citet{liu2024minimax} to estimate the variance of the value function in \eqref{eq: x2 robust bellman equation}.
Specifically, given an estimated robust value function $\hat{V}^{\lambda}_{h+1}$ and dual variable $\alpha$, the estimations of $\EE_{s\sim\mu^0}[\hat{V}^{\lambda}_{h+1}(s)]_{\alpha}$ and $\EE_{s\sim\mu^0}[\hat{V}^{\lambda}_{h+1}(s)]^2_{\alpha}$ are: 
\begin{align}
        & \hat{\EE}^{\mu_{h,i}^0}[\hat{V}_{h+1}^{\lambda}(s)]_{\alpha} = \bigg[\argmin_{\bw \in R^d} \sum_{\tau=1}^K ([\hat{V}_{h+1}^{\lambda}(s_{h+1}^\tau)]_{\alpha} \nonumber \\
        & \qquad- \bphi(s_h^\tau, a_h^{\tau})^{\top}\bw)^2 + \gamma \| \bw \|_2^2 \bigg]_{[0,H]}^i, \label{eq: estimation of EV in chi2} \\     
        &\hat{\EE}^{\mu_{h,i}^0}[\hat{V}_{h+1}^{\lambda}(s)]^2_{\alpha} = \bigg[\argmin_{\bw \in R^d} \sum_{\tau=1}^K ([\hat{V}_{h+1}^{\lambda}(s_{h+1}^\tau)]^2_{\alpha} \nonumber
        \\
        & \qquad - \bphi(s_h^\tau, a_h^{\tau})^{\top}\bw)^2 + \gamma \| \bw \|_2^2 \bigg]_{[0,H^2]}^i,        \label{eq: estimation of EV2 in chi2} \end{align}       
where the superscript $i$ represents the $i$-th element of a vector. Then we construct the estimator $\hat{\bw}_h^\lambda$ element-wisely:
    \begin{align}
        \label{eq:parameter estimation x2}  
        \begin{split}
        &\hat{w}_{h,i}^{\lambda} = \max_{\alpha \in [(\hat{V}_{h+1}^{\lambda})_{\min},(\hat{V}_{h+1}^{\lambda})_{\max}]} \Big \{ \hat{\EE}^{\mu_{h,i}^0}[\hat{V}_{h+1}^{\lambda}(s)]_{\alpha} +  \\ 
        &\frac{1}{4\lambda}(\hat{\EE}^{\mu_{h,i}^0}[\hat{V}_{h+1}^{\lambda}(s)]_{\alpha})^2- \frac{1}{4\lambda}\hat{\EE}^{\mu_{h,i}^0}[\hat{V}_{h+1}^{\lambda}(s)]^2_{\alpha} \Big \}. 
        \end{split}
    \end{align}
    We note that the above parameter estimation procedure involves an optimization, which is distinct from that of TV and KL, since the duality of $\chi^2$ does not admit a closed form expression. Specifically, it estimates the parameter  $\bw_h^{\lambda}$ element-wisely. For each dimension, it solves an optimization problem over an estimated dual formulation.
    This parameter estimation procedure shares a similar spirit with that in $d$-DRMDPs with TV divergence defined uncertainty sets \citep{ liu2024minimax, wang2024sample}.

To conclude, we summarize the TV, KL and $\chi^2$ divergences instantiation of 
\Cref{alg:for f divergence} in \Cref{alg:meta}.

\section{Suboptimality Analysis}
\label{sec:main theoretical results}
In this section, we establish theoretical guarantees for algorithms proposed in \Cref{sec:Robust Regularized Pessimistic Value Iteration}. First, we derive instance-dependent upper bounds on the suboptimality gap of policies learned by the instantiated algorithms. Next, under a partial coverage assumption on the offline dataset, we present instance-independent upper bounds for the suboptimality gap and compare them with results from previous works. Finally, we provide an information-theoretic lower bound to highlight the intrinsic characteristics of offline $d$-RRMDPs.

\begin{algorithm}[ht]
    \caption{ \algname\ under TV, KL and $\chi^2$ divergence 
    \label{alg:meta}}
    \begin{algorithmic}[1]
    \REQUIRE{
        Dataset $\cD$, Regularizer $\lambda > 0$
        }
        \STATE init $\hat{V}_{H+1}^{\lambda}(\cdot)=0$
        \FOR {episode $h=H, \cdots, 1$}{
            \STATE Compute  $\bLambda_h \leftarrow \sum_{\tau = 1}^K \bphi(s_h^{\tau}, a_h^{\tau})\bphi(s_h^{\tau}, a_h^{\tau})^{\top} + \gamma \mathbf{I}$
            \STATE Obtain the parameter estimation $\hat{\bw}_h^{\lambda}$ as follows: \alglinelabel{eq:parameter estimation}  \\ 
            TV-divergence: use \eqref{TV-estimation} \\
            KL-divergence: use \eqref{KL-estimation}\hfill $\triangleright$ {\color{blue} Duality Estimation}\\
            $\chi^2$-divergence: use \eqref{eq:parameter estimation x2}
            \STATE  Construct the penalty $\Gamma_h(\cdot, \cdot)$. \alglinelabel{algline:pessimism penalty} \hfill $\triangleright$ {\color{blue} Pessimism}
            \STATE Estimate $\hat{Q}_h^{\lambda}(\cdot, \cdot) \leftarrow \min \{\langle \bphi(\cdot,\cdot), \btheta_h + \hat{\bw}_h^{\lambda} \rangle - \Gamma_h(\cdot, \cdot), H-h+1 \}^+$.
            \STATE Construct $\hat{\pi}_h(\cdot | \cdot) \leftarrow \argmax_{\pi_h}\la\hat{Q}_h^{\lambda}(\cdot, \cdot), \hat{\pi}_h(\cdot| \cdot)\ra_{\mathcal{A}}$ and $\hat{V}_h^{\lambda}(\cdot) \leftarrow\la\hat{Q}_h^{\lambda}(\cdot, \cdot), \hat{\pi}_h(\cdot|\cdot)\ra_{\mathcal{A}}$.
        }\ENDFOR
    \end{algorithmic}  
\end{algorithm}

\subsection{Instance-Dependent Upper Bound}
 \label{sec: Instance-Dependent Bound of R2PVI}

\begin{theorem}
    \label{thm: instance-dependent upper bounds}
    Suppose \Cref{assumption:linear MDP} holds. We set $\gamma=1$ and $\Gamma_h(s,a) = \beta \sum_{i=1}^d\|\phi_i(\cdot, \cdot) \mathbf{1}_i\| _{\bLambda_h^{-1}}$ in \Cref{alg:meta}. Let $\delta \in (0,1) $. $\beta$ is chosen as follows.  
     \begin{itemize}[leftmargin=*,nosep]
         \item (TV) $\beta = 16 Hd \sqrt{\xi_{\text{TV}}}$, %
         \item (KL)  $\beta = 16d\lambda  e^{H/\lambda} \sqrt{({H}/{\lambda} + \xi_{\text{KL}})}$,
         \item ($\chi^2$) $\beta = 8 dH(1+ 3H/4\lambda) \sqrt{\xi_{\chi^2}}$, %
     \end{itemize}
     where $\xi_{\text{TV}} = 2\log(1024Hd^{1/2} K^2 /\delta)$,  $\xi_{\text{KL}} = \log(1024d\lambda^2 K^3H/\delta)$ and $\xi_{\chi^2} = \log (192 K^5H^6 d^3 (1 + H/2\lambda)^3/ \delta)$. Then with probability at least $1-\delta$, for any $s \in \mathcal{S}$,  we have
         $\text{SubOpt}(\hat{\pi}, s, \lambda) \leq  2\beta\sup_{P \in \mathcal{U}^\lambda(P^0)}\sum_{h=1}^H\EE^{\pi^\star, P}[\sum_{i=1}^d\|\phi_i(s,a)\mathbf{1}_i\|_{\bLambda_h^{-1}}|s_1 = s]$, 
    where $\mathcal{U}^\lambda (P^0)$ is the robustly admissible set defined as 
\begin{align}
\label{def:robustly admissible set}
\mathcal{U}^\lambda(P^0)=\bigotimes_{(h, s, a) \in [H] \times \mathcal{S} \times \mathcal{A}} \mathcal{U}_h^\lambda(s, a ; \bmu_h^0),
\end{align}
and 
$\mathcal{U}_h^\lambda(s, a ; \bmu_h^0) =\big \{\sum_{i=1}^d \phi_i(s, a) \mu_{h, i}(\cdot):D(\mu_{h, i}\| \mu_{h,i}^0) \leq \max_{s\in\cS} V_{h+1}^{\star, \lambda}(s) /\lambda, \forall i \in[d]\big \}$.
\end{theorem}
\begin{remark} 
\Cref{thm: instance-dependent upper bounds} provides instance-dependent upper bounds on the suboptimality gap, closely resembling the bounds established for algorithms tailored to $d$-DRMDPs with TV divergence defined uncertainty sets \citep{liu2024minimax, wang2024sample}. Notably,  $\mathcal{U}^\lambda (P^0)$ in \Cref{thm: instance-dependent upper bounds} represents a subset of the feasible set $\cF_{\text{L}}$ in the $d$-RRMDP. While the RRMDP framework does not impose explicit uncertainty set constraints, this term naturally arises from our theoretical analysis (see \Cref{lemma:regularized robust suboptimality lemma} and its proof for details). Specifically, we show that only distributions within $\mathcal{U}^\lambda (P^0)$ are relevant when considering the infimum in the robust regularized value and Q-functions \eqref{def:definition of d-rec value function}. Intuitively, the regularization term in \eqref{def:definition of d-rec value function} should not exceed the change in expected cumulative rewards, so it should be upper bounded by the optimal value function.
Similar terms are also found in  \citet[Definition 1]{zhang2024soft} and \citet[Assumption 1]{panaganti2022robust}.
\end{remark}

\subsection{Instance-Independent Upper Bound}
Next, we derive instance-independent upper bounds on the suboptimality gap, building on \Cref{thm: instance-dependent upper bounds}. To achieve this, we adapt the robust partial coverage assumption on the offline dataset, originally proposed for $d$-DRMDPs (Assumption A.2 of \citet{blanchet2024double}). This adaptation is straightforward and involves replacing the uncertainty set in the $d$-DRMDP framework with the robustly admissible set defined in \eqref{def:robustly admissible set}.
\begin{assumption}[Robust Regularized Partial Coverage]
    \label{robust partial coverage assumption}
    For the offline dataset $\cD$, we assume that there exists some constant $c^\dagger >0$, such 
    that $\forall (h, s, P) \in [H] \times \cS \times \mathcal{U}^{\lambda}(P^0)$,
    \begin{align*}
        \bLambda_h \succeq \gamma \mathbf{I} + K \cdot c^\dagger \cdot \EE^{\pi^\star, P}\big[\phi_i^2(s,a)\mathbf{1}_i\mathbf{1}_i^{\top}| s_1 = s\big].
    \end{align*}
\end{assumption}
Intuitively, \Cref{robust partial coverage assumption} assumes that the offline dataset has good coverage on the $(s,a)$-space visited by the optimal robust policy $\pi^{\star}$ under any transition kernel in the robustly admissible set. With \Cref{robust partial coverage assumption} and \Cref{thm: instance-dependent upper bounds}, we present instance-independent bounds as follows.
\begin{corollary}
\label{corollary:simplified upper bounds}
 Under the same setting as \Cref{thm: instance-dependent upper bounds}, if we further assume \Cref{robust partial coverage assumption} holds, then for any $\delta \in (0,1)$ and $s\in\cS$, with probability at least $1-\delta$, we have
    \begin{itemize}[leftmargin=*,nosep]
        \item(TV) $\text{SubOpt}(\hat{\pi}, s, \lambda) \leq 16 H^2d^2\sqrt{\xi_{\text{TV}}}/{\sqrt{c^\dagger K}}$;
        \item(KL) $\text{SubOpt}(\hat{\pi}, s, \lambda) \leq 16\lambda e^{\frac{H}{\lambda}}d^2 H (\frac{H}{\lambda} + \xi_{\text{KL}})^{\frac{1}{2}}/\sqrt{c^\dagger K}$;
        \item($\chi^2$) $\text{SubOpt}(\hat{\pi}, s, \lambda) \leq 8 d^2H^2(1+ \frac{3H}{4\lambda}) \sqrt{\xi_{\chi^2}}/\sqrt{c^\dagger K}$.
    \end{itemize}
\end{corollary}
We compare \Cref{alg:meta} with algorithms proposed in previous works for the offline $d$-DRMDP in \Cref{tab:results}. For the case with TV-divergence, the suboptimality bound of R2PVI matches that of P2MPO \citep{blanchet2024double} in terms of $d$ and $H$. DRPVI \cite{liu2024minimax} and DROP \cite{wang2024sample} admit tighter bounds on the suboptimality gap, simply because their bounds are derived based on advanced techniques, such as reference-advantage decomposition \cite{xiong2022nearly}. We remark that our analysis can be tailored to adopt the same techniques and assumption, and thus get tighter bounds.

For the case with KL-divergence, existing theoretical results \citep{ma2022distributionally, blanchet2024double} rely on an additional regularity assumption regarding the KL dual variable, stating that the optimal dual variable for the KL duality admits a positive lower bound $\underline{\beta}$ under any feasible transition kernel (see \citet[Assumption F.1]{blanchet2024double}). However, this assumption presents the following drawbacks. First, it is challenging to verify the assumption's validity in practice; second, even if such a lower bound holds, there is no straightforward method to determine the magnitude of the lower bound. It can be seen from \Cref{tab:results} that the suboptimality bound of R2PVI matches that of DRVI-L \citep{ma2022distributionally} in terms of $d$ and $H$. However, our result depends on $\lambda$ which is the regularization parameter and can be arbitrarily chosen,  while the result of DRVI-L depends on $\underline{\beta}$ which can be extremely small such that $\sqrt{\underline{\beta}}e^{H/\underline{\beta}}\gg \sqrt{\lambda}e^{H/\lambda}$. Moreover, \citet{zhang2024soft}\footnote{\citet{zhang2024soft} studied the infinite horizon RRMDP with a discounted factor $\gamma$, we replace the effective horizon length $\frac{1}{1-\gamma}$ by the horizon length $H$ in the finite horizon setting.} studied the RRMDP with the regularization term defined by the KL-divergence in their Theorem 5 and the suboptimality bound also depends on the term $\sqrt{\lambda}e^{H/\lambda}$. Further, comparing the bounds of P2MPO \citep{blanchet2024double} and R2PVI, we can qualitatively conclude that the regularization parameter $\lambda$ in $d$-RRMDPs plays a role analogous to $1/\rho$ in $d$-DRMDPs. This relation aligns with the intuition that a smaller $\lambda$ in $d$-RRMDPs or a larger $\rho$ in $d$-DRMDPs can induce a more robust policy.

\begin{table*}[ht]
    \centering  
    \caption{
    Comparison of the suboptimality gap between this and previous works. The $^{\star}$ symbol denotes results that require an additional assumption (Assumption 4.4 of \citet{ma2022distributionally} and Assumption F.1 of \citet{blanchet2024double}) on the KL dual variable, an assumption not required by our R2PVI algorithm. The parameter $\rho$ represents the uncertainty level in DRMDPs, while $\lambda$ represents the regularization term in RRMDPs. The Coverage column indicates the assumption used to derive the suboptimality gap: the robust partial coverage assumption refers to Assumption A.2 of \citet{blanchet2024double}, and the regularized partial coverage assumption represents \Cref{robust partial coverage assumption}.
    \label{tab:results}}
    \resizebox{0.85\textwidth}{!}{%
    \begin{tabular}{llcccc}
        \toprule
        {Algorithm} & {Setting} &{Divergence} & {Coverage} & {Suboptimality Gap}  \\
        \midrule
         \makecell[l]{DRPVI \\ \citep{liu2024minimax}} & $d$-DRMDP & TV & full &  $\Tilde{O}(dH^2 K^{-1/2})$ \\
        \makecell[l]{DROP \\ \citep{wang2024sample}} & $d$-DRMDP & TV & robust partial &  $\Tilde{O}(d^{3/2}H^2 K^{-1/2})$ \\
        \makecell[l]{P2MPO (TV) \\ \citep{blanchet2024double}}& $d$-DRMDP & TV & robust partial &  $\Tilde{O}(d^{2}H^2 K^{-1/2})$ \\
        \rowcolor{lightblue} {R2PVI-TV \textbf{(ours)}} & $d$-RRMDP & TV & regularized partial &  $\Tilde{O}(d^{2}H^2 K^{-1/2})$ \\
        \midrule
        \makecell[l]{DRVI-L \\ \citep{ma2022distributionally}} &$d$-DRMDP & KL & robust partial& $\Tilde{O}(\sqrt{\underline{\beta}} e^{H/\underline{\beta}}d^{2}H^{3/2} K^{-1/2})^{\star}$ \\
        \makecell[l]{P2MPO (KL) \\ \citep{blanchet2024double}}& $d$-DRMDP & KL & robust partial & $\Tilde{O}(e^{H/\underline{\beta}}d^{2}H^{2}\rho^{-1} K^{-1/2})^{\star}$\\
        \rowcolor{lightblue} {R2PVI-KL \textbf{(ours)}} & $d$-RRMDP & KL & regularized partial &  $\Tilde{O}(\sqrt{\lambda} e^{H/ \lambda} d^2 H^{3/2} K^{-1/2})$ \\
        \midrule

       \rowcolor{lightblue} {R2PVI-$\chi^2$ \textbf{(ours)}} & $d$-RRMDP & $\chi^2$ & regularized partial &  $\Tilde{O}(d^2H^2(1+ H/\lambda)K^{-1/2})$ \\
        \bottomrule
        
        \end{tabular}}%
             
\end{table*}

For the case with $\chi^2$ divergence, our bound is the first result in literature. Compared with the TV divergence, the complexity is higher due to the more complex geometry and dual formulation of $\chi^2$ divergence. This observation aligns with findings of tabular DRMDPs with TV and $\chi^2$ divergence defined uncertainty sets \citep{shi2024curious}. While existing works have focused on the $(s,a)$-rectangular structured regularization, our work fills the theoretical gap in RRMDPs by introducing the $d$-rectangular structured regularization, a contribution that may be of independent interest.

\subsection{Information-Theoretic Lower Bound}
We highlight that in \Cref{thm: instance-dependent upper bounds}, suboptimality bounds under cases with TV, KL, $\chi^2$-divergence share the same term $\sup_{P\in\cU^{\lambda}(P^0)}
 \sum_{h=1}^H\EE^{\pi^\star, P}\big[\sum_{i=1}^d\Vert \phi_i(s_h,a_h)\mathbf{1}_i\Vert_{\bLambda_h^{-1}} |s_1=s\big]$.
 In this section, we establish information theoretic lower bounds to show that this term is intrinsic in offline $d$-RRMDPs.

 In order to give a formal presentation of the information-theoretical lower bound, we define $\cM$ as a class of $d$-RRMDPs and  $\text{SubOpt}(M,\hat{\pi},s,\rho)$ as the suboptimality gap specific to one $d$-RRMDP instance $M\in\cM$. We state the information-theoretic lower bound in the following theorem.
 \begin{theorem}
    \label{th:lower bound}
    Let $K>\max\{\tilde{O}(d^6), \tilde{O}(d^3H^2/\lambda^2)\}$ be the sample size, where we have regularizer $\lambda$, dimension $d$, horizon length $H$.  There exists a class of $d$-rectangular linear RRMDPs $\cM$ and  an offline dataset $\cD$ of size $K$ such that for any $\delta\in(0,1)$, $s\in\cS$, divergence $D$ among $D_{\text{TV}}, D_{\text{KL}}$ and $D_{\chi^2}$, with probability at least $1-\delta$, we have $\inf_{\hat{\pi}}\sup_{M\in\cM} \text{SubOpt}(M, \hat{\pi},s,\lambda, D) \geq c \cdot \sup_{P\in\cU^{\lambda}(P^0)}
    \sum_{h=1}^H\EE^{\pi^\star, P}[\sum_{i=1}^d\Vert \phi_i(s_h,a_h)\mathbf{1}_i\Vert_{\bLambda_h^{-1}} \big|s_1=s]$, 
    where $c$ is a universal constant.
\end{theorem}

\Cref{th:lower bound} is a universal information theoretic lower bound for $d$-RRDMPs with all three divergences studied in \Cref{sec:main theoretical results}. \Cref{th:lower bound} shows that the instance-dependent term is actually intrinsic to the offline $d$-RRDMPs, and \Cref{alg:meta} is near-optimal up to a factor $\beta$, for which the definition varies among different divergence metric $D$ as shown in \Cref{thm: instance-dependent upper bounds}. The proof outline of \Cref{th:lower bound} is inspired by that of Theorem 6.1 in \citet{liu2024minimax}, but here we need careful treatment on bounding the robust regularized value function by duality under different choices of $f$-divergences. 
We provide more details on the hard instance construction, the proof techniques, and the comparison with existing results in \Cref{sec:proof of lower bound}. 

\begin{figure*}[h]
    \centering
    \subfigure[$\lambda = 0.1$]{
    \includegraphics[width=0.23\linewidth]{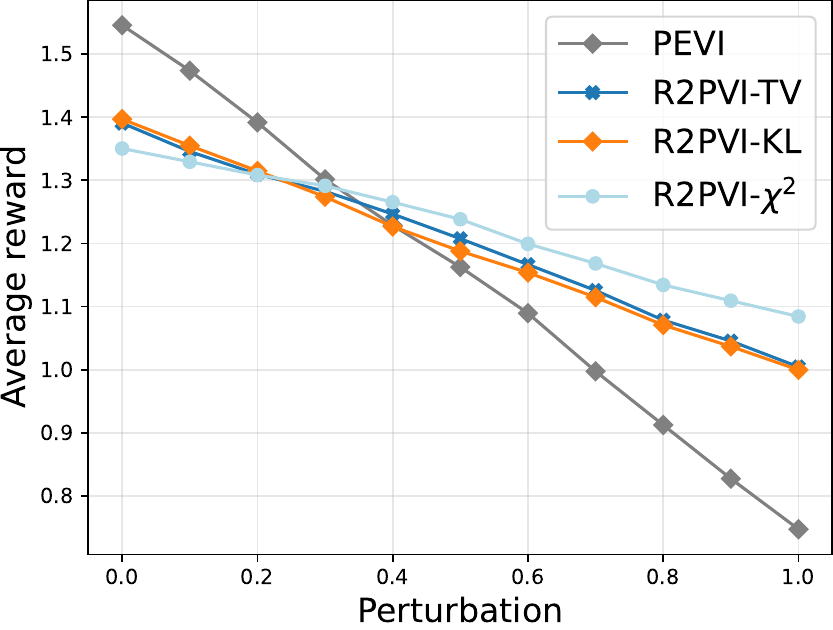}
    \label{fig:linear mdp - comparison}
    }
    \subfigure[R2PVI]{
        \includegraphics[width=0.23\linewidth]{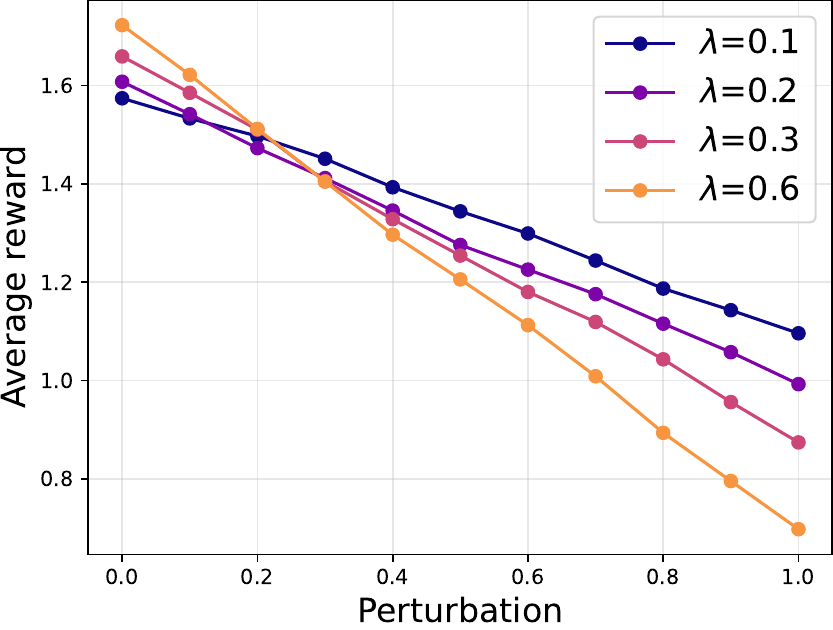}
        \label{fig:linear mdp - lambda}
    }   
    \subfigure[$q=0.9$]{
        \includegraphics[width=0.23\linewidth]{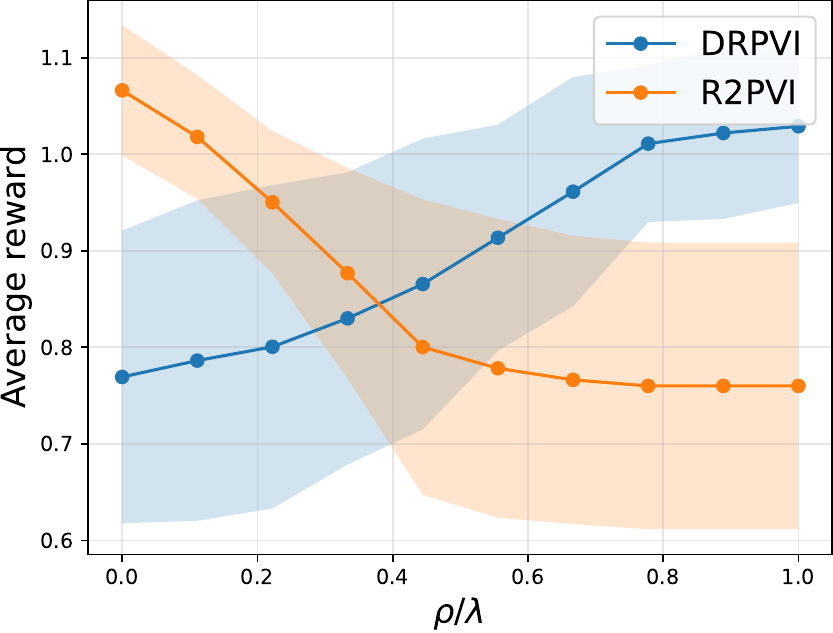}
    \label{fig:linear mdp - rho and lambda}
    }  
     \subfigure[(up to down) $\lambda = 1, 3, 5, \rho = 0.2, 0.1, 0.025$.]{
        \includegraphics[width=0.23\linewidth]{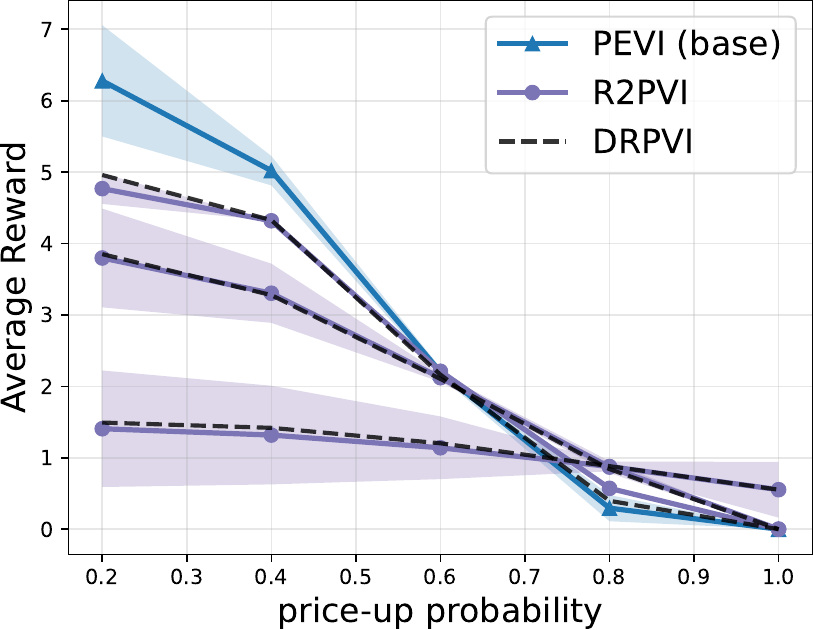}
    \label{fig: robustness}
    }
    \caption{Simulated results for linear MDP. In \Cref{fig:linear mdp - comparison} and \Cref{fig:linear mdp - lambda}, the $x$-axis refers to the perturbation in the testing environment. In \Cref{fig:linear mdp - rho and lambda}, the $x$-axis represents different robust level $\rho$ and regularized penalty $\lambda$, respectively. \Cref{fig: robustness} shows the robustness of algorithms under different robust level $\rho$ (DRPVI) or regularization penalty $\lambda$ (R2PVI). }
        \label{fig: simulated MDP}
\end{figure*}
\begin{figure}[ht]
    \centering
    \subfigure[execution time w.r.t N.]{
\includegraphics[width=0.46\linewidth]
    {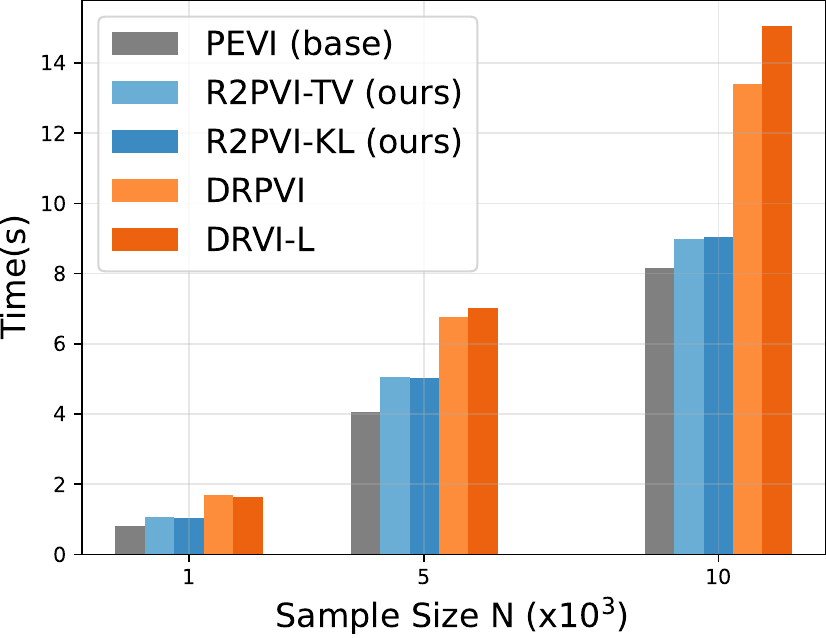}
    \label{fig: excution time w.r.t N}
    }
    \subfigure[execution time w.r.t d.]{
    \includegraphics[width=0.46\linewidth]{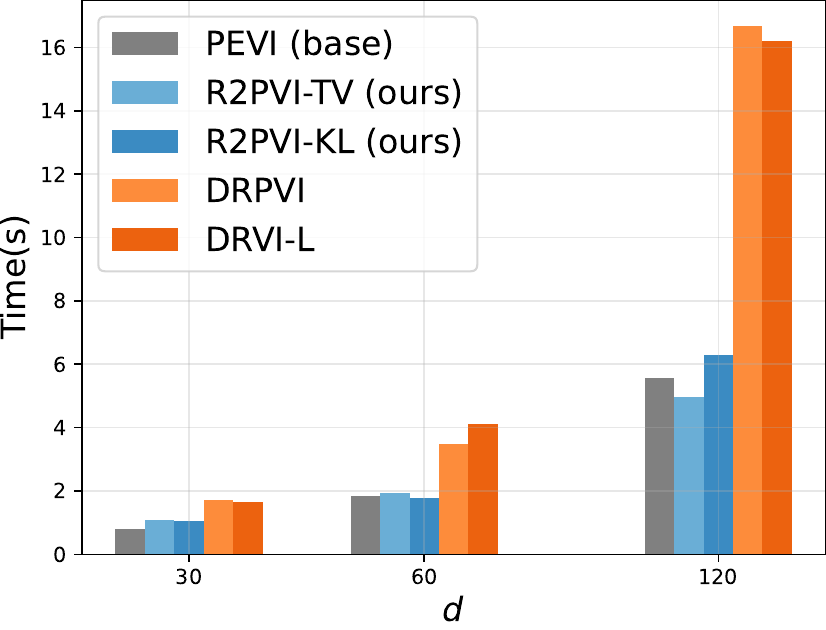}
    \label{fig:exution time w.r.t. d}
    }
    \caption{Simulation results for the simulated American put option task. \Cref{fig: excution time w.r.t N} shows the computation time of \algname\ with respect to the sample size $N$. 
    \Cref{fig:exution time w.r.t. d} shows the computation time of algorithms with respect to the feature dimension $d$.
    }
    \label{fig: American put option}
\end{figure}

\section{Experiment}
\label{sec:experiment}
In this section, we conduct numerical experiments to explore (1) the robustness of R2PVI regarding dynamics shifts, (2) how the regularizer $\lambda$ affects the robustness of R2PVI, and (3) the computation cost of R2PVI. We evaluate our algorithm in two off-dynamics problems that have been used in the literature \citep{ma2022distributionally,liu2024distributionally}. All experiments are conducted on a machine with an 11th Gen Intel(R) Core(TM) i5-11300H @ 3.10GHz processor, featuring 8 logical CPUs, 4 physical cores, and 2 threads per core. The implementation of our \algname\ algorithm is available at https://github.com/panxulab/Robust-Regularized-Pessimistic-Value-Iteration.

\paragraph{Baselines.}We compare our algorithms with three types of baseline frameworks: (1) non-robust pessimism-based algorithm: PEVI~\citep{jin2021pessimism}, (2) algorithms for $d$-DRMDPs with TV divergence defined uncertainty sets: DRPVI~\citep{liu2024minimax}, (3) algorithms for $d$-DRMDPs with KL divergence defined uncertainty sets: DRVI-L~\citep{ma2022distributionally}. We do not implement P2MPO and DROP mentioned in \Cref{tab:results} in our experiment, due to the lack of code base and numerical experiment in their works.

\subsection{Simulated Linear MDPs}
We borrow the simulated linear MDP constructed in \citet{liu2024distributionally} and adapt it to the offline setting. We set the behavior policy $\pi^b$ such that it chooses actions uniformly at random. The sample size of the offline dataset is set to 100. For completeness, we present more details on the experiment set up and results in \Cref{supp:numerical experiment}.

In \Cref{fig:linear mdp - comparison}, we compare \algname\ with its non-robust  counterpart PEVI \citep{jin2021pessimism}. We conclude that PEVI outperforms \algname\ when the perturbation of the environment is small, but underperforms  when the environment encounters a significant shift, which verifies the robustness of \algname. The regularizer $\lambda$ controls the extent of robustness of \algname\ by determining the magnitude of the penalty as shown in  \Cref{prop:regularized Robust Bellman equation}. By \Cref{fig:linear mdp - lambda}, we conclude that a smaller $\lambda$ leads to a more robust policy. To illustrate the relation between the $d$-RRMDP and the $d$-DRMDP, we fix a target environment, and then test \algname\ with different $\lambda$ and DRPVI \citep{liu2024minimax} with different $\rho$. We find from \Cref{fig:linear mdp - rho and lambda} that the ranges of the average reward are about the same for the two algorithms, though the behaviors w.r.t. $\lambda$ and $\rho$ are opposite. Thus, we verify that the regularizer $\lambda$ plays a similar role in the RRMDP as the inverted robustness parameter $1/\rho$ in the DRMDP.
\subsection{Simulated American Put Option}
In this section, we test our algorithm in a simulated American Put Option environment \citep{tamar2014scaling, zhou2021finite} that does not belong to the $d$-rectangular linear RRMDP. This environment is a finite horizon MDP with $H=20$, and 
is controlled by a hyperparameter $p_0$, which is set to be 0.5 in the nominal environment. We collect the offline data from the nominal environment by a uniformly random behavior policy. 
An agent uses the collected offline dataset to learn a policy which decides at each state whether or not to exercise the option. To implement our algorithm, we use a manually designed feature mapping of dimension $d$. For more details, we refer readers to \Cref{experiment: APO-setup}.

All experiment results are shown in \Cref{fig: American put option}. In particular, 
from \Cref{fig: excution time w.r.t N} and \Cref{fig:exution time w.r.t. d}, we can conclude that the computation cost of \algname\ is as low as its non-robust counterpart PEVI \citep{jin2021pessimism}, and improves that of DRPVI \citep{liu2024minimax} and DRVI-L \citep{ma2022distributionally} designed for 
the $d$-DRMDP. This is due to the closed form duality of TV and  KL under the $d$-RRMDP framework. From \Cref{fig: robustness}, we conclude that \algname\ not only demonstrates robustness to environment perturbations but also matches DRPVI's performance for appropriate values of the robust regularizer $\lambda$ and uncertainty level $\rho$.

\section{Conclusion}
We introduced the $d$-rectangular linear Robust Regularized Markov Decision Process ($d$-RRMDP) framework to address limitations of the $d$-rectangular DRMDP framework for robust policy learning in literature, improving both theoretical robustness and computational efficiency. We developed \algname, a provably effective algorithm for learning robust policies from offline datasets using $f$-divergence-based regularization. Our results highlight the advantages of $d$-RRMDPs, particularly in simplifying the duality oracle. Experiments confirm \algname's robustness and efficiency. It remains an intriguing open question to improve the current upper and lower bounds to study the fundamental hardness of $d$-RRMDPs. 

\section*{Impact Statement}
This paper presents work whose goal is to advance the field
of Machine Learning. There are many potential societal
consequences of our work, none which we feel must be
specifically highlighted here.

\section*{Acknowledgements}
Z. Liu and P. Xu are supported in part by the National Science Foundation (DMS-2323112) and the Whitehead Scholars Program at the Duke University School of Medicine.

\bibliography{referenceRMDP,reference}

\begin{thebibliography}{37}
\providecommand{\natexlab}[1]{#1}
\providecommand{\url}[1]{\texttt{#1}}
\expandafter\ifx\csname urlstyle\endcsname\relax
  \providecommand{\doi}[1]{doi: #1}\else
  \providecommand{\doi}{doi: \begingroup \urlstyle{rm}\Url}\fi

\bibitem[Blanchet et~al.(2024)Blanchet, Lu, Zhang, and Zhong]{blanchet2024double}
Blanchet, J., Lu, M., Zhang, T., and Zhong, H.
\newblock Double pessimism is provably efficient for distributionally robust offline reinforcement learning: Generic algorithm and robust partial coverage.
\newblock \emph{Advances in Neural Information Processing Systems}, 36, 2024.

\bibitem[Garc{\i}a \& Fern{\'a}ndez(2015)Garc{\i}a and Fern{\'a}ndez]{garcia2015comprehensive}
Garc{\i}a, J. and Fern{\'a}ndez, F.
\newblock A comprehensive survey on safe reinforcement learning.
\newblock \emph{Journal of Machine Learning Research}, 16\penalty0 (1):\penalty0 1437--1480, 2015.

\bibitem[Goyal \& Grand-Clement(2023)Goyal and Grand-Clement]{goyal2023robust}
Goyal, V. and Grand-Clement, J.
\newblock Robust markov decision processes: Beyond rectangularity.
\newblock \emph{Mathematics of Operations Research}, 48\penalty0 (1):\penalty0 203--226, 2023.

\bibitem[Guo et~al.(2024)Guo, Wang, Shi, Xu, and Liu]{guo2024off}
Guo, Y., Wang, Y., Shi, Y., Xu, P., and Liu, A.
\newblock Off-dynamics reinforcement learning via domain adaptation and reward augmented imitation.
\newblock In \emph{Advances in Neural Information Processing Systems}, volume~37, pp.\  136326--136360, 2024.

\bibitem[Iyengar(2005)]{iyengar2005robust}
Iyengar, G.~N.
\newblock Robust dynamic programming.
\newblock \emph{Mathematics of Operations Research}, 30\penalty0 (2):\penalty0 257--280, 2005.

\bibitem[Jin et~al.(2020)Jin, Yang, Wang, and Jordan]{jin2020provably}
Jin, C., Yang, Z., Wang, Z., and Jordan, M.~I.
\newblock Provably efficient reinforcement learning with linear function approximation.
\newblock In \emph{Conference on Learning Theory}, pp.\  2137--2143. PMLR, 2020.

\bibitem[Jin et~al.(2021)Jin, Yang, and Wang]{jin2021pessimism}
Jin, Y., Yang, Z., and Wang, Z.
\newblock Is pessimism provably efficient for offline rl?
\newblock In \emph{International Conference on Machine Learning}, pp.\  5084--5096. PMLR, 2021.

\bibitem[Levine et~al.(2020)Levine, Kumar, Tucker, and Fu]{levine2020offline}
Levine, S., Kumar, A., Tucker, G., and Fu, J.
\newblock Offline reinforcement learning: Tutorial, review, and perspectives on open problems.
\newblock \emph{arXiv preprint arXiv:2005.01643}, 2020.

\bibitem[Liu \& Xu(2024{\natexlab{a}})Liu and Xu]{liu2024distributionally}
Liu, Z. and Xu, P.
\newblock Distributionally robust off-dynamics reinforcement learning: Provable efficiency with linear function approximation.
\newblock In \emph{International Conference on Artificial Intelligence and Statistics}, pp.\  2719--2727. PMLR, 2024{\natexlab{a}}.

\bibitem[Liu \& Xu(2024{\natexlab{b}})Liu and Xu]{liu2024minimax}
Liu, Z. and Xu, P.
\newblock Minimax optimal and computationally efficient algorithms for distributionally robust offline reinforcement learning.
\newblock In \emph{Advances in Neural Information Processing Systems}, volume~37, pp.\  86602--86654, 2024{\natexlab{b}}.

\bibitem[Liu \& Xu(2025)Liu and Xu]{liu2025linear}
Liu, Z. and Xu, P.
\newblock Linear mixture distributionally robust markov decision processes.
\newblock \emph{arXiv preprint arXiv:2505.18044}, 2025.

\bibitem[Liu et~al.(2024)Liu, Wang, and Xu]{liu2024upper}
Liu, Z., Wang, W., and Xu, P.
\newblock Upper and lower bounds for distributionally robust off-dynamics reinforcement learning.
\newblock \emph{arXiv preprint arXiv:2409.20521}, 2024.

\bibitem[Lu et~al.(2024)Lu, Zhong, Zhang, and Blanchet]{lu2024distributionally}
Lu, M., Zhong, H., Zhang, T., and Blanchet, J.
\newblock Distributionally robust reinforcement learning with interactive data collection: Fundamental hardness and near-optimal algorithms.
\newblock In \emph{The Thirty-eighth Annual Conference on Neural Information Processing Systems}, 2024.

\bibitem[Ma et~al.(2022)Ma, Liang, Xia, Zhang, Blanchet, Liu, Zhao, and Zhou]{ma2022distributionally}
Ma, X., Liang, Z., Xia, L., Zhang, J., Blanchet, J., Liu, M., Zhao, Q., and Zhou, Z.
\newblock Distributionally robust offline reinforcement learning with linear function approximation.
\newblock \emph{arXiv preprint arXiv:2209.06620}, 2022.

\bibitem[Nelder \& Mead(1965)Nelder and Mead]{nelder1965simplex}
Nelder, J.~A. and Mead, R.
\newblock A simplex method for function minimization.
\newblock \emph{The computer journal}, 7\penalty0 (4):\penalty0 308--313, 1965.

\bibitem[Nilim \& El~Ghaoui(2005)Nilim and El~Ghaoui]{nilim2005robust}
Nilim, A. and El~Ghaoui, L.
\newblock Robust control of markov decision processes with uncertain transition matrices.
\newblock \emph{Operations Research}, 53\penalty0 (5):\penalty0 780--798, 2005.

\bibitem[Packer et~al.(2018)Packer, Gao, Kos, Kr{\"a}henb{\"u}hl, Koltun, and Song]{packer2018assessing}
Packer, C., Gao, K., Kos, J., Kr{\"a}henb{\"u}hl, P., Koltun, V., and Song, D.
\newblock Assessing generalization in deep reinforcement learning.
\newblock \emph{arXiv preprint arXiv:1810.12282}, 2018.

\bibitem[Panaganti \& Kalathil(2022)Panaganti and Kalathil]{panaganti2022sample}
Panaganti, K. and Kalathil, D.
\newblock Sample complexity of robust reinforcement learning with a generative model.
\newblock In \emph{International Conference on Artificial Intelligence and Statistics}, pp.\  9582--9602. PMLR, 2022.

\bibitem[Panaganti et~al.(2022)Panaganti, Xu, Kalathil, and Ghavamzadeh]{panaganti2022robust}
Panaganti, K., Xu, Z., Kalathil, D., and Ghavamzadeh, M.
\newblock Robust reinforcement learning using offline data.
\newblock \emph{Advances in Neural Information Processing Systems}, 35:\penalty0 32211--32224, 2022.

\bibitem[Panaganti et~al.(2024{\natexlab{a}})Panaganti, Wierman, and Mazumdar]{panaganti2024model}
Panaganti, K., Wierman, A., and Mazumdar, E.
\newblock Model-free robust $\phi$-divergence reinforcement learning using both offline and online data.
\newblock In \emph{Proceedings of the 41st International Conference on Machine Learning}, pp.\  39324--39363, 2024{\natexlab{a}}.

\bibitem[Panaganti et~al.(2024{\natexlab{b}})Panaganti, Xu, Kalathil, and Ghavamzadeh]{panaganti2024bridging}
Panaganti, K., Xu, Z., Kalathil, D., and Ghavamzadeh, M.
\newblock Bridging distributionally robust learning and offline rl: An approach to mitigate distribution shift and partial data coverage.
\newblock In \emph{ICML 2024 Workshop: Foundations of Reinforcement Learning and Control--Connections and Perspectives}, 2024{\natexlab{b}}.

\bibitem[Satia \& Lave~Jr(1973)Satia and Lave~Jr]{satia1973markovian}
Satia, J.~K. and Lave~Jr, R.~E.
\newblock Markovian decision processes with uncertain transition probabilities.
\newblock \emph{Operations Research}, 21\penalty0 (3):\penalty0 728--740, 1973.

\bibitem[Shi \& Chi(2024)Shi and Chi]{shi2024distributionally}
Shi, L. and Chi, Y.
\newblock Distributionally robust model-based offline reinforcement learning with near-optimal sample complexity.
\newblock \emph{Journal of Machine Learning Research}, 25\penalty0 (200):\penalty0 1--91, 2024.

\bibitem[Shi et~al.(2024)Shi, Li, Wei, Chen, Geist, and Chi]{shi2024curious}
Shi, L., Li, G., Wei, Y., Chen, Y., Geist, M., and Chi, Y.
\newblock The curious price of distributional robustness in reinforcement learning with a generative model.
\newblock \emph{Advances in Neural Information Processing Systems}, 36, 2024.

\bibitem[Sutton \& Barto(2018)Sutton and Barto]{sutton2018reinforcement}
Sutton, R.~S. and Barto, A.~G.
\newblock \emph{Reinforcement learning: An introduction}.
\newblock MIT press, 2018.

\bibitem[Tamar et~al.(2014)Tamar, Mannor, and Xu]{tamar2014scaling}
Tamar, A., Mannor, S., and Xu, H.
\newblock Scaling up robust mdps using function approximation.
\newblock In \emph{International Conference on Machine Learning}, pp.\  181--189. PMLR, 2014.

\bibitem[Tsybakov(2009)]{tsybakov2009nonparametric}
Tsybakov, A.~B.
\newblock \emph{Introduction to Nonparametric Estimation}.
\newblock Springer, New York, 2009.

\bibitem[Vershynin(2010)]{vershynin2010introduction}
Vershynin, R.
\newblock Introduction to the non-asymptotic analysis of random matrices.
\newblock \emph{arXiv preprint arXiv:1011.3027}, 2010.

\bibitem[Vershynin(2018)]{vershynin2018high}
Vershynin, R.
\newblock \emph{High-dimensional probability: An introduction with applications in data science}, volume~47.
\newblock Cambridge university press, 2018.

\bibitem[Wang et~al.(2024{\natexlab{a}})Wang, Shi, and Chi]{wang2024sample}
Wang, H., Shi, L., and Chi, Y.
\newblock Sample complexity of offline distributionally robust linear markov decision processes.
\newblock In \emph{Reinforcement Learning Conference}, 2024{\natexlab{a}}.

\bibitem[Wang et~al.(2024{\natexlab{b}})Wang, Yang, Liu, Zhou, and Xu]{wang2024return}
Wang, R., Yang, Y., Liu, Z., Zhou, D., and Xu, P.
\newblock Return augmented decision transformer for off-dynamics reinforcement learning.
\newblock \emph{arXiv preprint arXiv:2410.23450}, 2024{\natexlab{b}}.

\bibitem[Xiong et~al.(2022)Xiong, Zhong, Shi, Shen, Wang, and Zhang]{xiong2022nearly}
Xiong, W., Zhong, H., Shi, C., Shen, C., Wang, L., and Zhang, T.
\newblock Nearly minimax optimal offline reinforcement learning with linear function approximation: Single-agent mdp and markov game.
\newblock \emph{arXiv preprint arXiv:2205.15512}, 2022.

\bibitem[Xu et~al.(2023)Xu, Panaganti, and Kalathil]{xu2023improved}
Xu, Z., Panaganti, K., and Kalathil, D.
\newblock Improved sample complexity bounds for distributionally robust reinforcement learning.
\newblock In \emph{International Conference on Artificial Intelligence and Statistics}, pp.\  9728--9754. PMLR, 2023.

\bibitem[Yang et~al.(2023)Yang, Wang, Kozuno, Jordan, and Zhang]{yang2023robust}
Yang, W., Wang, H., Kozuno, T., Jordan, S.~M., and Zhang, Z.
\newblock Robust markov decision processes without model estimation.
\newblock \emph{arXiv preprint arXiv:2302.01248}, 2023.

\bibitem[Zhang et~al.(2020)Zhang, Chen, Xiao, Li, Liu, Boning, and Hsieh]{zhang2020robust}
Zhang, H., Chen, H., Xiao, C., Li, B., Liu, M., Boning, D., and Hsieh, C.-J.
\newblock Robust deep reinforcement learning against adversarial perturbations on state observations.
\newblock \emph{Advances in Neural Information Processing Systems}, 33:\penalty0 21024--21037, 2020.

\bibitem[Zhang et~al.(2024)Zhang, Hu, and Li]{zhang2024soft}
Zhang, R., Hu, Y., and Li, N.
\newblock Soft robust mdps and risk-sensitive mdps: Equivalence, policy gradient, and sample complexity.
\newblock In \emph{The Twelfth International Conference on Learning Representations}, 2024.

\bibitem[Zhou et~al.(2021)Zhou, Zhou, Bai, Qiu, Blanchet, and Glynn]{zhou2021finite}
Zhou, Z., Zhou, Z., Bai, Q., Qiu, L., Blanchet, J., and Glynn, P.
\newblock Finite-sample regret bound for distributionally robust offline tabular reinforcement learning.
\newblock In \emph{International Conference on Artificial Intelligence and Statistics}, pp.\  3331--3339. PMLR, 2021.

\end{thebibliography}
\bibliographystyle{icml2025}

\newpage
\appendix
\onecolumn
\section{Additional Details on Experiments}
\label{supp:numerical experiment}
In this section, we provide details on experiment setup. %
\subsection{Simulated Linear MDPs}
\paragraph{Construction of the Simulated Linear MDP} We leverage the simulated linear MDP instance proposed by \citet{liu2024distributionally}. 
The state space is $\cS = \{x_1,\cdots,x_5\}$ and the action space is $\cA=\{-1,1\}^4\subset\RR^4$. At each episode, the initial state is always $x_1$. From $x_1$, the next state can be $x_2, x_4, x_5$ with probability defined on the arrows. Both $x_4$ and $x_5$ are absorbing states. $x_4$ is the fail state with 0 reward and $x_5$ is the goal state with reward 1.
The hyperparameter $\bxi\in\RR^4$ is designed to determine the reward functions and transition probabilities and $\delta$ is the parameter defined to determine the environment. We perturb the transition probability at the initial stage to construct the source environment. The extend of perturbation is controlled by the hyperparameter $q\in(0,1)$. For more details on the simulated linear DRMDP, we refer readers to the Supplementary A.1 in \citet{liu2024distributionally}.
\paragraph{Hyperparameters} The hyper-parameters in our setting are shown in \Cref{table: Hyper-parameters}. The horizon is 3, the $\beta, \gamma, \delta$ are set the same in all tasks, the $\| \bxi \|_1$ is set as 0.3, 0.2, 0.1 in \Cref{fig: simulated MDP} in order to illustrate the versatility of our algorithms.
\begin{table}[h]
    \centering
    \caption{Hyper-parameters. \label{table: Hyper-parameters}}    
    \begin{tabular}{lc}
        \toprule
        {Hyper-parameters} & {Value} \\
        \midrule
        $H$ (Horizon)& 3 \\
        $\beta$ (pessimism parameter)& 1 \\
        $\gamma$ & 0.1 \\
        $\delta$ & 0.3 \\
        $\|\bxi\|_1$ & 0.3, 0.2, 0.1 \\
        \bottomrule
    \end{tabular}
\end{table}

\subsection{Simulated American Put Option}
\label{experiment: APO-setup}
\paragraph{Construction of the Simulated American Put Option} In each episode, there are $H=20$ stages, and each state $h$, the dynamics evolves following the Bernoulli distribution:
\begin{align}
\label{American put option transit} 
    s_{h+1} = 
    \begin{cases}
        1.02s_h, \text{  w.p } p_0 \\
        0.98 s_h, \text{  w.p } 1-p_0 
    \end{cases},
\end{align}
where $p_0 \in (0,1)$ is the probability of price up.
At each step, the agent has two actions to take: exercise the option $a_h = 1$ or not exercise $a=0$. If exercising the option $a_h =0$, the agent will obtain reward $r_h = \max \{0, 100-s_h\}$ and the state comes to an end. If not exercising the option $a_h=1$, The state will continue to transit based on \eqref{American put option transit} and no reward will be received.
To implement our algorithms, we use  the following feature mapping: 
\begin{align*}
\phi\left(s_h, a\right)= \begin{cases}{\left[\varphi_1\left(s_h\right),\cdots, \varphi_d\left(s_h\right), 0\right]} & \text { if } a=1 \\ {\left[0, \cdots, 0, \max \left\{0,100-s_h\right\}\right]} & \text { if } a=0
\end{cases},
\end{align*}
where $\varphi_i(s)=\max \left\{0,1-\left|s_h-s_i\right| / \Delta\right\}$,  $\{s_i\}_{i=1}^d$ are anchor states, $s_1=80$  $s_{i+1}-s_i=\Delta$ and $\Delta = 60/d$.
For more details on the simulated American put option environment, we refer readers to the Appendix C of \citet{ma2022distributionally}.
\paragraph{Offline Dataset and Hyperparameters} We set $p_0=0.5$ in the nominal environment, from which trajectories are collected by fixed behavior policy, which chooses $a_h = 0$. The $\beta = 0.1$ and $\gamma = 1$ are set hyper-parameters in all tasks. For the time efficiency comparison in \Cref{fig: excution time w.r.t N} and \Cref{fig:exution time w.r.t. d}, we counted the time it took for the agent to train once and repeated 5 times to take the average.

\section{Proof of Properties of $d$-RRMDPs}
In this section, we provide the proofs of  results in \Cref{sec:Problem formulation,sec:Robust Regularized Pessimistic Value Iteration}, namely, the robust regularized Bellman equation, the existence of the optimal robust policy,
and the linear representation of the robust regularized Q-function under the $d$-rectangular linear RRMDP.
\subsection{Proof of \Cref{prop:regularized Robust Bellman equation}}
\begin{proof}
 We prove the a stronger proposition by induction from the last stage $H$. Specifically, besides the equations in \Cref{prop:regularized Robust Bellman equation} hold, we further assume that there exist transition kernels $\{ \hat{\bmu}_t \}_{t=1}^H, \hat{P}_t = \la \bphi, \hat{\bmu}_t \ra$, such that for any $(h,s)\in[H]\times\cS$,
\begin{align}
    V_h^{\pi, \lambda}(s) &= \EE^{\{\hat{P}_t \}_{t=h}^H} \bigg[\sum_{t = h}^H \big[r_t(s_t, a_t) + \lambda \la\bphi(s_t, a_t),\bD(\hat{\bmu}_t||\bmu_t^0) \ra \Big| s_h = s, \pi \bigg]. \label{existence of transition kernel}
\end{align}As there is no transitional kernel involved, the base case holds trivially. Suppose the conclusion holds for stage $h+1$, that is to say, there exists $\hat{P}_{t}, t=h+1, h+2,\cdots, H$ such that
 \begin{align*}
     V_{h+1}^{\pi,\lambda}(s) = \EE^{\{\hat{P}_i\}_{i=h+1}^H} \bigg[ \sum_{t = h+1}^H \big[r_t(s_t, a_t) + \lambda \la\bphi(s_t, a_t),\bD(\hat{\bmu}_t||\bmu_t^0)\big] 
 \Big| s_{h+1}= s, \pi \bigg].
 \end{align*}
 For the case of $h$, recall the definition of $Q_h^{\pi}$, we have 
 \begin{align}
     Q^{\pi, \lambda}_h(s,a) 
     & =\inf_{\bmu_t \in \Delta(\cS)^d,P_t=\la \bphi,\bmu_t \ra} \EE^{\{P_t\}_{t=h}^H} \bigg[\sum_{t = h}^H \big[r_t(s_t, a_t) + \lambda \la\bphi(s_t, a_t),\bD(\bmu_t||\bmu_t^0) \ra\big] \Big| s_h = s, a_h = a, \pi \bigg] \nonumber\\ 
  & = r_h(s,a) + \inf_{\bmu_t \in \Delta(\cS)^d,P_t=\la \bphi,\bmu_t \ra} \lambda \la\bphi(s_h, a_h),\bD(\bmu_h||\bmu_h^0)\ra  \nonumber\\   
  &\quad + \int_{\cS} P_{h}(ds'| s,a)\EE^{\{P_t\}_{t=h+1}^H} \bigg[ \sum_{t = h+1}^H \big[r_t(s_t, a_t) + \lambda \la\bphi(s_t, a_t),\bD(\bmu_t||\bmu_t^0)\ra\big] 
 \Big| s_{h+1} = s' , \pi \bigg] \nonumber\\
  & \leq r_h(s,a) + \inf_{\bmu_h \in \Delta(\cS)^d,P_h=\la \bphi,\bmu_h \ra}\lambda \la\bphi(s_h, a_h),\bD(\bmu_h||\bmu_h^0)\ra  \nonumber\\
  &\quad + \int_{\cS} P_{h}(ds'| s,a)\EE^{\{\hat{P}_t\}_{t=h+1}^H} \bigg[ \sum_{t = h+1}^H \big[r_t(s_t, a_t) + \lambda \la\bphi(s_t, a_t),\bD(\hat{\bmu}_t||\bmu_t^0)\ra\big] 
 \Big| s_{h+1} = s' , \pi \bigg] \nonumber \\
  & = r_h(s,a) + \inf_{\bmu_h \in \Delta(\cS)^d,P_h=\la \bphi,\bmu_h \ra}  \lambda \la\bphi(s_h, a_h),\bD(\bmu_h||\bmu_h^0)\ra + \EE_{s'\sim P_h(\cdot |s,a)}[V_{h+1}^{\pi, \lambda}(s')] \label{B2},
 \end{align}
    where \eqref{B2} follows by the 
    inductive hypothesis of $V_{h+1}^{\pi, \lambda}(s)$. On the other hand, we can lower bound $Q_h^{\pi, \lambda}(s,a)$ as
    \begin{align}
        &Q^{\pi, \lambda}_h(s,a) \nonumber \\
    & = r_h(s,a) + \inf_{\bmu_t \in \Delta(\cS)^d,P_t=\la \bphi,\bmu_t \ra} \lambda \la\bphi(s_h, a_h),\bD(\bmu_h||\bmu_h^0)\ra  \nonumber\\   
  &\quad + \int_{\cS} P_{h}(ds'| s,a)\EE^{\{P_t\}_{t=h+1}^H} \bigg[ \sum_{t = h+1}^H \big[r_t(s_t, a_t) + \lambda \la\bphi(s_t, a_t),\bD(\bmu_t||\bmu_t^0)\ra\big] 
 \Big| s_{h+1} = s' , \pi \bigg] \nonumber\\
    &\geq  r_h(s,a) + \inf_{\bmu_h \in \Delta(\cS)^d,P_h=\la \bphi,\bmu_h \ra} \lambda \la\bphi(s_h, a_h),\bD(\bmu_h||\bmu_h^0)\ra  \label{B3}\\
    & \quad + \int_{\cS} P_{h}^\pi(ds'| s,a) \inf_{\bmu_t \in \Delta(\cS)^d,P_t=\la \bphi,\bmu_t \ra} \EE^{\{P_t\}_{t=h+1}^H} \bigg[ \sum_{t = h+1}^H \big[r_t(s_t, a_t) + \lambda \la\bphi(s_t, a_t),\bD(\bmu_t||\bmu_t^0)\ra\big] 
 \Big| s_{h+1} = s' , \pi \bigg] \nonumber\\
    & = r_h(s,a) + \inf_{\bmu_h \in \Delta(\cS)^d,P_h=\la \bphi,\bmu_h \ra} \lambda \la\bphi(s_h, a_h),\bD(\bmu_h||\bmu_h^0)\ra + \EE_{s'\sim P_h(\cdot |s,a)}[V_{h+1}^{\pi, \lambda}(s')] \label{B4},
    \end{align} 
    where \eqref{B3} follows by the Fatou's lemma, \eqref{B4} follows by the definition of $V_{h+1}^{\pi, \lambda}(s)$. Hence, combining the two above inequalities, we conclude the proof of the first equation. Next we focus on the proof of the (\ref{existence of transition kernel}), by which we aim to proof the existence of transition kernel $\{\hat{P}_t\}_{t=h}^H$. By the fact that
    \begin{align*}
        Q_{h}^{\pi, \lambda}(s, a) &= r_h(s, a) + \inf_{\bmu_h \in \Delta(\cS)^d,P_h=\la \bphi,\bmu_h \ra}\big[\EE_{s'\sim P_h(\cdot|s,a)}\big[V_{h+1}^{\pi, \lambda}(s')\big]+ \lambda \la\bphi(s, a),\bD(\bmu_h||\bmu_h^0) \ra\big],
    \end{align*}
    we notice that the $\inf$ problem above is constraint by the distance $D$. Therefore by Lagrange duality and the closeness of distribution $\Delta(\cS)$, there exists $\hat{\bmu}_h \in \Delta(\cS)^d, \hat{P}_h = \la \bphi, \hat{\bmu}_h \ra$ such that
    \begin{align}
        Q^{\pi, \lambda}_h(s,a) =  r_h(s, a) + \EE_{s'\sim \hat{P}_h(\cdot|s,a)}\big[V_{h+1}^{\pi, \lambda}(s')\big]+ \lambda \la\bphi(s, a),\bD(\hat{\bmu}_h||\bmu_h^0) \ra.
        \label{(A.1)}
    \end{align}
    Now it remains to proof \eqref{regularized robust bellman equation: 2}. By the definition of $V_{h}^{\pi, \lambda}(s)$, we have
    \begin{align}
     &V_{h}^{\pi, \lambda}(s) \\
     &= \inf_{\bmu_t \in \Delta(\cS)^d,P_t=\la \bphi,\bmu_t \ra} \EE^{\{P_t\}_{t=h}^H} \bigg[\sum_{t = h}^H \big[r_t(s_t, a_t) + \lambda \la\bphi(s_t, a_t),\bD(\bmu_t||\bmu_t^0) \ra\big] \Big| s_h = s, \pi \bigg] \nonumber\\
    &=  \inf_{\bmu_t \in \Delta(\cS)^d,P_t=\la \bphi,\bmu_t \ra}\sum_{a \in \cA} \pi(a|s) \EE^{\{P_t\}_{t=h}^H} \bigg[\sum_{t = h}^H \big[r_t(s_t, a_t) + \lambda \la\bphi(s_t, a_t),\bD(\bmu_t||\bmu_t^0) \ra\big] \Big| s_h = s, a_h =a, \pi \bigg] \nonumber\\
    & \leq \sum_{a \in \cA} \pi(a|s) \EE^{\{\hat{P}_t\}_{t=h}^H} \bigg[ \sum_{t = h}^H \big[r_t(s_t, a_t) + \lambda \la\bphi(s_t, a_t),\bD(\hat{\bmu}_t||\bmu_t^0)\ra\big] 
    \Big| s_{h}= s, a_h = a, \pi \bigg] \nonumber\\
    & = \sum_{a \in \cA} \pi(a|s) Q_{h}^{\pi, \lambda}(s,a), \label{B6}
    \end{align}
    where \eqref{B6} comes from \eqref{(A.1)} and the inductive hypothesis. 
    On the other hand, by the definition of $Q_{h}^{\pi, \lambda}(s,a)$, we have
    \begin{align}
       &\sum_{a \in \cA} \pi(a|s) Q_{h}^{\pi, \lambda}(s,a) \nonumber\\
       &= \sum_{a \in \cA} \pi(a|s) \inf_{\bmu_t \in \Delta(\cS)^d,P_t=\la \bphi,\bmu_t \ra} \EE^{\{P_t\}_{t=h}^H} \bigg[\sum_{t = h}^H \big[r_t(s_t, a_t) + \lambda \la\bphi(s_t, a_t),\bD(\bmu_t||\bmu_t^0) \ra\big] \Big| s_h = s, a_h = a, \pi \bigg] \nonumber\\
        & \leq  \inf_{\bmu_t \in \Delta(\cS)^d,P_t=\la \bphi,\bmu_t \ra}\sum_{a \in \cA} \pi(a|s) \EE^{\{P_t\}_{t=h}^H} \bigg[\sum_{t = h}^H \big[r_t(s_t, a_t) + \lambda \la\bphi(s_t, a_t),\bD(\bmu_t||\bmu_t^0) \ra\big] \Big| s_h = s, a_h = a, \pi \bigg] \nonumber
        \\
        & = V_{h}^{\pi, \lambda}(s), \label{B5}
    \end{align}
    where \eqref{B5} comes from the definition of $V_h^{\pi, \lambda}(s)$. Combining the two inequalities \eqref{B6} and \eqref{B5}, we have 
    \begin{align*}
        V_{h}^{\pi, \lambda}(s) = \EE_{a \sim \pi(\cdot | s)}\big[Q_{h}^{\pi}(s,a)\big].
    \end{align*}
    This proves the \eqref{regularized robust bellman equation: 2} for stage h. Therefore, by using an induction argument, we finish the proof of \Cref{prop:regularized Robust Bellman equation}.
\end{proof}

\subsection{Proof of \Cref{prop:existence of optimal value function}}
    \begin{proof}
    We define the optimal stationary policy $\pi^\star = \{ \pi_h^\star \}_{h=1}^H$ as: for all $(h,s) \in [H]\times\cS$,
    \begin{align*}
        \pi^\star_h(s) = \argmax_{a \in \cA} \Big[r_h(s,a) + \inf_{\bmu_h \in \Delta(\cS)^d,P_h=\la \bphi,\bmu_h \ra}\big[\EE_{s'\sim P_h(\cdot|s,a)}\big[V_{h+1}^{\star, \lambda}(s')\big]+ \lambda \la\bphi(s, a),\bD(\bmu_h||\bmu_h^0) \ra\big]\Big].
    \end{align*}
    Now it remains to show that the regularized robust value function $V_h^{\pi^\star, \lambda}, Q_h^{\pi^\star, \lambda}$ induced by policy $\pi^\star$ is optimal, i.e., for all $(h,s) \in [H] \times \cS$,
    \begin{align*}
        V^{\pi^\star, \lambda}_h(s) = V^{\star, \lambda}_{h}(s), Q^{\pi^\star, \lambda}_h(s, a) = Q^{\star, \lambda}_h(s,a).
    \end{align*}
    By the \eqref{prop:regularized Robust Bellman equation}, we only need to prove the first equation above, then the optimality of the $Q$ holds trivially. we prove this statement by induction from $H$ to $1$. For stage $H$, the conclusion holds by:
    \begin{align*}
        V_H^{\star, \lambda}(s) &= 
        \sup_{\pi}V_H^{\pi, \lambda}(s)\\
        &= \sup_{\pi} \inf_{\bmu_H \in \Delta(\cS)^d,P_H=\la \bphi,\bmu_H \ra} \EE^{P_H} \bigg[\big[r_H(s_H, a_H) + \lambda \la\bphi(s_H, a_H),\bD(\bmu_H||\bmu_H^0) \ra\big] \Big| s_H = s, \pi \bigg] \\
        & = \sup_{\pi} \Big[ r_H(s_H,\pi_H(s_H)) +  \inf_{\bmu_H \in \Delta(\cS)^d,P_H=\la \bphi,\bmu_H \ra} \lambda \la\bphi(s, a),\bD(\bmu_H||\bmu_H^0) \ra \Big] \\
        & = V_H^{\pi^\star, \lambda}(s).
    \end{align*}
    Now assume that the conclusion holds by stage $h+1$. Hence, we have that for all $s \in \cS$,
    \begin{align*}
        V_{h+1}^{\pi^\star, \lambda}(s) = V_{h+1}^{\star, \lambda}(s).
    \end{align*}
    For the case of $h$, by \eqref{prop:regularized Robust Bellman equation}, we have 
    \begin{align}
        &V_{h}^{\pi^\star, \lambda}(s) \nonumber\\&= \EE_{a \sim \pi^\star_{h}(\cdot|s)} \Big[ Q_h^{\pi^\star, \lambda}(s,a) \Big] \nonumber\\
        & = \EE_{a \sim \pi^\star_{h}(\cdot|s)} \Big[r_h(s, a) + \inf_{\bmu_h \in \Delta(\cS)^d,P_h=\la \bphi,\bmu_h \ra}\big[\EE_{s'\sim P_h(\cdot|s,a)}\big[V_{h+1}^{\pi^\star, \lambda}(s')\big]+ \lambda \la\bphi(s, a),\bD(\bmu_h||\bmu_h^0) \ra\big] \Big] \nonumber\\
        & = \EE_{a \sim \pi^\star_{h}(\cdot|s)} \Big[r_h(s, a) + \inf_{\bmu_h \in \Delta(\cS)^d,P_h=\la \bphi,\bmu_h \ra}\big[\EE_{s'\sim P_h(\cdot|s,a)}\big[V_{h+1}^{\star, \lambda}(s')\big]+ \lambda \la\bphi(s, a),\bD(\bmu_h||\bmu_h^0) \ra\big] \Big] \label{B10}\\
        & = \max_{a \in \cA}  \Big[r_h(s, a) + \inf_{\bmu_h \in \Delta(\cS)^d,P_h=\la \bphi,\bmu_h \ra}\big[\EE_{s'\sim P_h(\cdot|s,a)}\big[V_{h+1}^{\star, \lambda}(s')\big]+ \lambda \la\bphi(s, a),\bD(\bmu_h||\bmu_h^0) \ra\big] \Big], \label{i}
    \end{align}
    where \eqref{B10} holds by the inductive hypothesis, \eqref{i} holds by the definition of $\pi^\star$. On the other hand, recall the definition of $V^{\star,\lambda}_h(s)$, then for any $s \in \cS$, by \eqref{prop:regularized Robust Bellman equation} we have 
    \begin{align}
        &V_h^{\star,\lambda}(s) \nonumber
        \\&= \sup_{\pi}V_h^{\pi, \lambda}(s) \nonumber\\
        &= \sup_{\pi}\EE_{a \sim \pi_{h}(\cdot|s)} \Big[ Q_h^{\pi, \lambda}(s,a) \Big] \nonumber\\
        &  = \sup_{\pi} \EE_{a \sim \pi_{h}(\cdot|s)} \Big[r_h(s, a) + \inf_{\bmu_h \in \Delta(\cS)^d,P_h=\la \bphi,\bmu_h \ra}\big[\EE_{s'\sim P_h(\cdot|s,a)}\big[V_{h+1}^{\pi, \lambda}(s')\big]+ \lambda \la\bphi(s, a),\bD(\bmu_h||\bmu_h^0) \ra\big] \Big] \nonumber\\
        & \leq \sup_{\pi} \EE_{a \sim \pi_{h}(\cdot|s)} \Big[r_h(s, a) + \inf_{\bmu_h \in \Delta(\cS)^d,P_h=\la \bphi,\bmu_h \ra}\big[\EE_{s'\sim P_h(\cdot|s,a)}\big[V_{h+1}^{\star, \lambda}(s')\big]+ \lambda \la\bphi(s, a),\bD(\bmu_h||\bmu_h^0) \ra\big] \Big] \label{B11}\\
        & = \max_{a \in \cA} \Big[r_h(s, a) + \inf_{\bmu_h \in \Delta(\cS)^d,P_h=\la \bphi,\bmu_h \ra}\big[\EE_{s'\sim P_h(\cdot|s,a)}\big[V_{h+1}^{\star, \lambda}(s')\big]+ \lambda \la\bphi(s, a),\bD(\bmu_h||\bmu_h^0) \ra\big] \Big] \nonumber\\
        & = V_{h}^{\pi^\star, \lambda}(s), \label{B12}
    \end{align}
    where \eqref{B11} holds by the fact that $V^{\pi^\star, \lambda}_{h+1}(s) \leq V^{\star,\lambda}_{h+1}(s) ,\forall s \in \cS$, \eqref{B12} holds by (\ref{i}). In turn, we trivially have $V_h^{\star, \lambda}(s) \geq V_h^{\pi^\star, \lambda}(s)$ due to the optimality of the value function. Hence, we obtain $V_{h}^{\pi^\star, \lambda}(s) = V_h^{\star, \lambda}(s), \forall s \in \cS$. Therefore, by the induction argument, we conclude the proof.
    \end{proof}
\subsection{Proof of \Cref{prop:linear of bellman operator}}
    \begin{proof}
    By \Cref{prop:regularized Robust Bellman equation}, we have 
    \begin{align}
        Q_h^{\pi, \lambda}(s,a) \nonumber
        &= r_h(s, a) + \inf_{\bmu_h \in \Delta(\cS)^d,P_h=\la \bphi,\bmu_h \ra}\big[\EE_{s'\sim P_h(\cdot|s,a)}\big[V_{h+1}^{\pi, \lambda}(s')\big]+ \lambda \la\bphi(s, a),\bD(\bmu_h||\bmu_h^0) \ra\big] \nonumber\\
        & = \big\langle \bphi(s,a) , \btheta_h \big\rangle + \inf_{\bmu_h \in \Delta(\cS)^d} \Big[ \big\langle \bphi(s,a), \EE_{s' \sim \bmu_h}[V_{h+1}^{\pi, \lambda}(s')] \big\rangle + \lambda \sum_{i=1}^d\phi_i(s,a)D(\mu_{h, i} \| \mu_{h,i}^0) \Big] \nonumber\\
        & = \big\langle \bphi(s,a) ,\btheta_h \big\rangle +\big\langle \bphi(s,a) , \bw^{\pi, \lambda}_h \big\rangle \nonumber\\
        & =\big\langle \bphi(s,a) ,\btheta_h + \bw_h^{\pi, \lambda} \big\rangle.\nonumber
    \end{align}
    Hence we conclude the proof.
\end{proof}

\subsection{Proof of \Cref{prop:the regularized duality under TV divergence}}
\begin{proof}
    The optimization problem can be formalized as: 
   \begin{align*}
       \inf_{\mu}\EE_{s\sim\mu}V(s) + \lambda D_{\text{TV}}(\mu\|\mu^0) ~~
        \text{subject to} \sum_s\mu(s) = 1, \mu(s) \geq 0.
   \end{align*}
   Denote $y(s) = \mu(s) - \mu^0(s)$, the objective function can be rewritten as:
   \begin{align*}
    \EE_{s\sim\mu}V(s) + \lambda D_{\text{TV}}(\mu\| \mu^0) 
    &=\sum_{s} \mu(s)V(s) + \lambda/2\sum_s|\mu(s)-\mu^0(s)| \\ 
    &= \sum_{s} V(s)(y(s) + \mu^0(s)) + \lambda/2 \sum |y(s)| \\&= \EE_{s\sim\mu^0}V(s)+ \sum_{s} V(s)y(s) + \lambda /2 \sum |y(s)|.
   \end{align*}
    Recall the constraint 
   $\sum_{s}y(s) = 0, y(s) \geq -\mu^0(s)$, by the Lagrange duality, we establish the Lagrangian function:
   \begin{align*}
       \cL = \min_y\max_{\mu \geq 0, r \in R}\Big(\sum_s [y(s)(V(s)-\mu(s)-r)) + \lambda/2|y(s)|] - \sum_s \mu(s)\mu^0(s) \Big).
   \end{align*}
   In order to achieve the minimax optimality, for any $s$, term $ y(s)(V(s)-\mu(s)-r)) + \lambda/2|y(s)|$ should obtain a bounded lower bound with respect to $y(s)$, which requires that $ \mu(s), r$ should satisfy the following conditions:
   \begin{align*}
       \forall s \in \cS, |V(s) - \mu(s) - r| \leq \lambda/2 
        \Rightarrow \max_{s \in \cS}\{V(s)-\mu(s)\} - \min_{s \in \cS}\{V(s) - \mu (s)\} \leq \lambda.
   \end{align*}
   With the constraint above, we denote $g(s): = V(s) - \mu(s)$, we have %
   \begin{align*}
       \cL &= \max_{\mu \geq 0, r \in R}\min_y\Big \{\sum_s [y(s)(V(s)-\mu(s)-r)) + \lambda/2|y(s)|] - \sum_s \mu(s)\mu^0(s) \Big \} \\
       &= \max_{\max_{s \in \cS}(V(s)-\mu(s)) - \min_{s \in \cS}(V(s) - \mu (s)) \leq \lambda} - \sum\mu(s)\mu^0(s) \\
       & = \max_{\max_sg(s) - \min_sg(s) \leq \lambda, g(s) \leq V(s)}\Big \{ \sum g(s)\mu^0(s)\Big \} - \EE_{s\sim\mu^0}V(s).
   \end{align*}
    Thus we have, %
   \begin{align}
       &\EE_{s\sim\mu}V(s) + \lambda D_{\text{TV}}(\mu\| \mu^0) \nonumber\\
       &= \EE_{s\sim\mu^0}V(s) +\max_{\max_sg(s) - \min_sg(s) \leq \lambda, g(s) \leq V(s)}\Big \{\sum g(s)\mu^0(s)\Big \} -\EE_{s\sim\mu^0}V(s) \nonumber \\
       &= \EE_{s\sim\mu^0}[V(s)]_{V_{\min}+ \lambda}, \label{3.2}
   \end{align}
    where \eqref{3.2} holds by directly solving the max problem. Hence we conclude the proof.
\end{proof}

\subsection{Proof of \Cref{prop:the regularized duality under x2 divergence}}
\begin{proof}
     Similar to the proof of \Cref{prop:the regularized duality under TV divergence}, define $y(s) = \mu(s) - \mu^0(s)$, with lagrange duality, we have:
     \begin{align*}
         \mathcal{L} &= \sum_{s} V(s) (y(s) + \mu^0(s)) + \lambda \sum_{s} \frac{y(s)^2}{\mu^0(s)} - \sum_{s} (\mu^0(s) + y(s)) \mu(s) - r \sum_{s} y(s) \\
        &= \sum_{s} \Big( \frac{\lambda y(s)^2}{\mu^0(s)} + y(s) (V(s) - \mu(s) - r) + \sum_{s} V(s) \mu^0(s) - \sum_{s} \mu^0(s) \mu(s) \Big).
     \end{align*}
     Noticing that $\cL$ is a quadratic function with respect to $y(s)$, therefore after we fix the term $\mu(s), r$ and compute the min with respect to $y(s)$, we have
     \begin{align}
         \mathcal{L}
        &= -\frac{1}{4 \lambda} \sum_{s} \mu^0(s) (V(s) - \mu(s) - r)^2 + \sum_{s} V(s) \mu^0(s) - \sum_{s} \mu^0(s) \mu(s) \nonumber\\
        & = -\frac{1}{4 \lambda} \Big[ \EE_{s\sim\mu^0} [V - \mu]^2 - \big(\EE_{s\sim\mu^0} [V - \mu] \big)^2 \Big] + \EE_{s\sim\mu^0} [V - \mu] \label{B3.3}\\
        & = \EE_{s\sim\mu^0} [V - \mu] - \frac{1}{4 \lambda} \Var_{s\sim\mu^0}[V - \mu] \nonumber\\
        &= \sup_{\alpha \in [V_{\min}, V_{\max}]} \EE_{s\sim\mu^0}[V(s)]_{\alpha} - \frac{1}{4\lambda}\Var_{s\sim\mu^0}[V(s)]_{\alpha} \label{B3.4},
     \end{align}
     where \eqref{B3.3} comes from maximizing $r$, and \eqref{B3.4} comes from maximizing $\mu(s), s \in \cS$ and the observation that $\mu(s) = 0 \text{ or } V(s), \forall s \in \cS$ when achieving its maximum. Hence, we conclude the proof.
    \end{proof}

\section{Proof of the Upper Bounds of Suboptimality}
In this section, we prove \Cref{thm: instance-dependent upper bounds} and \Cref{corollary:simplified upper bounds}. For simplicity, we denote $\bphi_h^{\tau} = \bphi(s_h^{\tau}, a_h^{\tau})$. According to the robust regularized Bellman equation in \Cref{prop:regularized Robust Bellman equation}, we first define the robust regularized  Bellman operator: for any $(h,s,a) \in [H]\times\cS\times\cA$ and any function $V:\cS\times\cA\rightarrow [0,H]$, 
\begin{align}
    \cT^{\lambda} _hV(s,a):= r_h(s, a) + \inf_{\bmu_h \in \Delta(\cS)^d,P_h=\la \bphi,\bmu_h \ra}\big[\EE_{s'\sim P_h(\cdot|s,a)}\big[V_{h+1}^{\pi, \lambda}(s')\big]+ \lambda \la\bphi(s, a),\bD(\bmu_h||\bmu_h^0) \ra\big].
    \label{def: regularized bellman operator}
\end{align}
We have $Q_h^{\pi, \lambda}(s,a) = \cT^{\lambda} _hV_{h+1}^{\pi, \lambda}(s,a)$.

\subsection{Proof of \Cref{thm: instance-dependent upper bounds}}
We start from bounding the suboptimality gap by the estimation uncertainty in the following Lemma.

\begin{lemma}
 \label{lemma:regularized robust suboptimality lemma}
     If the following inequality holds for any $(h,s,a) \in [H] \times \cS \times \cA$:
     \begin{align*}
        |\mathcal{T}_h^{\lambda}\hat{V}^{\lambda}_{h+1}(s,a) - \la \bphi(s,a), \hat{\bw}_h^\lambda \ra| \leq \Gamma_h(s,a), 
     \end{align*}
     then we have 
     \begin{align*}
         \text{SupOpt}(\hat{\pi}, s, \lambda) \leq 
        2\sup_{P \in \mathcal{U}^\lambda(P^0)}\sum_{h=1}^H\EE^{\pi^\star, P}\big[\Gamma_h(s_h, a_h)|s_1 = s\big].
     \end{align*}
\end{lemma}

\subsubsection{Proof of \Cref{thm: instance-dependent upper bounds} - Case with the TV Divergence}
\begin{algorithm}[ht]
    \caption{Robust Regularized Pessimistic Value Iteration under TV distance (\algname-TV) }\label{alg:R2PVI-TV}
    \begin{algorithmic}[1]
    \REQUIRE{
        Dataset $\cD$, regularizer $\lambda >0$, $\gamma>0$ and parameter $\beta$
        }
        \STATE init $\hat{V}^\lambda_{H+1}(\cdot)=0$
        \FOR {episode $h=H, \cdots, 1$}{
            \STATE $\bLambda_h \leftarrow \sum_{\tau = 1}^K \bphi(s_h^{\tau}, a_h^{\tau})(\bphi(s_h^{\tau}, a_h^{\tau}))^{\top} + \gamma \mathbf{I}$
            \STATE $\alpha_{h+1} \leftarrow \min_{s \in \cS}\{\hat{V}^\lambda_{h+1}(s) \} + \lambda$
            \STATE $\hat{\bw}_h^{\lambda}\leftarrow \bLambda_h^{-1}(\sum_{\tau=1}^{K} \bphi(s_h^{\tau}, a_h^{\tau}) [\hat{V}^\lambda_{h+1}(s_{h+1}^{\tau})]_{\alpha_{h+1}})$ \hfill $\triangleleft$ {\color{blue} Estimated by \eqref{TV-estimation}}
            \STATE $\Gamma_h(\cdot, \cdot) \leftarrow 
            \beta \sum_{i=1}^d\|\phi_i(\cdot, \cdot) \mathbf{1}_i\| _{\bLambda_h^{-1}}$
            \STATE $\hat{Q}_h^{\lambda}(\cdot, \cdot) \leftarrow \min\{\bphi(\cdot, \cdot)^{\top}(\btheta_h+ \hat{\bw}_h^{\lambda}) - \Gamma_h(\cdot, \cdot), H-h+1\}^+$
            \STATE $\hat{\pi}_h(\cdot | \cdot) \leftarrow \argmax_{\pi_h}\la\hat{Q}_h^{\lambda}(\cdot, \cdot), \pi_h(\cdot| \cdot)\ra_{\mathcal{A}}$ and $\hat{V}^{\lambda}_h(\cdot) \leftarrow\la\hat{Q}_h^{\lambda}(\cdot, \cdot), \hat{\pi}_h(\cdot|\cdot)\ra_{\mathcal{A}}$
        }\ENDFOR
    \end{algorithmic}  
\end{algorithm}

For completeness, we present \algname\  specific to the TV distance in \Cref{alg:R2PVI-TV}, which gives a closed form solution of \eqref{TV-estimation}. Now we present the upper bound of weights as follows.
        \begin{lemma}
            \label{bound of weights in TV distance}
            (Bound of weights - TV) For any $h \in [H]$, we have
            \begin{align*}
                \|\bw_h^\lambda\|_2 \leq H\sqrt{d}, \|\hat{\bw}_h^{\lambda}\|_2 \leq H\sqrt{\frac{Kd}{\gamma}}.
            \end{align*}
        \end{lemma}
    \begin{proof}[Proof of \Cref{thm: instance-dependent upper bounds} - TV]
    The \algname\ with TV-divergence is presented in \Cref{alg:R2PVI-TV}. We derive the upper bound on the estimation uncertainty $\Gamma_h(s,a)$ to prove the theorem. 
    We first decompose the difference between the regularized robust bellman operator $\mathcal{T}_h^{\lambda}$ and the empirical regularized robust bellman operator $\hat{\mathcal{T}}_h^{\lambda}$ as
    \begin{align}
         &\big|\mathcal{T}_h^{\lambda} \hat{V}^{\lambda}_{h+1}(s,a)-\la \bphi(s,a), \hat{\bw}_h^\lambda \ra\big| 
            \\&= \Big|\sum_{i=1}^d \phi_i(s,a)(w_{h,i}^\lambda - \hat{w}_{h,i}^{\lambda})\Big| \nonumber\\
           &= \Big|\sum_{i =1}^d\phi_i(s,a)\mathbf{1}_i(\bw_h^\lambda - \hat{\bw}_{h}^{\lambda}) \Big|\nonumber\\
           &= \Big| \gamma\sum_{i =1 }^d\phi_i(s,a)\mathbf{1}_i\bLambda_h^{-1}\bw_h^\lambda + \sum_{i =1 }^d\phi_i(s,a)\mathbf{1}_i\bLambda_h^{-1} \sum_{\tau =1}^K\bphi(s_h^\tau, a_h^{\tau})\eta_h^{\tau}([\hat{V}^{\lambda}_{h+1}(s)]_{\alpha_{h+1}})\Big| \label{C1.1}
           \\
           & \leq \gamma \sum_{i =1 }^d\|\phi_i(s,a)\mathbf{1}_i\|_{\bLambda_h^{-1}}\underbrace{\|\bw_h^\lambda\|_{\bLambda_h^{-1}}}_{\text{(i)}} + \sum_{i =1 }^d\|\phi_i(s,a)\mathbf{1}_i\|_{\bLambda_h^{-1}} \underbrace{\|\sum_{\tau =1}^K\bphi(s_h^\tau, a_h^{\tau})\eta_h^{\tau}([\hat{V}^{\lambda}_{h+1}(s)]_{\alpha_{h+1}})\|_{\bLambda_h^{-1}}}_{\text{(ii)}},
           \label{upperbound in TV}
    \end{align}
where \eqref{C1.1} comes from the definition of $\hat{\bw}_h^\lambda$, while \eqref{upperbound in TV} follows by the Cauchy-Schwartz inequality. By \Cref{bound of weights in TV distance} and the fact that $\hat{V}^{\lambda}_{h+1}(s) \leq H$ and $\gamma = 1$, we have
\begin{align*}
    \text{(i)} = \|\bw_h^\lambda\|_{\bLambda_h^{-1}}\leq \|\bLambda_h^{-1}\|^{1/2} \|\bw_h^\lambda\|_2 \leq H\sqrt{d},
\end{align*}
where the last inequality comes from the fact that $\|\bLambda_h^{-1}\| \leq \gamma^{-1}$. Now it remains to bound term (ii), as  $\hat{V}^{\lambda}_{h+1}$ depends on data, which makes it difficult to bound it directly by concentration equality. Instead, we consider focus on the function class $\mathcal{V}_{h}(R_0, B_0, \gamma)$:
         \begin{align*}
            \mathcal{V}_{h}(R_0, B_0, \gamma) = \{ V_h(x; \btheta, \beta, \bLambda): \mathcal{S} \rightarrow [0, H], \|\btheta\|_2 \leq R_0, \beta \in [0, B_0], \gamma_{\min}(\bLambda_h) \geq \gamma \},
         \end{align*}
         \noindent
         where $V_h(x; \btheta, \beta, \bLambda) = \max_{a \in \mathcal{A}}[\bphi(s,a)^\top\btheta - \beta \sum_{i=1}^d\|\phi_i(s,a)\|_{\bLambda_h^{-1}}]_{[0, H-h+1]}$. By \Cref{bound of weights in TV distance} and the definition of $\hat{V}_{h+1}^\lambda$, when we set $R_0 = H\sqrt{Kd/\gamma}, B_0 = \beta = 16 Hd \sqrt{\xi_{\text{TV}}}$, it suffices to show that $\hat{V}_{h+1}^\lambda \in \mathcal{V}_{h+1}(R_0, B_0, \gamma)$. 
         Next we aim to find a union cover of the $\mathcal{V}_{h+1}(R_0, B_0, \gamma)$, hence the term (ii) can be upper bounded. Let $\mathcal{N}_h(\epsilon; R_0, B_0, \gamma)$ be the minimum $\epsilon$-cover of $\mathcal{V}_h(R_0, B_0, \lambda)$ with respect to the supreme norm, $\mathcal{N}_{h}([0, H])$ be the minimum $\epsilon$-cover of $[0, H]$ respectively. In other words, for any function $V \in \mathcal{V}_{h}(R_0, B_0, \gamma), \alpha_{h+1} \in [0, H]$, there exists a function $V' \in \mathcal{V}_{h}(R_0, B_0, \gamma)$ and a real number $\alpha_{\epsilon} \in [0,H]$ such that:
         \begin{align*}
            \sup_{s \in \mathcal{S}}|V(s) - V'(s)| \leq \epsilon, |\alpha_{\epsilon} - \alpha_{h+1}| \leq \epsilon.
         \end{align*}
By Cauchy-Schwartz inequality and the fact that $\|a + b\|_{\bLambda_h^{-1}}^2 \leq 2\|a\|_{\bLambda_h^{-1}}^2 + 2\|b\|_{\bLambda_h^{-1}}^2$ and the definition of the term (ii), we have
         \begin{align}
            \text{(ii)}^2 &\leq 2 \Big\|\sum_{\tau = 1}^K \bphi(s_h^\tau, a_h^{\tau})\eta_{h}^{\tau}([\hat{V}^{\lambda}_{h+1}]_{\alpha_{\epsilon}})\Big\|^2_{\bLambda_h^{-1}} + 2 \Big\|\sum_{\tau = 1}^K \bphi(s_h^\tau, a_h^{\tau})\eta_{h}^{\tau}([\hat{V}^{\lambda}_{h+1}]_{\alpha_{h+1}} - [\hat{V}^{\lambda}_{h+1}]_{\alpha_{\epsilon}})\Big\|^2_{\bLambda_h^{-1}} \nonumber \\
            & \leq 4 \Big\|\sum_{\tau = 1}^K \bphi(s_h^\tau, a_h^{\tau})\eta_{h}^{\tau}([V'_{h+1}]_{\alpha_{\epsilon}})\Big\|^2_{\bLambda_h^{-1}} + 4 \Big\|\sum_{\tau = 1}^K \bphi(s_h^\tau, a_h^{\tau})\eta_{h}^{\tau}([\hat{V}^{\lambda}_{h+1}]_{\alpha_{\epsilon}} - [V'_{h+1}]_{\alpha_{\epsilon}})\Big\|^2_{\bLambda_h^{-1}} + \frac{2\epsilon^2K^2}{\gamma},
            \label{(eq.1)}
         \end{align}
 where \eqref{(eq.1)} follows by the fact that 
         \begin{align*}
            2 \Big\|\sum_{\tau = 1}^K \bphi(s_h^\tau, a_h^{\tau})\eta_{h}^{\tau}([\hat{V}^{\lambda}_{h+1}]_{\alpha_{h+1}} - [\hat{V}^{\lambda}_{h+1}]_{\alpha_{\epsilon}})\Big\|^2_{\bLambda_h^{-1}} \leq 2 \epsilon^2 \sum_{\tau = 1, \tau' = 1}^{K}|\bphi_h^\tau\bLambda_h^{-1}\bphi_h^{\tau\top}| \leq \frac{2\epsilon^2K^2}{\gamma}.
         \end{align*}
Meanwhile, by the fact that $|[\hat{V}^{\lambda}_{h+1}]_{\alpha_{\epsilon}} - [V'_{h+1}]_{\alpha_{\epsilon}}| \leq |\hat{V}^{\lambda}_{h+1} - V'_{h+1}|$, we have
         \begin{align}
            &4 \Big\|\sum_{\tau = 1}^K \bphi(s_h^\tau, a_h^{\tau})\eta_{h}^{\tau}([\hat{V}^{\lambda}_{h+1}]_{\alpha_{\epsilon}} - [V'_{h+1}]_{\alpha_{\epsilon}})\Big\|^2_{\bLambda_h^{-1}}\\
            &\leq 4 \sum_{\tau = 1, \tau' = 1}^{K}|\bphi_h^\tau\bLambda_h^{-1}\bphi_h^{\tau'}|
            \max|\eta_{h}^{\tau}([\hat{V}^{\lambda}_{h+1}]_{\alpha_{\epsilon}} - [V'_{h+1}]_{\alpha_{\epsilon}})|^2 \nonumber\\
            &\leq 4 \epsilon^2\sum_{\tau = 1, \tau' = 1}^{K}|\bphi_h^\tau\bLambda_h^{-1}\bphi_h^{\tau'}| \nonumber\\
            &\leq \frac{4\epsilon^2K^2}{\gamma}. \label{TV-100}
         \end{align}
By applying the (\ref{TV-100}) into (\ref{(eq.1)}), we have
         \begin{align}
            (\text{ii})^2 \leq 4 \sup_{V' \in \mathcal{N}_{h}(\epsilon; R_0, B_0, \gamma), \alpha_{\epsilon} \in \mathcal{N}_{h}([0, H])}\Big\|\sum_{\tau = 1}^K \bphi(s_h^\tau, a_h^{\tau})\eta_{h}^{\tau}([V'_{h+1}]_{\alpha_{\epsilon}})\Big\|^2_{\bLambda_h^{-1}} + \frac{6\epsilon^2K^2}{\gamma}.
            \label{TV-1}
         \end{align}
By \Cref{lemma:concentration bound}, applying a union bound over $\mathcal{N}_{h}(\epsilon; R_0, B_0, \gamma)$ and $\mathcal{N}_{h}([0, H])$, with probability at least $1 - \delta / 2H$, we have
         \begin{align}
            &4 \sup_{V' \in \mathcal{N}_{h}(\epsilon; R_0, B_0, \gamma), \alpha_{\epsilon} \in \mathcal{N}_{h}([0, H])}\Big\|\sum_{\tau = 1}^K \bphi(s_h^\tau, a_h^{\tau})\eta_{h}^{\tau}([V'_{h+1}]_{\alpha_{\epsilon}})\Big\|^2_{\bLambda_h^{-1}} + \frac{6\epsilon^2K^2}{\gamma} \nonumber\\
            &\leq 4 H^2\Big(2\log\frac{2H|\mathcal{N}_{h}(\epsilon; R_0, B_0, \gamma)\|\mathcal{N}_{h}([0, H])|}{\delta} + d\log(1+ K/\gamma)\Big) + \frac{6\epsilon^2K^2}{\gamma} .
            \label{TV-2}
         \end{align}
Applying \Cref{lemma:covering number of function class}, we have
    \begin{align}
        \log|\mathcal{N}_h(\epsilon;R_0, B_0, \lambda)| &\leq d\log(1+4R_0/\epsilon) + d^2\log(1+8d^{1/2}B_0^2/\gamma \epsilon^2) \nonumber\\
        &= d \log(1+ 4K^{3/2}d^{-1/2}) + d^2\log(1+ 8d^{-3/2}B_0^2K^2H^{-2}) \nonumber\\
        & \leq 2d^2\log(1+8d^{-3/2}B_0^2K^2H^{-2}).
        \label{TV-3}
    \end{align}
Similarly, by \Cref{lemma:covering number of interval}, we have
    \begin{align*}
        |\mathcal{N}_h([0, H])| \leq 3H/\epsilon.
    \end{align*}
    Combining (\ref{TV-3}) with (\ref{TV-1}) and (\ref{TV-2}), by setting $\epsilon = dH/K$, we have
    \begin{align*}
        \text{(ii)}^2 &\leq  4 H^2\Big(2\log\frac{2H|\mathcal{N}_{h}(\epsilon; R_0, B_0, \gamma)\|\mathcal{N}_{h}([0, H])|}{\delta} + d\log(1+ K/\gamma)\Big) + \frac{6\epsilon^2K^2}{\gamma} \\
        &\leq 4H^2(4d^2 \log(1 + 8d^{-3/2} B_0^2 K^2 H^{-2}) + \log(3K/d) + d \log(1+ K) + 2\log 2H/\delta) + 6d^2H^2 \\
        & \leq 16H^2d^2 (\log(1 + 8d^{-3/2} B_0^2 K^2 H^{-2})  + \log(1+K)/d + 3/8 + \log H/\delta) \\
        & \leq 32H^2 d^2 \log 8d^{-3/2} B_0^2 K^2 H^{-1}/\delta \\
        & = 32H^2 d^2 \log 1024Hd^{1/2} K^2 \xi_{\text{TV}} /\delta\\
        &= 32H^2 d^2 (\log 1024Hd^{1/2} K^2/\delta  + \log \xi_{\text{TV}}) \\
        & \leq 64 H^2 d^2 \xi_{\text{TV}} := \frac{\beta^2}{4}.
    \end{align*}
        Recall the upper bound in (\ref{upperbound in TV}), we have with probability at least $1- \delta$,
    \begin{align}
         &|\mathcal{T}_h^{\lambda} \hat{V}^{\lambda}_{h+1}(s,a)-\la \bphi(s,a), \hat{\bw}_h^\lambda \ra| \nonumber\\
           & \leq \gamma \sum_{i =1 }^d\|\phi_i(s,a)\mathbf{1}_i\|_{\bLambda_h^{-1}}\|\bw_h^\lambda\|_{\bLambda_h^{-1}} + \sum_{i =1 }^d\|\phi_i(s,a)\mathbf{1}_i\|_{\bLambda_h^{-1}}\Big\|\sum_{\tau =1}^K\bphi(s_h^\tau, a_h^{\tau})\eta_h^{\tau}([\hat{V}^{\lambda}_{h+1}(s)]_{\alpha_{\epsilon}})\Big\|_{\bLambda_h^{-1}} \nonumber\\
           & \leq \sum_{i =1 }^d\|\phi_i(s,a)\mathbf{1}_i\|_{\bLambda_h^{-1}}(H\sqrt{d} + \beta/2) \nonumber\\
           &\leq \beta\sum_{i =1 }^d\|\phi_i(s,a)\mathbf{1}_i\|_{\bLambda_h^{-1}} ,\label{C1.2}
    \end{align}
    where \eqref{C1.2} follows by the fact that $ 2H\sqrt{d} \leq \beta $. Hence, the prerequisite is satisfied in \Cref{lemma:regularized robust suboptimality lemma}, we can upper bound the suboptimality gap as:
     \begin{align*}
         \text{SubOpt}(\hat{\pi}, s, \lambda) &\leq 2\sup_{P \in \mathcal{U}^\lambda(P^0)}\sum_{h=1}^H\EE^{\pi^\star, P}\big[\Gamma_h(s_h, a_h)|s_1 = s\big] \\
         & = 2\beta\cdot\sup_{P \in \mathcal{U}^\lambda(P^0)}\sum_{h=1}^H\EE^{\pi^\star, P}\Big[\sum_{i=1}^d\|\phi_i(s,a)\mathbf{1}_i\|_{\bLambda_h^{-1}}|s_1 = s\Big].
     \end{align*}
     \noindent
     This concludes the proof.
    \end{proof}
\subsubsection{Proof of \Cref{thm: instance-dependent upper bounds} - Case with KL Divergence}
\begin{algorithm}[ht]
    \caption{Robust Regularized Pessimistic Value Iteration under KL distance (\algname-KL)}\label{alg:R2PVI-KL}
    \begin{algorithmic}[1]
    \REQUIRE{
        Dataset $\cD$, regularizer $\lambda>0, \gamma > 0$ and parameter $\beta$
        }
        \STATE init $\hat{V}_{H+1}^{\lambda}(\cdot)=0$ 
        \FOR {episode $h=H, \cdots, 1$}{
            \STATE $\bLambda_h \leftarrow \sum_{\tau = 1}^K \bphi(s_h^{\tau}, a_h^{\tau})(\bphi(s_h^{\tau}, a_h^{\tau}))^{\top} + \gamma \mathbf{I}$ 
            \STATE $\hat{\bw}'_h \leftarrow \bLambda_h^{-1}\big(\sum_{\tau=1}^{K} \bphi(s_h^{\tau}, a_h^{\tau} ) e^{-\frac{\hat{V}_{h+1}^{\lambda}(s^\tau_{h+1})}{\lambda}} \big)$ 
            \STATE $\hat{\bw}^\lambda_h \leftarrow -\lambda \log \max\{\hat{\bw}'_h, e^{-H/\lambda}\}$ \label{eq:modification}
            \STATE $\Gamma_h(\cdot, \cdot) \leftarrow  \beta \sum_{i=1}^d\|\phi_i(\cdot, \cdot) \mathbf{1}_i\|_{\Lambda_h^{-1}}$
            \STATE $\hat{Q}_h^{\lambda}(\cdot, \cdot) \leftarrow \min \{\bphi(\cdot, \cdot)^{\top}(\btheta_h + \hat{\bw}^\lambda_h) - \Gamma_h(\cdot, \cdot), H-h+1\}^+$
            \STATE $\hat{\pi}_h(\cdot | \cdot) \leftarrow \argmax_{\pi_h}\langle \hat{Q}_h^{\lambda}(\cdot, \cdot), \pi_h(\cdot| \cdot) \rangle_{\mathcal{A}}$ and $\hat{V}_h^{\lambda}(\cdot) \leftarrow \langle \hat{Q}_h^{\lambda}(\cdot, \cdot), \hat{\pi}_h(\cdot|\cdot) \rangle_{\mathcal{A}}$
        }\ENDFOR
    \end{algorithmic}
\end{algorithm}
For completeness, we present the \algname\ algorithm specific to the KL distance in \Cref{alg:R2PVI-KL}. Note that we have used the following closed form solution
\begin{align}
    \hat{\bw}'_h &= \argmin_{\bw \in \RR^d} \sum_{\tau=1}^K \Big(e^{-\frac{\hat{V}_{h+1}^{\lambda}(s^\tau_{h+1})}{\lambda}} - \bphi(s_h^\tau, a_h^{\tau})^{\top}\bw \Big)^2 + \gamma \| \bw \|_2^2. \\
    &=\bLambda_h^{-1}\big(\sum_{\tau=1}^{K} \bphi(s_h^{\tau}, a_h^{\tau} ) e^{-\frac{\hat{V}_{h+1}^{\lambda}(s^\tau_{h+1})}{\lambda}} \big).
\end{align}
Our proof relies on the following lemmas on bounding the regression parameter and $\epsilon$-covering number of the robust value function class. 
\begin{lemma}[Bound of weights - KL]
 \label{lemma:bound of weights in KL distance}
     For any $h \in [H]$,
     \begin{align*}
        \|\bw^\lambda_h\|_2 \leq \sqrt{d}, \|\hat{\bw}'_h\|_2 \leq \sqrt{\frac{Kd}{\gamma}}.
     \end{align*}
 \end{lemma}

\begin{lemma}[Bound of covering number - KL]
 \label{lemma: covering number of function class in KL}
     For any $h \in [H]$, let $\mathcal{V}_h$ denote a class of functions mapping from $\mathcal{S}$ to $\RR$ with the following form:
     \begin{align*}
        V_h(x; \btheta, \beta, \bLambda_h) = \max_{a \in \mathcal{A}} \Big\{ \bphi(s,a)^{\top} \btheta - \lambda \log\Big(1+\beta \sum_{i=1}^d\|\phi_i(\cdot, \cdot) \mathbf{1}_i\|_{\bLambda_h^{-1}}\Big) \Big\}_{[0, H-h+1]},
     \end{align*}
    the parameters $(\btheta, \beta, \bLambda_h) $ satisfy $\|\btheta\|_2 \leq L, \beta \in [0, B], \gamma_{\min}(\bLambda_h) \geq \gamma$. Let $\mathcal{N}_h(\epsilon)$ be the $\epsilon$-covering number of $\mathcal{V}$ with respect to the distance $\text{dist}(V_1, V_2) = \sup_x|V_1(x) - V_2(x)|$.Then
     \begin{align*}
        \log|\mathcal{N}_h(\epsilon)| \leq d\log(1+4L/\epsilon) + d^2\log(1+8\lambda^2d^{1/2}B^2/\gamma \epsilon^2).
     \end{align*}     
 \end{lemma}

\begin{proof}
The \algname\ with KL-divergence is presented in \Cref{alg:R2PVI-KL}. Similar to the proof of TV divergence, we decompose the estimation uncertainty between $\mathcal{T}_h^{\lambda}$ and $\hat{\mathcal{T}}_h^{\lambda}$ as:
         \begin{align}
             |\mathcal{T}_h^{\lambda} \hat{V}^{\lambda}_{h+1}(s,a)-\la \bphi(s,a), \hat{\bw}_h^\lambda \ra| &= \big|\bphi(s,a)^{\top}\big(\btheta_h - \lambda \log \bw_h^{\lambda} - \btheta_h + \lambda \log \max \{\hat{\bw}'_h, e^{-H/\lambda} \}\big)\big| \nonumber \\
             &=\big |\bphi(s,a)^{\top}\big(\lambda \log \max \{\hat{\bw}'_h, e^{-H/\lambda} \} - \lambda \log \bw_h^{\lambda}\big)\big| \nonumber \\
             &= \lambda \Big|\sum_{i=1}^d \phi_i(s,a) \log \frac{\max \{ \hat{w}'_{h,i}, e^{-H/\lambda}\}}{w_{h,i}^{\lambda}}\Big| \nonumber \\
             &\leq 
             \lambda \sum_{i=1}^d \phi_i(s,a) \Big|\log \frac{\max \{\hat{w}'_{h,i}, e^{-H/\lambda}\}}{w_{h,i}^{\lambda}} \Big| \nonumber \\
             & \leq \lambda \sum_{i=1}^d \phi_i(s,a) \big|\log \big(1+ e^{H/\lambda}|\max \{\hat{w}'_{h,i}, e^{-H/\lambda}\} - w_{h,i}^{\lambda}|\big)\big| \label{KL1}\\
             & \leq \lambda \sum_{i=1}^d \phi_i(s,a) \log \big(1+ e^{H/\lambda}|\hat{w}'_{h,i} - w_{h,i}^{\lambda}|\big) \label{KL2}\\
             & \leq \lambda \log \Big(\sum_{i=1}^d\phi_i(s,a)+ e^{H/\lambda}\sum_{i=1}^d\phi_i(s,a)|\hat{w}'_{h,i} - w_{h,i}^{\lambda}|\Big) \label{KL3} \\
             & = \lambda \log\Big(1 + e^{H/\lambda}\sum_{i=1}^d\phi_i(s,a) \mathbf{1}_i^\top|\hat{\bw}'_h - \bw_h^\lambda|\Big),\nonumber
         \end{align}
where \eqref{KL1} and \eqref{KL2} comes from the fact that: 
         \begin{align*}
           w_{h,i}^{\lambda} = \EE_{s' \sim \mu_{h,i}^0}\Big[e^{-\frac{\hat{V}_h^{\lambda}(s')}{\lambda}}\Big] \geq \EE_{s' \sim \mu_{h,i}^0}[e^{-H/\lambda}]=e^{-H/\lambda}, |\log A - \log B| = \log \Big(1 + \frac{|A-B|}{\min \{ A, B\} }\Big),  
         \end{align*}
and \eqref{KL3} comes from the Jensen's inequality applying to function $\log(x)$.   
Therefore, our next goal is to bound the term $\sum_{i=1}^d\phi_i(s,a)\mathbf{1}_i^\top|\hat{\bw}'_h - \bw_h^{\lambda}|$. Specifically, we have
         \begin{align*}           &\sum_{i=1}^d\phi_i(s,a)\mathbf{1}_i^\top|\hat{\bw}'_h - \bw_h^{\lambda}| \\&= \sum_{i=1}^d\phi_i(s,a)\mathbf{1}_i^\top\Big|\bw_h^\lambda - \bLambda_h^{-1}\sum_{\tau=1}^K(\bphi^{\tau}_h) e^{-\frac{\hat{V}^{\lambda}_{h+1}(s_{h+1})}{\lambda}}\Big| \\
            &=\sum_{i=1}^d\phi_i(s,a)\mathbf{1}_i^\top\Big|\bw_h^\lambda - \bLambda_h^{-1}\sum_{\tau=1}^K\bphi^{\tau}_h(\bphi_h^\tau)^{\top}\bw_h^\lambda +\bLambda_h^{-1}\sum_{\tau=1}^K\bphi^{\tau}_h(\bphi_h^\tau)^{\top}\bw_h^\lambda -\bLambda_h^{-1}\sum_{\tau=1}^K(\bphi^{\tau}_h) e^{-\frac{\hat{V}_{h+1}^{\lambda}(s_{h+1})}{\lambda}}\Big| \\
            &=\sum_{i=1}^d\phi_i(s,a)\mathbf{1}_i^\top\Big(\underbrace{\Big|\bw_h^\lambda - \bLambda_h^{-1}\sum_{\tau=1}^K\bphi^{\tau}_h(\bphi_h^\tau)^{\top}\bw_h^\lambda\Big|}_{\text{(i)}} +\underbrace{\Big|\bLambda_h^{-1}\sum_{\tau=1}^K \bphi^{\tau}_h (\bphi_h^\tau)^{\top} \bw_h^\lambda - \bLambda_h^{-1}\sum_{\tau=1}^K(\bphi^{\tau}_h) e^{-\frac{\hat{V}_{h+1}^{\lambda}(s_{h+1})}{\lambda}}\Big|}_{\text{(ii)}} \Big).
         \end{align*}         
        Next, we upper bound term (i) and (ii), respectively.
For the first term, we have: 
        \begin{align}
\sum_{i=1}^d\phi_i(s,a)\mathbf{1}_i^\top \cdot \text{(i)} &= \sum_{i=1}^d\phi_i(s,a)\mathbf{1}_i^\top(|\bw_h^{\lambda} - \bLambda_h^{-1}
            (\bLambda_h - \gamma \mathbf{I})\bw_h^{\lambda}|) \nonumber\\
            &=\gamma \sum_{i=1}^d\phi_i(s,a)\mathbf{1}_i^\top\bLambda_h^{-1}|\bw_h^{\lambda}| \nonumber\\
            &\leq \gamma \sum_{i=1}^d \|\phi_i(s,a)\mathbf{1}_i\|_{\bLambda_h^{-1}}\|\bw_h^{\lambda}\|_{\bLambda_h^{-1}} \label{KL4}\\
            &\leq \gamma \sqrt{d}\sum_{i=1}^d \|\phi_i(s,a)\mathbf{1}_i\|_{\bLambda_h^{-1}},\label{KL5}
        \end{align}
        where \eqref{KL4} follows from the Cauchy-Schwartz inequality, \eqref{KL5} follows from the fact that:
        \begin{align*}
        \|\bw_h^\lambda\|_{\bLambda_h^{-1}} \leq \|\bLambda_h^{-1}\|^{1/2}\|\bw_h^\lambda\|_2 \leq \sqrt{d/\gamma},   
        \end{align*}
        where the last inequality follows from  \Cref{lemma:bound of weights in KL distance} and the fact that $\|\bLambda_h^{-1}\| \leq \gamma^{-1}$. 
Now it remains to bound the term (ii), by the definition of $\eta_h^{\tau}(f) = \EE_{s' \sim P_h^0(\cdot |s^{\tau}_h, a_h^{\tau})}[f(s')] - f(s_{h+1}^{\tau})$, the term (ii) can be rewritten as:
     \begin{align*}
         \text{(ii)} &= \sum_{i=1}^d\phi_i(s,a) \mathbf{1}_i^\top \Big|\bLambda_h^{-1} \sum_{\tau=1}^K \bphi_h^\tau\Big[(\bphi_h^\tau)^{\top}\bw^\lambda_h - e^{-\frac{\hat{V}^{\lambda}_{h+1}(s_{h+1})}{\lambda}}\Big]\Big| \\
         &=\sum_{i=1}^d\phi_i(s,a) \mathbf{1}_i^\top\Big| \bLambda_h^{-1} \sum_{\tau=1}^K \bphi_h^\tau\eta_h^{\tau}\Big(e^{-\frac{\hat{V}^{\lambda}_{h+1}(s)}{\lambda}}\Big)\Big| \\
         & \leq \sum_{i=1}^d \|\phi_i(s,a) \mathbf{1}_i\|_{\bLambda_h^{-1}} \underbrace{\Big\|\sum_{\tau=1}^K \bphi^{\tau}_h \eta_h^{\tau}\Big(e^{-\frac{\hat{V}^{\lambda}_{h+1}(s)}{\lambda}}\Big)\Big\|_{\bLambda_h^{-1}}}_{\text{(iii)}}.
     \end{align*}
      For the rest of the proof, it's left to bound the term (iii). As the $\hat{V}^{\lambda}_{h+1}$ depends on the offline dataset, which makes it difficult to upper bound directly from concentration equality due to the dependence issue, we seek for providing a uniform concentration bound applied to the term (iii), i.e. we aim to upper bound the following term:
    \begin{align*}
        \sup_{V \in \mathcal{V}_{h+1}(R, B, \gamma)}\Big\|\sum_{\tau = 1}^{K} \bphi_{h}^{\tau}\eta_h^{\tau}(e^{-\frac{V}{\lambda}})\Big\|_{\bLambda_h^{-1}}.
    \end{align*}
    Here for all $h \in [H]$, the function class is defined as:
    \begin{align*}
        &\mathcal{V}_h(R,B, \gamma) = \{ V_h(x; \btheta, \beta, \bLambda_h) : \|\btheta\|_2 \leq R, \beta \in [0,B], \gamma_{\min}(\bLambda_h) \geq \gamma \},
    \end{align*}
    where $V_h(x; \btheta, \beta, \bLambda_h) = \max_{a \in \mathcal{A}} \{ \bphi(s,a)^{\top} \btheta - \lambda \log(1+\beta \sum_{i=1}^d\|\phi_i(\cdot, \cdot) \mathbf{1}_i\|_{\bLambda_h^{-1}}) \}_{[0, H-h+1]}$.
    In order to ensure $\hat{V}^{\lambda}_{h+1} \in \mathcal{V}_{h+1}(R_0, B_0, \lambda)$, we need to bound $\hat{\btheta}_h = \btheta_h - \lambda \log \max\{\hat{\bw}'_h, e^{-H/\lambda}\}$. Following the fact that:
    \begin{align*}
        \|\hat{\btheta}_h\|_2 \leq \|\btheta_h\|_2 + \lambda \|\log \max\{\hat{\bw}'_h, e^{-H/\lambda}\}\|_2.
    \end{align*}
    By \Cref{lemma:bound of weights in KL distance}, $e^{-H/\lambda} \leq \max\{\hat{w}'_{h,i}, e^{-H/\lambda}\} \leq \max\{\|\hat{\bw}'_h\|, e^{-H/\lambda}\} \leq \max \{ \sqrt{Kd/\lambda},e^{-H/\lambda}\} $, therefore the term can be bounded as:
    \begin{align}
        \|\hat{\btheta}_h\|_2 &\leq \sqrt{d} + \lambda \sqrt{d} \max{\Big(\log\sqrt{\frac{Kd}{\lambda}}, H/\lambda\Big)} \nonumber\\
        &\leq H\sqrt{d} + d \sqrt{K\lambda} \nonumber\\
        &\leq 2Hd\sqrt{K\lambda}.\label{bound of weights in KL}
    \end{align}
    \noindent
    Hence, we can choose $R_0 = 2Hd\sqrt{K\lambda}$ and $B_0 = \beta = 16d\lambda  e^{H/\lambda} \sqrt{(H/\lambda + \xi_{\text{KL}})}$, then we have for all $h \in [H], \hat{V}^{\lambda}_{h+1} \in \mathcal{V}_{h+1}(R_0, B_0, \lambda)$. 
    Next we aim to find a union cover of the $\mathcal{V}_{h+1}(R_0, B_0, \gamma)$, hence the term (iii) can be upper bounded. For all $\epsilon \in (0, \lambda), h \in [H]$, let $\mathcal{N}_h(\epsilon; R, B, \lambda): = \mathcal{N}_h(\epsilon)$ denote the minimal $\epsilon$-cover of $\mathcal{V}_h(R,B, \lambda)$ with respect to the supreme norm. In other words, for any function $\hat{V}^{\lambda} \in \mathcal{V}_h(R, B, \lambda)$, there exists a function $V' \in \mathcal{N}_{h+1}(\epsilon)$ such that 
    \begin{align*}
        \sup_{x \in \mathcal{S}}|\hat{V}^{\lambda}_{h+1}(x) - V'_{h+1}(x)| \leq \epsilon.
    \end{align*}
        Hence, given $\hat{V}^{\lambda}_{h+1}, V'_{h+1}$ satisfying the inequality above, recall the definition of $\eta_h^\tau = \eta_h^{\tau}(f)= \EE_{s' \sim P_h^0(\cdot |s^{\tau}_h, a_h^{\tau})}[f(s')] - f(s_{h+1}^{\tau}) $, we have:
    \begin{align}
        &\Big|\eta_h^{\tau}\Big(e^{-\frac{\hat{V}^{\lambda}_{h+1}(s)}{\lambda}}\Big) - \eta_h^{\tau}\Big(e^{-\frac{V'_{h+1}(s)}{\lambda}}\Big)\Big| \nonumber\\
        &\leq 
        \Big|\EE_{s \sim P_h^0(\cdot |s^{\tau}_h, a_h^{\tau})}\Big[e^{-\frac{\hat{V}^{\lambda}_{h+1}(s)}{\lambda}} - e^{-\frac{V'_{h+1}(s)}{\lambda}}\Big] - e^{-\frac{\hat{V}^{\lambda}_{h+1}(s_{h+1})}{\lambda}} + e^{-\frac{V'_{h+1}(s_{h+1})}{\lambda}}\Big| \nonumber\\
        &\leq \Big|\EE_{s \sim P_h^0(\cdot |s^{\tau}_h, a_h^{\tau})}\Big[e^{-\frac{\hat{V}^{\lambda}_{h+1}(s)}{\lambda}} - e^{-\frac{V'_{h+1}(s)}{\lambda}}\Big]\Big| + \Big|e^{-\frac{\hat{V}^{\lambda}_{h+1}(s_{h+1})}{\lambda}} - e^{-\frac{V'_{h+1}(s_{h+1})}{\lambda}}\Big| \nonumber \\
        &\leq 2\epsilon /\lambda + 2\epsilon / \lambda=4\epsilon /\lambda, \label{KL6}
    \end{align}
    \noindent
    where \eqref{KL6} follows from the fact that for any $s \in \mathcal{S}$, 
    \begin{align*}
       \Big|e^{-\frac{\hat{V}^{\lambda}_{h+1}(s)}{\lambda}} - e^{-\frac{V'_{h+1}(s)}{\lambda}}\Big| \leq e^{\frac{|\hat{V}^{\lambda}_{h+1}(s)-V'_{h+1}(s)|}{\lambda}}  - 1 \leq e^{\frac{\epsilon}{\lambda}}-1 \leq 2\epsilon /\lambda,
    \end{align*}
    where the last inequality is held by the fact that $\epsilon \in (0, \lambda)$.
    By the Cauchy-Schwartz inequality, for any two vectors $a,b \in \RR^d$ and positive definite matrix $\bLambda \in \RR^{d \times d}$, it holds that $\|a+b\|^2_{\bLambda} \leq 2\|a\|_{\bLambda}^2 + 2\|b\|_{\bLambda}^2$, hence for all $h \in [H]$, we have:
    \begin{align}
        |\text{(iii)}|^2 &\leq 2\Big\|\sum_{\tau = 1}^{K} \bphi_{h}^{\tau}\eta_h^{\tau}\Big(e^{-\frac{V'_{h+1}(s)}{\lambda}}\Big)\Big\|^2_{\bLambda_h^{-1}} + 2\Big\|\sum_{\tau = 1}^{K} \bphi_{h}^{\tau}\Big[\eta_h^{\tau}\Big(e^{-\frac{V'_{h+1}(s)}{\lambda}}\Big) - \eta_h^{\tau}\Big(e^{-\frac{\hat{V}^{\lambda}_{h+1}(s)}{\lambda}}\Big)\Big]\Big\|^2_{\bLambda_h^{-1}} \nonumber\\
        &\leq 2\Big\|\sum_{\tau = 1}^{K} \bphi_{h}^{\tau}\eta_h^{\tau}\Big(e^{-\frac{V'_{h+1}(s)}{\lambda}}\Big)\Big\|_{\bLambda_h^{-1}}^2 + 32\epsilon^2/\lambda^2\sum_{\tau, \tau' =1}^K|\bphi_h^\tau\bLambda_h^{-1}\bphi_h^{\tau'}| \nonumber\\
        &\leq 2\sup_{V \in \mathcal{N}_{h+1}(\epsilon)}\Big\|\sum_{\tau = 1}^{K} \bphi_{h}^{\tau}\eta_h^{\tau}\Big(e^{-\frac{V(s)}{\lambda}}\Big)\Big\|_{\bLambda_h^{-1}}^2 + \frac{32\epsilon^2K^2}{\lambda^2 \gamma}.
        \label{KL10}
    \end{align}
     We set $f(s) = e^{-\frac{V(s)}{\lambda}}$, by applying \Cref{lemma:concentration bound}, for any fixed $h \in [H], \delta \in (0, 1)$, we have:
    \begin{align}
        P\Big(\sup_{V \in \mathcal{N}_{h+1}(\epsilon)}\Big\|\sum_{\tau = 1}^K \bphi_h^\tau\eta_{h}^{\tau}\Big(e^{-\frac{V(s)}{\lambda}}\Big)\Big\|_{\bLambda_h^{-1}}^2 \geq4\Big(2\log\frac{H |\mathcal{N}_{h+1}(\epsilon)|}{\delta}+ d\log\Big(1+ \frac{K}{\gamma}\Big)\Big)\Big) \leq \delta/H. \label{KL7}
    \end{align}
    Hence, combining \eqref{KL7} with \eqref{KL10} and let $\gamma = 1$, then for all $h \in [H]$, it holds that
    \begin{align}
        \Big\|\sum_{\tau = 1}^K \bphi_h^\tau\eta_{h}^{\tau}\Big(e^{-\frac{\hat{V}^{\lambda}_{h+1}(s)}{\lambda}}\Big)\Big\|_{\bLambda_h^{-1}}^2 \leq 8\Big(2\log\frac{H |\mathcal{N}_{h+1}(\epsilon)|}{\delta}+ d\log(1+ K) + \frac{4\epsilon^2K^2}{\lambda^2}\Big),
        \label{E.5}
    \end{align}
    with probability at least $1-\delta$.
    By \Cref{lemma: covering number of function class in KL}, recall $L = R_0 = 2Hd\sqrt{K\lambda}$ in this setting, we have 
    \begin{align}
        \log(|\mathcal{N}_{h+1}(\epsilon)|) 
        \leq d\log(1+ 4R_0/\epsilon) + d^2\log(1+ 8\lambda^2 d^{1/2}B^2/ \epsilon^2).  \label{E.6}
    \end{align}
    We then set $\epsilon = d\lambda/{K} \in (0, \lambda)$ and define $\beta' = \beta / \lambda e^{H/\lambda} = 16d \sqrt{(H/\lambda + \xi_{\text{KL}})}$ for brevity, then (\ref{E.6}) 
    can be bounded as:
    \begin{align*}
        \log(|\mathcal{N}_{h+1}(\epsilon)|) 
        &\leq d\log(1+ 4R_0K/d\lambda) + d^2\log(1+ 8\lambda^2 K^2d^{-3/2}e^{\frac{2H}{\lambda}}\beta'^2)   \\
        & = d \log (1+ 4R_0 K/d\lambda) + d^2 \log(e^{-2H/\lambda}+ 8\lambda^2 K^2d^{-3/2}\beta'^2) + 2d^2H/\lambda\\
        &\leq 2d^2\log(8\lambda^2K^2d^{-3/2}\beta'^2) + 2d^2H/\lambda.
    \end{align*}
    \noindent
    Therefore, by combining the result with the inequality (\ref{E.5}), we can get  
    \begin{align}        
        \Big\|\sum_{\tau = 1}^K \bphi_h^\tau\eta_{h}^{\tau}\Big(e^{-\frac{\hat{V}^{\lambda}_{h+1}(s)}{\lambda}}\Big)\Big\|_{\bLambda_h^{-1}}^2 
       &\leq 8\Big(2\log\frac{H}{\delta} + 4d^2H/\lambda + 4d^2\log(8\lambda^2K^2d^{-3/2}\beta'^2) + 4d^2 + d\log(1+K)\Big) \nonumber\\ 
    &\leq 8(4d^2H/\lambda + 4d^2\log(8\lambda^2K^3Hd^{-3/2}\beta'^2/\delta)) \label{KL11}
    \\
    & = 8(4d^2H/\lambda + 4d^2\log(8\lambda^2K^3Hd^{-3/2}/\delta) + 4d^2\log(\beta'^2))\nonumber\\
    & \leq \beta'^2/4, \label{KL12}
    \end{align}
    where \eqref{KL11} follows by the fact that $2\log\frac{H}{\delta} + 4d ^2 + d\log(1+K) \leq 4d^2 \log (\frac{HK}{\delta})$, and \eqref{KL12} is held due to the fact that 
    \begin{align}
        \beta'^2/4 &= 64d^2\Big(H/\lambda + \log\frac{1024d\lambda^2 K^3H}{\delta}\Big)\nonumber \\
                   &= 8\Big(8d^2 H /\lambda + 4d^2\log(8\lambda^2K^3Hd^{-3/2}/\delta)+4d^2\log\frac{1024d^{7/2} \lambda^2 K^3H}{\delta} + 4d^2 \log(128)\Big) \nonumber\\
                   &\geq 8\Big(4d^2 H /\lambda + 4d^2\log(8\lambda^2K^3Hd^{-3/2}/\delta) + 4d^2 \log(\beta'^2)\Big),
                   \label{KL13}
    \end{align}
    where \eqref{KL13} holds by 
    \begin{align*}
        \log(\beta'^2) &= \log\Big(256d^2 \Big(H/\lambda + \log\frac{1024d\lambda^2 K^3H}{\delta}\Big) \Big) \\
        & \leq \log(256d^2) + \Big(H/\lambda + \log\frac{1024d\lambda^2 K^3H}{\delta}\Big)      \\
        & \leq \log(128) + \log\frac{1024d^{7/2} \lambda^2 K^3H}{\delta} + H/\lambda.
    \end{align*}
    By the bound on (i), (ii), (iii), for all $h\in H$ and $(s,a) \in \mathcal{S} \times \mathcal{A}$, with probability at least $1-\delta$, it holds that
    \begin{align}        
   |\mathcal{T}_h^{\lambda} \hat{V}^{\lambda}_{h+1}(s,a)-\la \bphi(s,a), \hat{\bw}_h^\lambda \ra| &\leq\lambda \log\Big(1 + e^{H/\lambda}(\sqrt{d} +\beta'/2)\sum_{i=1}^d\|\phi_i(s,a)\mathbf{1}_i\|_{\bLambda_h^{-1}}\Big) \nonumber\\ 
   &\leq\lambda \log\Big(1 + e^{H/\lambda}\beta'\sum_{i=1}^d\|\phi_i(s,a)\mathbf{1}_i\|_{\bLambda_h^{-1}}\Big) \label{KL14}\\
   & \leq  \beta\sum_{i=1}^d\|\phi_i(s,a)\mathbf{1}_i\|_{\bLambda_h^{-1}},\label{KL15}
    \end{align}
     where \eqref{KL14} follows by the fact that $\beta' \geq 2\sqrt{d}$, \eqref{KL15} follows by the fact that $\log(1+x) \leq x$ holds for any positive $x$. Thus, by \Cref{lemma:regularized robust suboptimality lemma}, we can upper bound the suboptimality gap as:
     \begin{align*}
         \text{SubOpt}(\hat{\pi}_1, s) &\leq 2\sup_{P \in \mathcal{U}^\lambda(P^0)}\sum_{h=1}^H\EE^{\pi^\star, P}[\Gamma_h(s_h, a_h)|s_1 = s] \\
         & = 2\beta\sup_{P \in \mathcal{U}^\lambda(P^0)}\sum_{h=1}^H\EE^{\pi^\star, P}\Big[\sum_{i=1}^d\|\phi_i(s,a)\mathbf{1}_i\|_{\bLambda_h^{-1}}|s_1 = s\Big].
     \end{align*}
     \noindent
     Therefore, we conclude the proof.
\end{proof}

\subsubsection{Proof of \Cref{thm: instance-dependent upper bounds} - Case with $\chi^2$ Divergence} 
\begin{algorithm}[ht]
    \caption{ Robust Regularized Pessimistic Value Iteration under $\chi^2$ distance (\algname-$\chi^2$)}\label{alg:R2PVI-chi2}
    \begin{algorithmic}[1]
    \REQUIRE{
        Dataset $\cD$, regularizer $\lambda>0, \gamma>0$ and parameter $\beta$
        }
        \STATE init $\hat{V}^\lambda_{H+1}(\cdot)=0$
        \FOR {episode $h=H, \cdots, 1$}{
            \STATE $\bLambda_h \leftarrow \sum_{\tau = 1}^K \bphi(s_h^{\tau}, a_h^{\tau})(\bphi(s_h^{\tau}, a_h^{\tau}))^{\top} + \gamma \mathbf{I}$
             \STATE $\hat{\EE}^{\mu_{h,i}^0}[\hat{V}_{h+1}^{\lambda}(s)]_{\alpha} \leftarrow  [\Lambda_h^{-1}(\sum_{\tau=1}^{K} \phi(s_h^{\tau}, a_h^{\tau})^{\top} [\hat{V}_{h+1}^\lambda(s_{h+1}^{\tau})]_{\alpha})]_{[0,H]}$ \hfill $\triangleleft$ {\color{blue} Estimated by \eqref{eq: estimation of EV in chi2}}
             \STATE $\hat{\EE}^{\mu_{h,i}^0}[\hat{V}_{h+1}^{\lambda}(s)]^2_{\alpha} \leftarrow [\Lambda_h^{-1}(\sum_{\tau=1}^{K} \phi(s_h^{\tau}, a_h^{\tau})^{\top} [\hat{V}^\lambda_{h+1}(s_{h+1}^{\tau})]_{\alpha}^2 )]_{[0,H^2]}$  \hfill $\triangleleft$ {\color{blue} Estimated by \eqref{eq: estimation of EV2 in chi2}}
            \STATE Estimate $\hat{w}_{h,i}^{\lambda}$ according to \eqref{eq:parameter estimation x2} 
            \STATE $\Gamma_h(\cdot, \cdot) \leftarrow 
            \beta \sum_{i=1}^d\|\phi_i(\cdot, \cdot)\|_{\Lambda_h^{-1}}$
            \STATE $\hat{Q}_h^{\lambda}(\cdot, \cdot) \leftarrow \min \{\bphi(\cdot, \cdot)^{\top}(\btheta_h +  \hat{\bw}_h^{\lambda}) - \Gamma_h(\cdot, \cdot), H-h+1\}^+$
            \STATE $\hat{\pi}_h(\cdot | \cdot) \leftarrow \argmax_{\pi_h}\langle \hat{Q}_h^{\lambda}(\cdot, \cdot), {\pi}_h(\cdot| \cdot) \rangle_{\mathcal{A}}$ and $\hat{V}_h^{\lambda}(\cdot) \leftarrow  \langle \hat{Q}_h^{\lambda}(\cdot, \cdot), \hat{\pi}_h(\cdot|\cdot) \rangle_{\mathcal{A}}$
        }\ENDFOR
    \end{algorithmic}
\end{algorithm}
For completeness, we present the \algname\ algorithm specific to the $\chi^2$ distance in \Cref{alg:R2PVI-chi2},  which gives closed form solution of \eqref{eq: estimation of EV in chi2} and \eqref{eq: estimation of EV2 in chi2}.
Before the proof, we first present the bound on weights under $\chi^2$-divergence:
\begin{lemma}[Bound of weights - $\chi^2$]
 \label{lemma:bound of weights in chi2 distance}
     For any $h \in [H]$,
     \begin{align*}
            \|\hat{\bw}_h^{\lambda}\|_2 
             \leq \sqrt{d}\Big(H + \frac{H^2}{2\lambda}\Big).
     \end{align*}
 \end{lemma}
   \begin{proof}[Proof of \Cref{thm: instance-dependent upper bounds} - $\chi^2$]
   The \algname\ with $\chi^2$-divergence is presented in \Cref{alg:R2PVI-chi2}. By the definition of $\mathcal{T}_h^{\lambda}, \hat{\mathcal{T}}_h^{\lambda}$, we have
         \begin{align}
             &\mathcal{T}_h^{\lambda} \hat{V}^{\lambda}_{h+1}(s,a)-\la \bphi(s,a), \hat{\bw}_h^\lambda \ra \nonumber\\
             &= \bphi(s,a)^{\top}(\btheta_h + \bw_h^{\lambda} - \btheta_h - \hat{\bw}_h^{\lambda}) \nonumber\\
             &= \sum_{i=1}^d \phi_i(s,a)(w_{h, i}^{\lambda} -  \hat{w}'_{h,i}) \nonumber\\
             & = \sum_{i=1}^d \phi_i(s,a)\Big[ \sup_{\alpha \in [0, H]}\Big\{\EE_{s \sim \mu_{h,i}^0}[\hat{V}^{\lambda}_{h+1}(s)]_{\alpha} + \frac{1}{4\lambda}(\EE_{s \sim \mu_{h,i}^0}[\hat{V}^{\lambda}_{h+1}(s)]_{\alpha})^2 - \frac{1}{4\lambda}\EE_{s \sim \mu_{h,i}^0}[\hat{V}^{\lambda}_{h+1}(s)]_{\alpha}^2 \Big\}
             \nonumber\\
             & \quad - \sup_{\alpha \in [0, H]}\Big\{\hat{\EE}^{\mu_{h,i}^0}[\hat{V}^{\lambda}_{h+1}(s)]_{\alpha} + \frac{1}{4\lambda}(\hat{\EE}^{\mu_{h,i}^0}[\hat{V}^{\lambda}_{h+1}(s)]_{\alpha})^2 - \frac{1}{4\lambda}\hat{\EE}^{\mu_{h,i}^0}[\hat{V}^{\lambda}_{h+1}(s)]_{\alpha}^2 \Big\} \Big].  \label{chi2-first decompose}
         \end{align}
        To continue, for any $i \in [d]$, we denote 
         \begin{align*}
             \alpha_i = \argmax_{\alpha \in [0, H]}\Big\{\EE_{s \sim \mu_{h,i}^0}[\hat{V}^{\lambda}_{h+1}(s)]_{\alpha} + \frac{1}{4\lambda}(\EE_{s \sim \mu_{h,i}^0}[\hat{V}^{\lambda}_{h+1}(s)]_{\alpha})^2 - \frac{1}{4\lambda}\EE_{s \sim \mu_{h,i}^0}[\hat{V}^{\lambda}_{h+1}(s)]_{\alpha}^2 \Big\}.
         \end{align*}
         Hence, \eqref{chi2-first decompose} can be further upper bounded as 
         \begin{align}
             &\mathcal{T}_h^{\lambda} \hat{V}^{\lambda}_{h+1}(s,a)-\la \bphi(s,a), \hat{\bw}_h^\lambda \ra \nonumber
             \\
             &\leq 
              \underbrace{\sum_{i=1}^d \phi_i(s,a)\big(\EE_{s \sim \mu_{h,i}^0}[\hat{V}^{\lambda}_{h+1}(s)]_{\alpha_i} - \hat{\EE}^{\mu_{h,i}^0}[\hat{V}^{\lambda}_{h+1}(s)]_{\alpha_i}\big)\Big(\frac{1}{4\lambda}\big(\EE_{s \sim \mu_{h,i}^0}[\hat{V}^{\lambda}_{h+1}(s)]_{\alpha_i} + \hat{\EE}^{\mu_{h,i}^0}[\hat{V}^{\lambda}_{h+1}(s)]_{\alpha_i}\big)+1\Big)}_{\text{(i)}} \nonumber  \\
             & \quad- \underbrace{\sum_{i=1}^d \phi_i(s,a)\frac{1}{4\lambda}\big(\EE_{s \sim \mu_{h,i}^0}[\hat{V}^{\lambda}_{h+1}(s)]^2_{\alpha_i} - \hat{\EE}^{\mu_{h,i}^0}[\hat{V}^{\lambda}_{h+1}(s)]^2_{\alpha_i}\big)}_{\text{(ii)}}. \label{upper bound for suboptimality}
         \end{align}
         Next, we bound (i) and (ii), respectively.
         \paragraph{Bounding term (i).} We define 
         \begin{align*}
             \tilde{\EE}^{\mu_{h,i}^0}[\hat{V}_{h+1}^\lambda(s)]_{\alpha} &= \Big[\argmin_{\bw \in R^d} \sum_{\tau=1}^K ([\hat{V}_{h+1}^{\lambda}(s_{h+1}^\tau)]_{\alpha} - \bphi(s_h^\tau, a_h^{\tau})^{\top}\bw)^2 + \gamma \| \bw \|_2^2 \Big]^i.
         \end{align*}
        Considering the gap between the $\hat{\EE}^{\mu_{h,i}^0}[\hat{V}^{\lambda}_{h+1}(s)]_{\alpha_i}$ and $\tilde{\EE}^{\mu_{h,i}^0}[\hat{V}^{\lambda}_{h+1}(s)]_{\alpha_i}$ due to the definition that  $\hat{\EE}^{\mu_{h,i}^0}[\hat{V}^{\lambda}_{h+1}(s)]_{\alpha_i} = [\tilde{\EE}^{\mu_{h,i}^0}[\hat{V}^{\lambda}_{h+1}(s)]_{\alpha_i}]_{[0,H]}$, we eliminate the clip operator at first. We rewrite (i) as follows:
        \begin{align*}
            \text{(i)} &= \sum_{i=1}^d \phi_i(s,a)\big(\EE_{s \sim \mu_{h,i}^0}[\hat{V}^{\lambda}_{h+1}(s)]_{\alpha_i} - \hat{\EE}^{\mu_{h,i}^0}[\hat{V}^{\lambda}_{h+1}(s)]_{\alpha_i}\big)\Big(\frac{1}{4\lambda}\big(\EE_{s \sim \mu_{h,i}^0}[\hat{V}^{\lambda}_{h+1}(s)]_{\alpha_i} + \hat{\EE}^{\mu_{h,i}^0}[\hat{V}^{\lambda}_{h+1}(s)]_{\alpha_i}\big)+1\Big) \\
            & =  \sum_{i=1}^d \phi_i(s,a)(\EE_{s \sim \mu_{h,i}^0}[\hat{V}^{\lambda}_{h+1}(s)]_{\alpha_i} - \tilde{\EE}^{\mu_{h,i}^0}[\hat{V}^{\lambda}_{h+1}(s)]_{\alpha_i}) \\
            &\qquad \times\underbrace{\Big(\frac{1}{4\lambda}\big(\EE_{s \sim \mu_{h,i}^0}[\hat{V}^{\lambda}_{h+1}(s)]_{\alpha_i} + \hat{\EE}^{\mu_{h,i}^0}[\hat{V}^{\lambda}_{h+1}(s)]_{\alpha_i}\big)+1\Big) \frac{\EE^{\mu_{h, i}^0}[\hat{V}_{h+1}^\lambda(s)]_{\alpha_i} -\hat{\EE}^{\mu_{h, i}^0}[\hat{V}_{h+1}^\lambda(s)]_{\alpha_i}}{\EE^{\mu_{h, i}^0}[\hat{V}_{h+1}^\lambda(s)]_{\alpha_i}- \tilde{\EE}^{\mu_{h, i}^0}[\hat{V}_{h+1}^\lambda(s)]_{\alpha_i}}}_{:=C_i}.
         \end{align*}
         We claim that $|C_i| \leq 1 + H/2\lambda, \forall i \in [H]$. We prove the claim by discussing the value of $\tilde{\EE}^{\mu_{h,i}^0}[\hat{V}^{\lambda}_{h+1}(s)]_{\alpha_i}$ in the following three cases:
         \paragraph{Case I.} $\tilde{\EE}^{\mu_{h,i}^0}[\hat{V}^{\lambda}_{h+1}(s)]_{\alpha_i} \leq 0$. By the fact that $\EE_{s \sim \mu_{h,i}^0}[\hat{V}^{\lambda}_{h+1}(s)]_{\alpha_{i}} \leq H$, we have:
         \begin{align*}
             |C_i| &= \Bigg|\Big( \frac{1}{4\lambda}\EE_{s \sim \mu_{h,i}^0}[\hat{V}^{\lambda}_{h+1}(s)]_{\alpha_i} + 1 \Big) \frac{\EE_{s \sim \mu_{h,i}^0}[\hat{V}^{\lambda}_{h+1}(s)]_{\alpha_i}}{\EE_{s \sim \mu_{h,i}^0}[\hat{V}^{\lambda}_{h+1}(s)]_{\alpha_i} - \tilde{\EE}^{\mu_{h,i}^0}[\hat{V}^{\lambda}_{h+1}(s)]_{\alpha_i}} \Bigg|
               \leq 1 + H/4\lambda,
         \end{align*}
         where the equality holds by $\frac{1}{4\lambda}\EE_{s \sim \mu_{h,i}^0}[\hat{V}^{\lambda}_{h+1}(s)]_{\alpha_i} + 1 \leq 1 + H/4\lambda$. Hence the claim holds by \textbf{Case I}. 
         \paragraph{Case II.} $0\leq \tilde{\EE}^{\mu_{h,i}^0} 
             [\hat{V}^{\lambda}_{h+1}(s)]_{\alpha_i} \leq H$. The claim holds trivially, as we have:
        \begin{align*}
            |C_i| = \frac{1}{4\lambda}\big(\EE_{s \sim \mu_{h,i}^0}[\hat{V}^{\lambda}_{h+1}(s)]_{\alpha_i} + \tilde{\EE}^{\mu_{h,i}^0}[\hat{V}^{\lambda}_{h+1}(s)]_{\alpha_i}\big)+1  \leq 1 + H/2\lambda.
        \end{align*}
        Hence, we conclude the claim. 
        \paragraph{Case III.} $\tilde{\EE}^{\mu_{h,i}^0} 
             [\hat{V}^{\lambda}_{h+1}(s)]_{\alpha_i} > H$. Notice that 
        \begin{align}
            |C_i| &= \Bigg|\Big(\frac{1}{4\lambda}\big(\EE_{s \sim \mu_{h,i}^0}[\hat{V}^{\lambda}_{h+1}(s)]_{\alpha_i} + H\big) + 1 \Big) \frac{\EE_{s \sim \mu_{h,i}^0}[\hat{V}^{\lambda}_{h+1}(s)]_{\alpha_i} - H}{\EE_{s \sim \mu_{h,i}^0}[\hat{V}^{\lambda}_{h+1}(s)]_{\alpha_i} - \tilde{\EE}^{\mu_{h,i}^0}[\hat{V}^{\lambda}_{h+1}(s)]_{\alpha_i}} \Bigg| \nonumber\\
            & = \Big(\frac{1}{4\lambda}\big(\EE_{s \sim \mu_{h,i}^0}[\hat{V}^{\lambda}_{h+1}(s)]_{\alpha_i} + H\big) + 1 \Big)\frac{H - \EE_{s \sim \mu_{h,i}^0}[\hat{V}^{\lambda}_{h+1}(s)]_{\alpha_i}}{\tilde{\EE}^{\mu_{h,i}^0}[\hat{V}^{\lambda}_{h+1}(s)]_{\alpha_i} - \EE_{s \sim \mu_{h,i}^0}[\hat{V}^{\lambda}_{h+1}(s)]_{\alpha_i}} \nonumber\\
            &\leq H/2\lambda + 1, \label{chi2.10}
        \end{align}
        where \eqref{chi2.10} holds by the fact that $\tilde{\EE}^{\mu_{h,i}^0}[\hat{V}^{\lambda}_{h+1}(s)]_{\alpha_i} >H$.
        
         With the upper bound for $C_i$, we can upper bound (i) as 
        \begin{align}
             |\text{(i)}| & = \Big|\sum_{i=1}^d \phi_i(s,a)(\EE_{s \sim \mu_{h,i}^0}[\hat{V}^{\lambda}_{h+1}(s)]_{\alpha_i} - \tilde{\EE}^{\mu_{h,i}^0}[\hat{V}^{\lambda}_{h+1}(s)]_{\alpha_i})C_i\Big| \nonumber
             \\ &= \Big|\gamma \sum_{i=1}^d\phi_i(s,a)\mathbf{1}_i\bLambda_h^{-1}\EE^{\bmu_h^0}[\hat{V}^{\lambda}_h(s)]_{\alpha_i}C_i +  \sum_{i=1}^d\phi_i(s,a)\mathbf{1}_i\bLambda_h^{-1}\sum_{\tau = 1}^K\bphi(s_h^\tau, a_h^{\tau})\eta_{h}^{\tau} ([\hat{V}^{\lambda}_{h+1}]_{\alpha_i})C_i \Big| \nonumber\\ 
             &\leq (1 + H/2\lambda)\sum_{i=1}^d\|\phi_i(s,a)\mathbf{1}_i\|_{\bLambda_h^{-1}}\Big(\gamma H  +\Big\|\sum_{\tau = 1}^K\bphi(s_h^\tau, a_h^{\tau})\eta_{h}^{\tau}([\hat{V}^{\lambda}_{h+1}]_{\alpha_i})\Big\|_{\bLambda_{h}^{-1}}\Big). \label{upper bound for (i)}
         \end{align}
         \paragraph{Bounding term (ii).} Similar to bounding (i), we can deduce that: 
         \begin{align}
            |\text{(ii)}| & = \frac{1}{4\lambda}\Big |\sum_{i=1}^d \phi_i(s,a)(\EE_{s \sim \mu_{h,i}^0}[\hat{V}^{\lambda}_{h+1}(s)]^2_{\alpha_i} - \hat{\EE}^{\mu_{h,i}^0}[\hat{V}^{\lambda}_{h+1}(s)]^2_{\alpha_i}) \Big | \nonumber \\
             &\leq 
             \frac{1
             }{4\lambda}\Big(\gamma H ^2\sum_{i = 1}^d \|\phi_i(s,a)\mathbf{1}_i\|_{\bLambda_h^{-1}} + \sum_{i=1}^d \|\phi_i(s,a)\mathbf{1}_i\|_{\bLambda_h^{-1}}\Big\|\sum_{\tau=1}^{K} \bphi(s_h^\tau, a_h^{\tau})\eta_h^{\tau}([\hat{V}^{\lambda}_{h+1}]_{\alpha_i}^2)\Big\|_{\bLambda_h^{-1}}\Big), \label{upper bound for (ii)}
         \end{align}
         where \eqref{upper bound for (ii)} follows by the Cauchy Schwartz inequality and the fact that $\EE_{s \sim \mu_{h,i}^0}[\hat{V}^{\lambda}_h(s)]_{\alpha_i}^2 \leq H^2, \forall i \in [d]$.
         Hence combining \eqref{upper bound for suboptimality}, \eqref{upper bound for (i)} and \eqref{upper bound for (ii)}, we have 
        \begin{align*}
            &\mathcal{T}_h^{\lambda} \hat{V}^{\lambda}_{h+1}(s,a) -\la \bphi(s,a), \hat{\bw}_h^\lambda \ra   \\
            & \leq (1 + H/2\lambda)\Big(\gamma H \sum_{i=1}^d\|\phi_i(s,a)\mathbf{1}_i\|_{\bLambda_h^{-1}} +\sum_{i=1}^d\|\phi_i(s,a)\mathbf{1}_i\|_{\bLambda_h^{-1}}\Big\|\sum_{\tau = 1}^K\bphi(s_h^\tau, a_h^{\tau})\eta_{h}^{\tau}([\hat{V}^{\lambda}_{h+1}]_{\alpha'_i})\Big\|_{\bLambda_{h}^{-1}} \Big) \\
              &\quad + \frac{1
             }{4\lambda} \Big(\gamma H ^2\sum_{i = 1}^d \|\phi_i(s,a)\mathbf{1}_i\|_{\bLambda_h^{-1}} + \sum_{i=1}^d \|\phi_i(s,a)\mathbf{1}_i\|_{\bLambda_h^{-1}}\Big\|\sum_{\tau=1}^{K} \bphi(s_h^\tau, a_h^{\tau})\eta_h^{\tau}([\hat{V}^{\lambda}_{h+1}]_{\alpha'_i}^2)\Big\|_{\bLambda_h^{-1}}  \Big).
         \end{align*}
         On the other hand, we can similarly deduce that there exists $\alpha'_i$ s.t. 
         \begin{align*}
            &\la \bphi(s,a), \hat{\bw}_h^\lambda \ra- \mathcal{T}_h^{\lambda} \hat{V}^{\lambda}_{h+1}(s,a)  \\
            & \leq (1 + H/2\lambda)\Big(\gamma H \sum_{i=1}^d\|\phi_i(s,a)\mathbf{1}_i\|_{\bLambda_h^{-1}} +\sum_{i=1}^d\|\phi_i(s,a)\mathbf{1}_i\|_{\bLambda_h^{-1}}\Big\|\sum_{\tau = 1}^K\bphi(s_h^\tau, a_h^{\tau})\eta_{h}^{\tau}([\hat{V}^{\lambda}_{h+1}]_{\alpha'_i})\Big\|_{\bLambda_{h}^{-1}} \Big)\\
              &\quad + \frac{1
             }{4\lambda} \Big(\gamma H ^2\sum_{i = 1}^d \|\phi_i(s,a)\mathbf{1}_i\|_{\bLambda_h^{-1}} + \sum_{i=1}^d \|\phi_i(s,a)\mathbf{1}_i\|_{\bLambda_h^{-1}}\Big\|\sum_{\tau=1}^{K} \bphi(s_h^\tau, a_h^{\tau})\eta_h^{\tau}([\hat{V}^{\lambda}_{h+1}]_{\alpha'_i}^2)\Big\|_{\bLambda_h^{-1}} \Big).
         \end{align*}
         Then for all $i \in [d]$, there exists $\hat{\alpha}_i \in \{ \alpha_i, \alpha'_i \}$, such that 
         \begin{align*}
            &| \mathcal{T}_h^{\lambda} \hat{V}^{\lambda}_{h+1}(s,a) -\la \bphi(s,a), \hat{\bw}_h^\lambda \ra| \\& \leq (1 + H/2\lambda)\Big(\gamma H \sum_{i=1}^d\|\phi_i(s,a)\mathbf{1}_i\|_{\bLambda_h^{-1}}\Big)+ \gamma (H ^2/4\lambda) \sum_{i = 1}^d \|\phi_i(s,a)\mathbf{1}_i\|_{\bLambda_h^{-1}} \\ &\quad +\sum_{i=1}^d\|\phi_i(s,a)\mathbf{1}_i\|_{\bLambda_h^{-1}}\Big(( 1+ H/2\lambda) \Big\|\sum_{\tau = 1}^K\bphi(s_h^\tau, a_h^{\tau})\eta_{h}^{\tau}([\hat{V}^{\lambda}_{h+1}]_{ \hat{\alpha}_i})\Big\|_{\bLambda_{h}^{-1}}
             + (1/4\lambda) \Big\|\sum_{\tau=1}^{K} \bphi(s_h^\tau, a_h^{\tau})\eta_h^{\tau}([\hat{V}^{\lambda}_{h+1}]_{\hat{\alpha}_i}^2)\Big\|_{\bLambda_h^{-1}} \Big).
         \end{align*}
         Now it remains to bound the terms
         \begin{align*}
             \underbrace{\Big\|\sum_{\tau = 1}^K\bphi(s_h^\tau, a_h^{\tau})\eta_{h}^{\tau}([\hat{V}^{\lambda}_{h+1}]_{\hat{\alpha}_i})\Big\|_{\bLambda_{h}^{-1}}}_{(\text{iii})}~\text{and}~\underbrace{\Big\|\sum_{\tau=1}^{K} \bphi(s_h^\tau, a_h^{\tau})\eta_h^{\tau}([\hat{V}^{\lambda}_{h+1}]_{\hat{\alpha}_i}^2)\Big\|_{\bLambda_h^{-1}}}_{\text{(iv)}}.
         \end{align*}
        Similar to the proof in KL divergence, we aim to find a union function class $\mathcal{V}_{h+1}(R_0, B_0, \lambda)$, which holds uniformly that $\hat{V}^{\lambda}_{h+1} \in \mathcal{V}_{h+1}(R_0, B_0, \lambda)$, here for all $h \in [H]$, the function class is defined as:
        \noindent
         \begin{align*}
            \mathcal{V}_{h}(R_0, B_0, \lambda) = \{ V_h(x; \btheta, \beta, \bLambda): \mathcal{S} \rightarrow [0, H], \|\btheta\|_2 \leq R_0, \beta \in [0, B_0], \gamma_{\min}(\bLambda_h) \geq
            \gamma \},
         \end{align*}
         where $V_h(x; \btheta, \beta, \bLambda) = \max_{a \in \mathcal{A}}[\bphi(s,a)^\top\btheta - \beta \sum_{i=1}^d\|\phi_i(s,a)\|_{\bLambda_h^{-1}}]_{[0, H-h+1]}$. By \Cref{lemma:bound of weights in chi2 distance},  when we set $R_0 = \sqrt{d}(H + H^2/2\lambda), B_0 = \beta = 8 dH(1+ 3H/4\lambda) \sqrt{\xi_{\chi^2}}$, it suffices to show that $\hat{V}_{h+1}^\lambda \in \mathcal{V}_{h+1}(R_0, B_0, \gamma)$.
         Next we aim to find a union cover of the $\mathcal{V}_{h+1}(R_0, B_0, \gamma)$, hence the term (iii) and (iv) can be upper bounded. Let $\mathcal{N}_h(\epsilon; R_0, B_0, \gamma)$ be the minimum $\epsilon$-cover of $\mathcal{V}_h(R, B, \lambda)$ with respect to the supreme norm, $\mathcal{N}_{h}([0, H])$ be the minimum $\epsilon$-cover of $[0, H]$ respectively. In other words, for any function $V \in \mathcal{V}_{h}(R, B, \lambda), \alpha \in [0, H]$, there exists a function $V' \in \mathcal{V}_{h}(R, B, \lambda)$ and a real number $\alpha_{\epsilon} \in [0,H]$ such that:
         \begin{align*}
            \sup_{s \in \mathcal{S}}|V(s) - V'(s)| \leq \epsilon, |\alpha - \alpha_{\epsilon}| \leq \epsilon.
         \end{align*}
         Recall the definition of (iii) and (iv). By Cauchy-Schwartz inequality and the fact that $\|a + b\|_{\bLambda_h^{-1}}^2 \leq 2\|a\|_{\bLambda_h^{-1}}^2 + 2\|b\|_{\bLambda_h^{-1}}^2$, we have
         \begin{align}
            \text{(iii)}^2 &\leq 2 \Big\|\sum_{\tau = 1}^K \bphi(s_h^\tau, a_h^{\tau})\eta_{h}^{\tau}([\hat{V}^{\lambda}_{h+1}]_{\alpha_{\epsilon}})\Big\|^2_{\bLambda_h^{-1}} + 2 \Big\|\sum_{\tau = 1}^K \bphi(s_h^\tau, a_h^{\tau})\eta_{h}^{\tau}([\hat{V}^{\lambda}_{h+1}]_{\hat{\alpha}} - [\hat{V}^{\lambda}_{h+1}]_{\alpha_{\epsilon}})\Big\|^2_{\bLambda_h^{-1}} \nonumber \\
            & \leq 4 \Big\|\sum_{\tau = 1}^K \bphi(s_h^\tau, a_h^{\tau})\eta_{h}^{\tau}([V'_{h+1}]_{\alpha_{\epsilon}})\Big\|^2_{\bLambda_h^{-1}} + 4 \Big\|\sum_{\tau = 1}^K \bphi(s_h^\tau, a_h^{\tau})\eta_{h}^{\tau}([\hat{V}^{\lambda}_{h+1}]_{\alpha_{\epsilon}} - [V'_{h+1}]_{\alpha_{\epsilon}})\Big\|^2_{\bLambda_h^{-1}} + \frac{2\epsilon^2K^2}{\gamma},
            \label{(chi2eq.1)}
         \end{align}
         where \eqref{(chi2eq.1)} follows by the fact that 
         \begin{align*}
            2 \Big\|\sum_{\tau = 1}^K \bphi(s_h^\tau, a_h^{\tau})\eta_{h}^{\tau}([\hat{V}^{\lambda}_{h+1}]_{\hat{\alpha}} - [\hat{V}^{\lambda}_{h+1}]_{\alpha_{\epsilon}})\Big\|^2_{\bLambda_h^{-1}} \leq 2 \epsilon^2 \sum_{\tau = 1, \tau' = 1}^{K}|\bphi_h^\tau\bLambda_h^{-1}\bphi_h^{\tau'}| \leq \frac{2\epsilon^2K^2}{\gamma}.
         \end{align*}
         Meanwhile, by the fact that $|[\hat{V}^{\lambda}_{h+1}]_{\alpha_{\epsilon}} - [V'_{h+1}]_{\alpha_{\epsilon}}| \leq |\hat{V}^{\lambda}_{h+1} - V'_{h+1}|$, we have
         \begin{align*}
            &4 \Big\|\sum_{\tau = 1}^K \bphi(s_h^\tau, a_h^{\tau})\eta_{h}^{\tau}([\hat{V}^{\lambda}_{h+1}]_{\alpha_{\epsilon}} - [V'_{h+1}]_{\alpha_{\epsilon}})\Big\|^2_{\bLambda_h^{-1}} \\
            &\leq 4 \sum_{\tau = 1, \tau' = 1}^{K}|\bphi_h^\tau\bLambda_h^{-1}\bphi_h^{\tau'}|
            \max|\eta_{h}^{\tau}([\hat{V}^{\lambda}_{h+1}]_{\alpha_{\epsilon}} - [V'_{h+1}]_{\alpha_{\epsilon}})|^2 \\
            &\leq 4 \epsilon^2\sum_{\tau = 1, \tau' = 1}^{K}|\bphi_h^\tau\bLambda_h^{-1}\bphi_h^{\tau'}| \\
            &\leq \frac{4\epsilon^2K^2}{\gamma}.
         \end{align*}
         By applying the above two inequalities and the union bound into (\ref{(chi2eq.1)}), we have
         \begin{align*}
            (\text{iii})^2 \leq 4 \sup_{V' \in \mathcal{N}_{h}(\epsilon; R_0, B_0, \gamma), \alpha_{\epsilon} \in \mathcal{N}_{h}([0, H])}\Big\|\sum_{\tau = 1}^K \bphi(s_h^\tau, a_h^{\tau})\eta_{h}^{\tau}([V'_{h+1}]_{\alpha_{\epsilon}})\Big\|^2_{\bLambda_h^{-1}} + \frac{6\epsilon^2K^2}{\gamma}.
         \end{align*}
         By \Cref{lemma:concentration bound}, applying a union bound over $\mathcal{N}_{h}(\epsilon; R_0, B_0, \gamma)$ and $\mathcal{N}_{h}([0, H])$, with probability at least $1 - \delta / 2H$, we have
         \begin{align*}
            &4 \sup_{V' \in \mathcal{N}_{h}(\epsilon; R_0, B_0, \gamma), \alpha_{\epsilon} \in \mathcal{N}_{h}([0, H])}\Big\|\sum_{\tau = 1}^K \bphi(s_h^\tau, a_h^{\tau})\eta_{h}^{\tau}([V'_{h+1}]_{\alpha_{\epsilon}})\Big\|^2_{\bLambda_h^{-1}} + \frac{6\epsilon^2K^2}{\gamma} \\
            &\leq 4 H^2\Big(2\log\frac{2H|\mathcal{N}_{h}(\epsilon; R_0, B_0, \gamma)\|\mathcal{N}_{h}([0, H])|}{\delta} + d\log(1+ K/\gamma)\Big) + \frac{6\epsilon^2K^2}{\gamma} .
         \end{align*}
         Similarly, by the fact that $\|a + b\|_{\bLambda_h^{-1}}^2 \leq 2\|a\|_{\bLambda_h^{-1}}^2 + 2\|b\|_{\bLambda_h^{-1}}^2$, noticing (iv) has the almost same form as (iii), we have 
         \begin{align}
            \text{(iv)}^2 &\leq 2 \Big\|\sum_{\tau = 1}^K \bphi(s_h^\tau, a_h^{\tau})\eta_{h}^{\tau}([\hat{V}^{\lambda}_{h+1}]^2_{\alpha_{\epsilon}})\Big\|^2_{\bLambda_h^{-1}} + 2 \Big\|\sum_{\tau = 1}^K \bphi(s_h^\tau, a_h^{\tau})\eta_{h}^{\tau}([\hat{V}^{\lambda}_{h+1}]^2_{\hat{\alpha}} - [\hat{V}^{\lambda}_{h+1}]^2_{\alpha_{\epsilon}})\Big
            \|^2_{\bLambda_h^{-1}} \nonumber \\
            & \leq 4 \Big\|\sum_{\tau = 1}^K \bphi(s_h^\tau, a_h^{\tau})\eta_{h}^{\tau}([V'_{h+1}]^2_{\alpha_{\epsilon}})\Big\|^2_{\bLambda_h^{-1}} + \frac{24H^2\epsilon^2K^2}{\gamma}, \label{eq1}
         \end{align}
          where \eqref{eq1} follows by the fact that
         \begin{align*}
            2 \Big\|\sum_{\tau = 1}^K \bphi(s_h^\tau, a_h^{\tau})\eta_{h}^{\tau}([\hat{V}^{\lambda}_{h+1}]^2_{\hat{\alpha}} - [\hat{V}^{\lambda}_{h+1}]^2_{\alpha_{\epsilon}})\Big\|^2_{\bLambda_h^{-1}} &\leq 8H^2 \epsilon^2 \sum_{\tau = 1, \tau' = 1}^{K}|\bphi_h^\tau\bLambda_h^{-1}\bphi_h^{\tau'}|
            \leq \frac{8H^2\epsilon^2K^2}{\gamma},
        \end{align*}
        and 
        \begin{align*}
           & 4 \|\sum_{\tau = 1}^K \bphi(s_h^\tau, a_h^{\tau})\eta_{h}^{\tau}([\hat{V}^{\lambda}_{h+1}]^2_{\alpha_{\epsilon}} - [V'_{h+1}]^2_{\alpha_{\epsilon}})\|^2_{\bLambda_h^{-1}} \\
           &\leq 4 \sum_{\tau = 1, \tau' = 1}^{K}|\bphi_h^\tau\bLambda_h^{-1}\bphi_h^{\tau'}|
            \max|\eta_{h}^{\tau}([\hat{V}^{\lambda}_{h+1}]_{\alpha_{\epsilon}} - [V'_{h+1}]_{\alpha_{\epsilon}})|^2 \\
            &\leq 16 H^2 \epsilon^2\sum_{\tau = 1, \tau' = 1}^{K}|\bphi_h^\tau\bLambda_h^{-1}\bphi_h^{\tau'}| \\
            &\leq \frac{16H^2\epsilon^2K^2}{\gamma}.
         \end{align*}
         We apply the union bound and \Cref{lemma:concentration bound}, with probability at least $1 - \delta/2H$
         \begin{align*}
            \text{(iv)}^2
            & \leq  4 H^4\Big(2\log\frac{2H|\mathcal{N}_{h}(\epsilon; R_0, B_0, \gamma)\|\mathcal{N}_{h}([0, H])|}{\delta} + d\log(1+ K/\gamma)\Big) + \frac{24H^2\epsilon^2K^2}{\gamma}.
         \end{align*}
        Therefore, with probability at least $1 - \delta$,
         \begin{align}
            &| \mathcal{T}_h^{\lambda} \hat{V}^{\lambda}_{h+1}(s,a) -\la \bphi(s,a), \hat{\bw}_h^\lambda \ra|  \\&\leq \gamma H(1 + H/2\lambda) \sum_{i=1}^d\|\phi_i(s,a)\mathbf{1}_i\|_{\bLambda_h^{-1}}+ (\gamma H ^2/4\lambda)\sum_{i = 1}^d \|\phi_i(s,a)\mathbf{1}_i\|_{\bLambda_h^{-1}} \nonumber \\ &\quad+\sum_{i=1}^d\|\phi_i(s,a)\mathbf{1}_i\|_{\bLambda_h^{-1}}\Big(( 1+ H/2\lambda) \Big\|\sum_{\tau = 1}^K\bphi(s_h^\tau, a_h^{\tau})\eta_{h}^{\tau}([\hat{V}^{\lambda}_{h+1}]_{ \hat{\alpha}_i})\Big\|_{\bLambda_{h}^{-1}}
             + (1/4\lambda)\Big\|\sum_{\tau=1}^{K} \bphi(s_h^\tau, a_h^{\tau})\eta_h^{\tau}([\hat{V}^{\lambda}_{h+1}]_{\hat{\alpha}_i}^2)\Big\|_{\bLambda_h^{-1}} \Big) \nonumber
            \\
            &\leq\sum_{i=1}^d\|\phi_i(s,a)\mathbf{1}_i\|_{\bLambda_h^{-1}}\Bigg[\gamma H(1 + 3H/4\lambda) \nonumber \\ 
            &\quad+ (1 + H/2\lambda)\sqrt{4 H^2\Big(2\log\frac{2H|\mathcal{N}_{h}(\epsilon; R_0, B_0, \gamma)\|\mathcal{N}_{h}([0, H])|}{\delta} + d\log(1+ K/\gamma)\Big) + \frac{6\epsilon^2K^2}{\gamma}} \nonumber\\
            & \quad+ (1/4\lambda)\sqrt{4H^4\Big(2\log\frac{2H|\mathcal{N}_{h}(\epsilon; R_0, B_0, \gamma)\|\mathcal{N}_{h}([0, H])|}{\delta} + d\log(1+ K/\gamma)\Big) + \frac{24H^2\epsilon^2K^2}{\gamma}} \Bigg].
            \label{(chi2)}
         \end{align}
         By the fact that $R_0 = \sqrt{d}(H + H^2/2\lambda)$, \Cref{lemma:covering number of function class} and \Cref{lemma:covering number of interval}, we can upper bound the term $|\mathcal{N}_h(\epsilon; R_0,B_0,\gamma)|$ and $|\mathcal{N}_{h}([0,H])|$ as follows: 
    \begin{align*}
        \log|\mathcal{N}_h(\epsilon; R_0,B_0,\gamma)| \leq d\log(1+ 4R_0/\epsilon) + d^2\log(1+ 8d^{1/2}B^2/\gamma \epsilon^2),
        \log |\mathcal{N}_{h}([0,H])| \leq \log(3H/\epsilon).
    \end{align*}
    \noindent
    We set $\epsilon = \frac{1}{K}, \gamma = 1$, with the upper bound above, we have 
    \begin{align*}
    &4 H^2\Big(2\log\frac{2H|\mathcal{N}_{h}(\epsilon; R_0, B_0, \gamma)\|\mathcal{N}_{h}([0, H])|}{\delta} + d\log(1+ K/\gamma)\Big) + \frac{6\epsilon^2K^2}{\gamma} \\
    &\leq 4 H^2\Big(2\log\frac{6H^2K}{ \delta} + d\log(1 + K) + d\log(1+d^{1/2}(1+H/2\lambda)K) + d^2\log(1 + d^{1/2}B^2K^2) + 3/2\Big) \\
    &\leq 4 H^2 \Big(2d\log\frac{6H^2K}{ \delta} + 2d\log(K) + d\log(d^{1/2}(1+H/2\lambda)K) + 2d^2 \log d^{1/2}B^2K^2\Big) \\
    & \leq 8H^2d^2\Big(\log\frac{6H^2K}{ \delta} + \log(K) + \log(d^{1/2}(1+H/2\lambda)K) +  \log 8d^{1/2}B^2K^2\Big)  \\
    & = 8H^2d^2 (\log 48 K^5 H^2 B^2 d (1 + H/2\lambda)/ \delta).
    \end{align*}
    Similarly, we can upper bound the third term in (\ref{(chi2)}) as follows:
    \begin{align*}
    &4H^4\Big(2\log\frac{2H|\mathcal{N}_{h}(\epsilon; R_0, B_0, \gamma)\|\mathcal{N}_{h}([0, H])|}{\delta} + d\log(1+ K/\gamma)\Big) + \frac{24H^2\epsilon^2K^2}{\gamma} \\
    & = H^2 \Big( 4 H^2(2\log\frac{2H|\mathcal{N}_{h}(\epsilon; R_0, B_0, \gamma)\|\mathcal{N}_{h}([0, H])|}{\delta} + d\log(1+ K/\gamma)) + \frac{24\epsilon^2K^2}{\gamma}\Big) \\
    &\leq 8H^4d^2 (\log 48 K^5H^2 B^2 d (1 + H/2\lambda)/ \delta).
    \end{align*}
    \noindent
    Hence, we apply this bound into the (\ref{(chi2)}), we have
    \begin{align}
        &|\la \bphi(s,a), \hat{\bw}_h^\lambda \ra- \mathcal{T}_h^{\lambda} \hat{V}^{\lambda}_{h+1}(s,a)| \nonumber\\
        & \leq \sum_{i=1}^d \|\bphi(s,a)\mathbf{1}_i\|_{\bLambda_h^{-1}}(H(1 + 3H/4\lambda) \notag\\
        &\qquad + 
         (1+ 3H/4\lambda)\sqrt{8H^2d^2 (\log 48 K^5H^2 B^2 d (1 + H/2\lambda)/ \delta)}) \nonumber\\
        &\leq \sum_{i=1}^d \|\bphi(s,a)\mathbf{1}_i\|_{\bLambda_h^{-1}}2(1 + 3H/4\lambda)\sqrt{8H^2d^2 (\log 48 K^5H^2 B^2 d (1 + H/2\lambda)/ \delta)} \nonumber\\
        & \leq \sum_{i=1}^d \|\bphi(s,a)\mathbf{1}_i\|_{\bLambda_h^{-1}} 2(1 + 3H/4\lambda)Hd \sqrt{8(\log 192 K^5H^6 d^3 (1 + H/2\lambda)^3/ \delta) + \log \xi_{\chi^2})} \nonumber\\
        &= \sum_{i=1}^d \|\bphi(s,a)\mathbf{1}_i\|_{\bLambda_h^{-1}} 2(1 + 3H/4\lambda)Hd \sqrt{8(\xi_{\chi^2} + \log \xi_{\chi^2})} \nonumber\\
        & \leq \beta\sum_{i=1}^d\|\bphi(s,a)\mathbf{1}_i\|_{\bLambda_h^{-1}} ,\label{chi4}
    \end{align}
    where \eqref{chi4} comes from the fact that $\log \xi_{\chi^2} \leq \xi_{\chi^2}$. Hence, the prerequisite is satisfied in \Cref{lemma:regularized robust suboptimality lemma}, we can upper bound the suboptimality gap as:
     \begin{align*}
         \text{SubOpt}(\hat{\pi}, s, \lambda) &\leq 2\sup_{P \in \mathcal{U}^\lambda(P^0)}\sum_{h=1}^H\EE^{\pi^\star, P}
         \big[\Gamma_h(s_h, a_h)|s_1 = s\big] \\
         & = 2\beta\sup_{P \in \mathcal{U}^\lambda(P^0)}\sum_{h=1}^H
         \EE^{\pi^\star, P}\Big[\sum_{i=1}^d\|\phi_i(s,a)\mathbf{1}_i\|_{\bLambda_h^{-1}}|s_1 = s\Big].
     \end{align*}
     \noindent
    This concludes the proof.
\end{proof}

\subsection{Proof of \Cref{corollary:simplified upper bounds}}
\begin{proof} The proof follows the argument in (F.15) and (F.16) of \cite{blanchet2024double}. Specifically, we denote 
    \begin{align}
&\bLambda_{h, i}^P=\mathbb{E}^{\pi^\star, P}\Big[(\phi_i(s_h, a_h) \mathbf{1}_i)(\phi_i(s_h, a_h) \mathbf{1}_i)^{\top}\big|s_1=s\Big], \quad \forall(h, i, P) \in[H] \times[d] \times \mathcal{U}^\lambda(P^0). \label{definition of Lambda_hi}
    \end{align}
  By \Cref{robust partial coverage assumption}, setting $\gamma = 1$, we have
    \begin{align}
& \sup _{P \in \mathcal{U}^\lambda(P^0)} \sum_{h=1}^H \mathbb{E}^{\pi^\star, P}\bigg[\sum_{i=1}^d\|\phi_i(s_h, a_h) \mathbf{1}_i\|_{\bLambda_h^{-1}} \big| s_1=s\bigg] \nonumber\\
&= \sup _{P \in \mathcal{U}^\lambda(P^0)} \sum_{h=1}^H \sum_{i=1}^d \mathbb{E}^{\pi^\star, P}\Big[\sqrt{\operatorname{Tr}\big((\phi_i(s_h, a_h) \mathbf{1}_i)(\phi_i(s_h, a_h) \mathbf{1}_i)^{\top} \bLambda_h^{-1}\big)} \big | s_1=s\Big] \nonumber\\
& \leq \sup _{P \in \mathcal{U}^\lambda(P^0)} \sum_{h=1}^H \sum_{i=1}^d \sqrt{\operatorname{Tr}\big(\mathbb{E}^{\pi^\star, P}\big[(\phi_i(s_h, a_h) \mathbf{1}_i)(\phi_i(s_h, a_h) \mathbf{1}_i)^{\top} | s_1=s\big] \bLambda_h^{-1}\big)} \label{B1}\\
& \leq \sup_{P \in \mathcal{U}^\lambda(P^0)} \sum_{h=1}^H \sum_{i=1}^d \sqrt{\operatorname{Tr}\big(\bLambda_{h, i}^P \cdot \big(\mathbf{I}+K \cdot c^\dagger \cdot \bLambda_{h, i}^P\big)^{-1}\big)} \label{B_2}\\
& =\sup _{P \in \mathcal{U}^\lambda(P^0)} \sum_{h=1}^H \sum_{i=1}^d \sqrt{\frac{( \EE^{\pi^\star,P}[\phi_i(s_h, a_h) | s_1=s])^2}{1+c^\dagger \cdot K \cdot(\EE^{\pi^\star,P}[\phi_i(s_h, a_h) | s_1=s])^2}} \nonumber\\
& \leq \sup _{P \in \mathcal{U}^\lambda\left(P^0\right)} \sum_{h=1}^H \sum_{i=1}^d \sqrt{\frac{1}{c^\dagger \cdot K}} \label{B_3}\\
& =\frac{d H}{ \sqrt{c^\dagger K}},\nonumber
\end{align}
where (\ref{B1}) is due to the Jensen's inequality, (\ref{B_2}) holds by the definition in (\ref{definition of Lambda_hi}) and \Cref{robust partial coverage assumption}, (\ref{B_3}) holds by the fact that the only nonzero element of $\bLambda_{h, i}^P$ is the $i$-th diagonal element. Thus, by \Cref{thm: instance-dependent upper bounds}, with probability at least $1-\delta$, for any $s \in \cS$ the suboptimality can be upper bounded as: 
    \begin{align*}
        \text{SubOpt}(\hat{\pi}, s, \lambda) \leq         \beta \sup_{P \in \mathcal{U}^\lambda(P^0)}\sum_{h=1}^H\EE^{\pi^\star, P}\Big[\sum_{i=1}^d \| \phi_i(s,a)\mathbf{1}_i\|_{{\bLambda}_h^{-1}}\big|s_1 = s\Big] \leq \frac{\beta dH}{\sqrt{c^\dagger K}},
    \end{align*}
    where 
        \begin{align*}
            \beta = \begin{cases}
                16 Hd \sqrt{\xi_{\text{TV}}}, &  \text{if D is TV;} \\
                16d\lambda  e^{H/\lambda} \sqrt{(H/\lambda + \xi_{\text{KL}})}, & \text{if D is KL;} \\
                8 dH(1+ 3H/4\lambda) \sqrt{\xi_{\chi^2}}, &\text{if D is $\chi^2$.} 
            \end{cases} 
        \end{align*}
 Hence, we conclude the proof. 
\end{proof}

\section{Proof of the Information-Theoretic Lower Bound}
\label{sec:proof of lower bound}
In this section, we prove the information-theoretic lower bound. We first introduce the construction of hard instances in \Cref{sec:Construction of Hard Instances}, then we prove \Cref{th:lower bound} in  \Cref{sec:proof of lower bound-subsection}

\subsection{Construction of Hard Instances}
\label{sec:Construction of Hard Instances}
\begin{figure}[ht]
    \centering
    \subfigure[The source MDP environment.]{
\begin{tikzpicture}[->,>=stealth',shorten >=1pt,auto,node distance=3cm,semithick]
  \node[circle,draw,thick,minimum size=1.5cm] (s1) {$s_1$};
  \node[circle,draw,thick,minimum size=1.5cm,right of=s1] (s2) {$s_2$};

  \path[->,thick]
  (s1) edge[loop above] node[above] {$1 - \epsilon$} (s1)
  (s2) edge[loop above] node[above] {$1$} (s2);

  \path[->,thick]
  (s1) edge[bend right] node[below=0.2] {$\epsilon$} (s2);

\end{tikzpicture}
\label{fig:illustration of the hard instance-source}
    }
    \subfigure[The target MDP environment.]{
\begin{tikzpicture}[->,>=stealth',shorten >=1pt,auto,node distance=3cm,semithick]
  \node[circle,draw,thick,minimum size=1.5cm] (s1) {$s_1$};
  \node[circle,draw,thick,minimum size=1.5cm,right of=s1] (s2) {$s_2$};

  \path[->,thick]
  (s1) edge[loop above] node[above=0] {$1 - \epsilon - \Delta_h^\lambda(\epsilon,D)$} (s1)
  (s2) edge[loop above] node[above] {$1$} (s2);

  \path[->,thick]
  (s1) edge[bend right] node[below=0] {$\epsilon + \Delta_h^\lambda(\epsilon,D)$} (s2);
\end{tikzpicture}
         \label{fig:illustration of the hard instance-target}
     }
     \caption{The nominal environment and the worst case environment. The value on each arrow represents the transition probability.
The MDP has two states, $s_1$ and $s_2$, and  $H$ steps. For he nominal environment on the left, the $s_1$ is the good state where the transition is determined by an error term $\epsilon$, and $s_2$ is a fail state with reward 0 and only transitions to itself.
The worst case environment on the right is obtained by perturbing the transition probability at each step of the nominal environment. The magnitude of the perturbation $\Delta_h^\lambda(\epsilon,D)$ at each stage $h$ depends on the divergence metric $D$, the regularized $\lambda$ and the parameter $\epsilon$.
     }
     \label{fig: hard instances}
 \end{figure}
The construction of the information-theoretic lower bound relies on a novel family of hard instances. We illustrate one such instance in \Cref{fig: hard instances}. Both the nominal and target environments satisfy \Cref{assumption:linear MDP}. The environment consists of two states, $s_1$ and $s_2$.
In the nominal environment \Cref{fig:illustration of the hard instance-source}, $s_1$ represents the good state with a positive reward. For any transition originating from $s_1$, there is a $1-\epsilon$ probability of transitioning to itself and an $\epsilon$ probability of transitioning to $s_2$ regardless of the action taken, where $\epsilon$ is a parameter to be determined. The state $s_2$ is a fail state with zero reward and can only transition to itself.
The worst-case target environment \Cref{fig:illustration of the hard instance-target} is obtained by perturbing the transition probabilities in the nominal environment. The perturbation magnitude $\Delta_h^{\lambda}(\epsilon,D)$ depends on the stage $h$, regularizer $\lambda$, divergence metric $D$, and parameter $\epsilon$.

The family of $d$-rectangular linear RRMDPs are parameterized by a Boolean vector $\bxi=\{\bxi_h\}_{h\in[H]}$, where $\bxi_h\in\{-1,1\}^d$.
For a given $\bxi$ and regularizer $\lambda$, the corresponding $d$-rectangular linear RRMDP $M_{\bxi}^\rho$ is constructed as follows. The state space $\cS=\{x_1, x_2\}$ and the action space $\cA=\{0,1\}^d$. 
The initial state distribution $\mu_0$ is defined as 
\begin{align*}
    \mu_0(s_1)=\frac{d+1}{d+2}\quad \text{and} \quad \mu_0(x_2)=\frac{1}{d+2}.
\end{align*}
The feature mapping $\bphi:\cS\times\cA\rightarrow\RR^{d+2}$ is defined as
\begin{align*}
    &\bphi(s_1,a)^{\top}=\Big(\frac{a_1}{d}, \frac{a_2}{d}, \cdots, \frac{a_d}{d}, 1-\sum_{i=1}^d\frac{a_i}{d}, 0 \Big),\\
    &\bphi(s_2,a)^{\top}=\big(0, 0, \cdots, 0, 0, 1\big),
\end{align*}
which satisfies $\phi_i(s,a)\geq 0$ and $\sum_{i=1}^d\phi_i(s,a)=1$.
The nominal distributions $\{\bmu_h^0\}_{h\in[H]}$ are defined as 
\begin{align*}
    \bmu_h^0=\big(\underbrace{(1-\epsilon)\delta_{s_1} +\epsilon \delta_{s_2}, (1-\epsilon)\delta_{s_1} + \epsilon \delta_{s_2}, \cdots, (1-\epsilon)\delta_{s_1} + \epsilon \delta_{s_2}}_\text{$d+1$}, \delta_{s_2}\big)^{\top}, \forall h\in[H],
\end{align*}
where $\epsilon$ is an error term injected into the nominal model, which is to be determined later.
Thus, the transition is homogeneous and does not depend on action but only on state.  The reward parameters $\{ \btheta_h \}_{h \in [H]}$ are defined as 
    \begin{align*}
        \btheta_h^{\top} = \delta\cdot\Big(\frac{\xi_{h1}+1}{2}, \frac{\xi_{h2}+1}{2},\cdots\frac{\xi_{hd}+1}{2}, \frac{1}{2}, 0\Big), \forall h \in [H],
    \end{align*}
where $\delta$ is a parameter to control the differences among instances, which is to be determined later.
The reward $r_h$ is generated from the normal distribution $r_h\sim\cN(r_h(s_h,a_h),1)$, where $r_h(s,a)=\bphi(s,a)^\top\btheta_h$. Note that
\begin{align*}
    r_h(s_1, a)=\bphi(s_1,a)^\top\btheta_h = \frac{\delta}{2d}\big(\la\bxi_h, a\ra+d\big)\geq 0\quad \text{and} \quad r_h(s_2, a)=\bphi(s_2,a)^\top\btheta_h=0,~ \forall a\in\cA,
\end{align*}
Thus, the worst case transition kernel should have the highest possible transition probability to $s_2$, and the optimal robust policy should lead to a transition probability to $s_2$ as small as possible. 
Therefore the optimal action at step $h$ is 
    \begin{align*}
        a_h^\star = \Big(\frac{1+\xi_{h1}}{2}, \frac{1+\xi_{h2}}{2} \cdots,\frac{1+\xi_{hd}}{2}\Big).
    \end{align*}
We illustrate the designed $d$-rectangular linear RRMDP $M_{\bxi}^{\lambda}$ in \Cref{fig:illustration of the hard instance-source} and \Cref{fig:illustration of the hard instance-target}.

Finally, the offline dataset is collected by the following procedure: the behavior policy $\pi^b=\{\pi_h^b\}_{h\in[H]}$ is defined as 
 \[\pi_h^b\sim \text{Unif}\big(\{\be_1, \cdots, \be_d, \mathbf{0}\}\big),\forall h\in [H],
 \]
 where $\{\be_i\}_{i\in[d]}$ are the canonical basis vectors in $\RR^d$. The initial state is generated according to $\mu_0$ and then the behavior policy interact with the nominal environment $K$ episodes to collect the offline dataset $\cD$.
\begin{remark}
We would like to highlight the difference between our hard instances and the hard instances developed in \citet{liu2024minimax}. We find out that instances developed in \citet{liu2024minimax} only allow perturbations measured in TV-divergence. The reason is that in their nominal environment, both $s_1$ and $s_2$ are absorbing states, and thus $P^0_h(\cdot|s,a)$ only has support on $s$, which could be either $s_1$ or $s_2$. In this case, any perturbation to $P^0_h(\cdot|s,a)$ would cause a violation of the absolute continuous condition in the definition of the KL-divergence and the $\chi^2$-divergence\footnote{It has been shown in Proposition 2.5 of \citet{lu2024distributionally} that the TV divergence can be extended to allow for two distributions with different support.}.  In comparison, we inject a small error $\epsilon$ in the nominal kernel such that $P^0_h(\cdot|s_1,a)$ has full support $\{s_1,s_2\}$ when the transition starts from $s_1$. Hence, we can make perturbations on $P^0_h(\cdot|s_1,a)$ safely without violating the absolutely continuous condition. Additionally, \citet{liu2024minimax} only construct perturbation in the first stage, while we admit perturbation in every stage $h$ in order to make our instance more general.
\end{remark}

\subsection{Proof of \Cref{th:lower bound}}
\label{sec:proof of lower bound-subsection}
With this family of hard instances, we are ready to prove the information-theoretic lower bound. For any $\bxi \in \{-1,1\}^{dH}$, let $\QQ_{\bxi}$ denote the distribution of dataset $\cD$ collected from the instance $M_{\bxi}$. Denote the family of parameters as $\Omega = \{-1,1\}^{dH}$ and the family of hard instances as $\cM=\{M_{\bxi}:\bxi\in\Omega\}$.
Before the proof, we introduce the following lemma bounding the robust value function.

\begin{lemma}
        \label{lemma: bounding robust value function - hard instance}
        Under the constructed hard instances in \Cref{sec:Construction of Hard Instances}, let $\delta = d^{3/2}/\sqrt{2K}$ and $K>d^3H^2/(2\lambda^2)$.
        For any $h\in [H]$ , we have
        \begin{align}
             0 \leq 
           \frac{\delta}{2d}\sum_{j=h}^H ( 1- \epsilon)^{j-h} \Big( d + \Big(\sum_{i=1}^d\xi_{ji}\EE^{\pi}a_{ji}\Big)\Big)-V^{\pi,\lambda}_h(s_1) \leq f_h^\lambda(\epsilon),
            \label{1}
        \end{align}
        where $f_h^\lambda(\epsilon)$ is a error term, which is defined as:
        \begin{align*}
            f_h^\lambda(\epsilon) = \begin{cases}
                0, &  \text{if D is TV;} \\
                (H-h)\lambda \epsilon (e-1), & \text{if D is KL;} \\
                (H-h)\lambda  \epsilon (1-\epsilon)/4, &\text{if D is $\chi^2$.} 
            \end{cases} 
        \end{align*}
        Furthermore, if we set the $\epsilon$ as
    \begin{align}
        \epsilon = \begin{cases}
                 1 - 2^{-1/H}, & \text{if D is TV;} \\
                \min \{1 - 2^{-1/H}, d^{3/2}/(64\lambda\sqrt{2K})\} , &\text{if D is KL;} \\
                \min \{1 - 2^{-1/H}, d^{3/2}/(8\lambda\sqrt{2K} )\}, &  \text{if D is $\chi^2$,}
                \end{cases}
                \label{definition of epsilon}
    \end{align}
    then we have $f_h^\lambda(\epsilon) \leq d^{3/2} H/{32\sqrt{2K}}$.
    \end{lemma}
\begin{proof}[Proof of \Cref{th:lower bound}]
Invoking \Cref{lemma: bounding robust value function - hard instance},  we have
\begin{align}
        V^{\pi^\star, \lambda}_1(s_1) - V^{\pi, \lambda}_1(s_1)& \geq \frac{\delta}{2d}\sum_{j=1}^H\sum_{i=1}^d(1-\epsilon)^{j-1}\Big(\frac{1+ \xi_{ji}}{2} - \xi_{ji}\EE^{\pi}a_{ji}\Big) - f_1^\lambda(\epsilon) \nonumber\\
        & = \frac{\delta}{4d}\sum_{j=1}^H\sum_{i=1}^d(1-\epsilon)^{j-1}(1 - \xi_{ji}\EE^{\pi}(2a_{ji}-1)) - f_1^\lambda(\epsilon) \nonumber\\
        & = \frac{\delta}{4d}\sum_{j=1}^H\sum_{i=1}^d(1-\epsilon)^{j-1}(\xi_{ji} - \EE^{\pi}(2a_{ji}-1))\xi_{ji} - f_1^\lambda(\epsilon)\nonumber\\
        & = \frac{\delta}{4d}\sum_{j=1}^H\sum_{i=1}^d(1-\epsilon)^{j-1}|\xi_{ji} - \EE^{\pi}(2a_{ji}-1)| - f_1^\lambda(\epsilon) \label{low bound 1}\\
        & \geq \frac{\delta}{4d} (1-\epsilon)^{H-1}\sum_{j=1}^H\sum_{i=1}^d|\xi_{ji} - \EE^{\pi}(2a_{ji}-1)| - f_1^\lambda(\epsilon), \label{eq: lowbound decomposition}
    \end{align}
    where \eqref{low bound 1} follows from the fact that $\xi_{ji} \in \{ -1,1 \}$. To continue, 
    \begin{align}
        \frac{\delta}{4d}\sum_{j=1}^H\sum_{i=1}^d|\xi_{ji} - \EE^{\pi}(2a_{ji}-1)|
        &\geq \frac{\delta}{4d}\sum_{j=1}^H\sum_{i=1}^d|\xi_{ji} - \EE^{\pi}(2a_{ji}-1)| \textbf{1}\{ \xi_{hi} \neq \sign(\EE(2a_{ji}-1)) \}| \nonumber\\
        & \geq \frac{\delta}{4d}\sum_{j=1}^H\sum_{i=1}^d \textbf{1}\{ \xi_{hi} \neq \sign(\EE(2a_{ji}-1)) \}|\nonumber\\
        & = \frac{\delta}{4d}D_{H}(\xi, \xi^{\pi}),   \label{2}
    \end{align}
    where $D_H(\cdot, \cdot)$ is the Hamming distance. Then applying the Assouad's method (Lemma 2.12 in \citet{tsybakov2009nonparametric}), we have 
    \begin{align}
        \inf_{\pi}\sup_{\xi \in \Omega}\EE_{\xi}[D_H(\xi, \xi^{\pi})] &\geq \frac{dH}{2}\min_{D_H(\xi, \xi^{\pi}) =1} \inf_{\phi}[\QQ_{\xi}(\psi(\mathcal{D}) \neq \xi) + \QQ_{\xi^{\pi}}(\psi(\mathcal{D}) \neq \xi^{\pi})] \nonumber\\
        & \geq \frac{dH}{2}\Big(1 - \Big(\frac{1}{2}\max_{D_H(\xi, \xi^{\pi})=1}D_{\text{KL}}(\QQ_{\xi}\|\QQ_{\xi^{\pi}})\Big)^{1/2}\Big), \label{3}
    \end{align}
    where $D_{\text{KL}}$ represents the KL divergence. Next we bound $D_{\text{KL}}(\QQ_{\xi}\|\QQ_{\xi^{\pi}})$, according to the definition of $\QQ_{\xi}(\cD)$, we have 
    \begin{align*}
        \QQ_{\xi}(\cD) = \prod_{k=1}^K \prod_{\tau=1}^H \pi^b_{h}(a_{h}^\tau | s_{h}^\tau)P_h^0(s_{h+1}^{\tau} | s_h^{\tau}, a_h^{\tau})R_h(r_h^{\tau}|s_h^{\tau}, a_h^{\tau}),
    \end{align*}
    where $R_h(r_h^{\tau}|s_h^{\tau}, a_h^{\tau})$ refers to the density function of $\mathcal{N}(r_h(s_h^{\tau},a_h^{\tau}), 1)$ at $r_h^{\tau}$. By the fact that the difference between the two distributions $\QQ_{\xi}(\cD)$ and $\QQ_{\xi^{\pi}}(\cD)$ lie only in the reward distribution corresponding to the index where $\xi$ and $\xi^{\pi}$ differ, we have 
    \begin{align}
        D_{KL}(\QQ_{\xi}(\cD)\|\QQ_{\xi^{\pi}}(\cD)) = \sum_{\tau=1}^{\frac{K}{d+2}}D_{KL}\Big(\mathcal{N}\Big(\frac{d+1}{2d}\delta,1 \Big) \| \mathcal{N}\Big(\frac{d-1}{2d}\delta, 1\Big)\Big) = \frac{K}{d+2} \frac{\delta^2}{d^2} \leq \frac{1}{2}, \label{4}
    \end{align}
    where the last inequality follows from the definition of $\delta$. By the fact that $\delta = d^{3/2}/\sqrt{2K}$, we have
    \begin{align}
         \inf_{\hat{\pi}}\sup_{M \in \mathcal{M}} \text{subopt}(M, \hat{\pi}, s, \lambda) &\geq \frac{\delta H(1-\epsilon)^{H-1}}{4}\Big( 1 - \Big(\frac{1}{2}\max_{D_H(\xi, \xi^{\pi})=1}D_{KL} (\QQ_{\xi}\|\QQ_{\xi^{\pi}})\Big)^{1/2}\Big) - f_h^\lambda(\epsilon)\label{low all 1}\\ 
         & \geq \frac{\delta H(1-\epsilon)^{H-1}}{8} - f_h^\lambda(\epsilon) \label{low all 2}\\
         &= \frac{d^{3/2} H(1-\epsilon)^{H-1}}{8\sqrt{2K}} - f_h^\lambda(\epsilon)\nonumber \\ \label{low1}
         & \geq \frac{d^{3/2} H}{16\sqrt{2K}} - \frac{d^{3/2} H}{32\sqrt{2K}} \\
         & \geq \frac{1}{128\sqrt{2}} \sum_{h=1}^H\EE^{\pi^\star, P}\Big[\sum_{i=1}^d \|\phi_i(s,a)\mathbf{1}_i\|_{{\Lambda}_h^{-1}}|s_1 = s\Big], \label{low all 3}
    \end{align}
     where \eqref{low all 1} holds by applying the inequality \eqref{eq: lowbound decomposition}, (\ref{2}) and (\ref{3}) in order, \eqref{low all 2} holds by (\ref{4}), \eqref{low1} holds by the definition of $\epsilon$ in \eqref{definition of epsilon}, and \eqref{low all 3} holds by \Cref{lemma: bound of dataset}. Hence, it is sufficient for taking $c = 1/128\sqrt{2}$. This concludes the proof.
    
\end{proof}

 \section{Proof of Technical Lemmas}
\subsection{Proof of \Cref{lemma:regularized robust suboptimality lemma}}
\begin{proof}
    We first decompose the $\text{SubOpt}(\pi, s, \lambda)$ as follows:
    \begin{align*}
        \text{SubOpt}(\hat{\pi}, s, \lambda) &= V_1^{\star, \lambda}(s) - V_1^{\hat{\pi},\lambda}(s) \\
        &= \underbrace{V_1^{\star, \lambda}(s) - \hat{V}_{1}^{\lambda}(s)}_{\text{(i)}} + \underbrace{\hat{V}_{1}^{\lambda}(s) - V_1^{\hat{\pi}, \lambda}(s)}_{\text{(ii)}},
    \end{align*}
    where $\hat{V}^{\lambda}_1(s)$ is computed in the algorithm. We next bound the term (i) and (ii) respectively.
    For term (i), the error comes from the estimated error of the value function and the Q-function, therefore by \eqref{prop:regularized Robust Bellman equation} and the definition of the $\hat{Q}_h^{\lambda}(s, a)$ in meta-algorithm, for any $h \in [H]$, we can decompose the error as:
    \begin{align}
        V_h^{\pi^\star, \lambda}(s) - \hat{V}^{\lambda}_h(s) &= Q_h^{\pi^\star, \lambda}(s, \pi^\star_h(s)) - \hat{Q}_h^{\lambda}(s, \hat{\pi}_h(s)) \nonumber\\
        & \leq Q_h^{\pi^\star, \lambda}(s, \pi^\star_h(s)) - \hat{Q}_h^{\lambda}(s, \pi^\star(s)) \label{D1.1}\\
        & = \mathcal{T}_h^{\lambda}V_{h+1}^{\pi^\star, \lambda}(s, \pi_h^\star(s)) - \mathcal{T}_h^{\lambda}\hat{V}^{\lambda}_{h+1}(s, \pi_h^\star(s))
        + \mathcal{T}_h^{\lambda}\hat{V}^{\lambda}_{h+1}(s, \pi_h^\star(s)) - \hat{Q}_h^{\lambda}(s, \pi^\star(s)) \nonumber\\
        & = \mathcal{T}_h^{\lambda}V_{h+1}^{\pi^\star, \lambda}(s, \pi_h^\star(s)) - \mathcal{T}_h^{\lambda}\hat{V}^{\lambda}_{h+1}(s, \pi_h^\star(s)) + \delta^\lambda_h(s,\pi_h^\star(s)), \label{EQ1}
    \end{align}
    where \eqref{D1.1} comes from the fact that $\hat{\pi}_h$ is the greedy policy with respect to $\hat{Q}_h^{\lambda}(s,a)$, the regularized robust Bellman update error $\delta_h^\lambda$ is defined as:
    \begin{align}
        \delta_h^\lambda(s,a) := \mathcal{T}_h^{\lambda}\hat{V}^{\lambda}_{h+1}(s, a) - \hat{Q}_h^{\lambda}(s, a), \forall (s,a) \in \cS \times \cA, \label{definition of regularized bellman error}
    \end{align}
     which aims to eliminate the clip operator in the definition of $\hat{Q}_h^{\lambda}(s, a)$. Denote the worst transition kernel w.r.t the regularized Bellman operator as $\hat{P} = \{ \hat{P}_h \}_{h \in [H]}$, where $\hat{P}_h$ is defined as:
    \begin{align*}
        \hat{P}_h(\cdot |s,a) &=  \argmin_{\bmu_h \in \Delta(\cS)^d,P_h=\la \bphi,\bmu_h \ra}\big[\EE_{s'\sim P_h(\cdot|s,a)}\big[\hat{V}_{h+1}^{\lambda}(s')\big]+ \lambda \la\bphi(s, a),\bD(\bmu_h||\bmu_h^0) \ra\big] \\
        & = \sum_{i=1}^d\phi_i(s,a) \argmin_{\mu_{h, i} \in \Delta(\cS)}\big[\EE_{s' \sim \mu_{h, i}}[\hat{V}^\lambda_{h+1}(s')]+ \lambda D(\mu_{h, i} \|\mu_{h, i}^0)\big]
        \\
        & = \sum_{i=1}^d\phi_i(s,a) \hat{\mu}_{h,i}(\cdot)
        ,
    \end{align*}
    where the $\hat{\mu}_{h,i}$ is defined as $\hat{\mu}_{h,i} = \argmin_{\mu_{h, i} \in \Delta(\cS)}\big[\EE_{s' \sim \mu_{h, i}}[\hat{V}_{h+1}(s')]+ \lambda D(\mu_{h, i} \|\mu_{h, i}^0)\big]$.
    Hence the difference between the regularized Bellman operator and the empirical regularized Bellman operator can be bounded as
    \begin{align}
        &\mathcal{T}_h^{\lambda}V_{h+1}^{\pi^\star, \lambda}(s, \pi_h^\star(s)) - \mathcal{T}_h^{\lambda}\hat{V}^{\lambda}_{h+1}(s, \pi_h^\star(s)) \nonumber\\
        &=  r_h(s, \pi_h^\star(s)) + \inf_{\bmu_h \in \Delta(\cS)^d,P_h=\la \bphi,\bmu_h \ra}\big[\EE_{s'\sim P_h(\cdot|s,\pi_h^\star(s))}\big[V_{h+1}^{\pi^\star, \lambda}(s')\big]+ \lambda \la\bphi(s, \pi_h^\star(s)),\bD(\bmu_h||\bmu_h^0) \ra\big]\nonumber \\
        &\quad - r_h(s, \pi_h^\star(s)) - \inf_{\bmu_h \in \Delta(\cS)^d,P_h=\la \bphi,\bmu_h \ra}\big[\EE_{s'\sim P_h(\cdot|s,\pi_h^\star(s))}\big[\hat{V}_{h+1}^{\lambda}(s')\big]+ \lambda \la\bphi(s, \pi_h^\star(s)),\bD(\bmu_h||\bmu_h^0) \ra\big] \nonumber\\
        &\leq \EE_{s' \sim \hat{P}_h(\cdot |s,\pi_h^\star(s))}[\hat{V}^{\lambda}_{h+1}(s')] - \EE_{s' \sim \hat{P}_h(\cdot |s,\pi_h^\star(s))}[V^{\star,\lambda}_{h+1}(s')] \nonumber\\
        &= \EE_{s' \sim \hat{P}_h(\cdot |s,\pi_h^\star(s))}[\hat{V}^{\lambda}_{h+1}(s')-V^{\star, \lambda}_{h+1}(s')]. \label{EQ2}
    \end{align}
    Combining inequality (\ref{EQ1}) and (\ref{EQ2}), we have for any $h \in [H]$
    \begin{align}
         V_h^{\pi^\star, \lambda}(s) - \hat{V}^\lambda_h(s) \leq \EE_{s' \sim \hat{P}_h(\cdot |s,\pi_h^\star(s))}[\hat{V}^\lambda_{h+1}(s')-V^{\star, \lambda}_{h+1}(s')] + \delta^\lambda_h(s,\pi_h^\star(s)). \label{EQ3}
    \end{align} 
    Recursively applying (\ref{EQ3}), we have
    \begin{align*}
         \text{(i)} = V_1^{\pi^\star, \lambda}(s) - \hat{V}^\lambda_1(s) \leq  \sum_{h=1}^H \EE^{\pi^\star, \hat{P}}\big[\delta^\lambda_h(s_h,a_h)\|s_1 = s\big].
    \end{align*}
    Next we bound term (ii), similar to term (i), by \eqref{def: regularized bellman operator}, the error can be decomposed to 
    \begin{align}
        \hat{V}^\lambda_h(s) - V_h^{\hat{\pi}, \lambda}(s) &= \hat{Q}_h^{\lambda}(s, \hat{\pi}_h(s)) - Q_h^{\hat{\pi}, \lambda}(s, \hat{\pi}_h(s))\nonumber \\
        & = \mathcal{T}_h^{\lambda} \hat{V}^\lambda_{h+1}(s, \hat{\pi}_h(s)) - \delta_h^\lambda(s, \hat{\pi}_h(s)) - \mathcal{T}_h^{\lambda} V_{h+1}^{\hat{\pi}, \lambda}(s, \hat{\pi}_h(s)). \label{EQ4}
    \end{align}
    Denote $P^{\hat{\pi}} = \{ P^{\hat{\pi}}_h \}_{h \in [H]}$ where $P^{\hat{\pi}}_h$ is defined as: $\forall (s,a) \in \cS \times \cA$,
    \begin{align*}
        P^{\hat{\pi}}_h(\cdot |s,a) =  \argmin_{\bmu_h \in \Delta(\cS)^d,P_h=\la \bphi,\bmu_h \ra}\big[\EE_{s'\sim P_h(\cdot|s,a)}\big[\hat{V}_{h+1}^{\hat{\pi}}(s')\big]+ \lambda \la\bphi(s, a),\bD(\bmu_h||\bmu_h^0) \ra\big]. 
    \end{align*}
    Hence similar to the bound in (\ref{EQ2}), the difference between the regularized Bellman operator and the empirical regularized Bellman operator can be bounded as 
    \begin{align}
        \mathcal{T}_h^{\lambda} \hat{V}^\lambda_{h+1}(s, \hat{\pi}_h(s))- \mathcal{T}_h^{\lambda} V_{h+1}^{\hat{\pi}, \lambda}(s, \hat{\pi}_h(s)) \leq \EE_{s' \sim P^{\hat{\pi}}_h(\cdot |s,\hat{\pi}_h(s))}[\hat{V}^\lambda_{h+1}(s')-V^{\hat{\pi}, \lambda}_{h+1}(s')]. \label{EQ5}
    \end{align}
    Combining inequality (\ref{EQ4}), (\ref{EQ5}), we have for any $h \in [H]$
    \begin{align}
        \hat{V}^\lambda_h(s) - V^{\hat{\pi}, \lambda}_h(s) \leq \EE_{s' \sim P^{\hat{\pi}}_h(\cdot |s,\hat{\pi}_h(s))}[\hat{V}^\lambda_{h+1}(s')-V^{\hat{\pi}, \lambda}_{h+1}(s')] - \delta_h^\lambda(s, \hat{\pi}_h(s)). \label{EQ6}
    \end{align}
    Recursively applying (\ref{EQ6}), we have the "pessimisim" of the estimated value function that $\forall h \in [H]$
    \begin{align*}
     \hat{V}^\lambda_1(s) - V^{\hat{\pi}, \lambda}_1(s) \leq \sum_{h=1}^H \EE^{\hat{\pi}, P^{\hat{\pi}}}\big[-\delta^\lambda_h(s_h,a_h)\|s_1 = s\big].     
    \end{align*}
    Therefore combining the two bounds above, we have 
    \begin{align}
     \text{SubOpt}(\hat{\pi}, s, \lambda) = \text{(i)} + \text{(ii)} \leq   \sum_{h=1}^H \EE^{\pi^\star, \hat{P}}\big[\delta^\lambda_h(s_h,a_h)|s_1 = s\big] + \sum_{h=1}^H \EE^{\hat{\pi}, P^{\hat{\pi}}}\big[-\delta^\lambda_h(s_h,a_h)\|s_1 = s\big].
     \label{upper bound for subopt in prep}
    \end{align}
    Hence, it requires to estimate the range of the regularized Bellman update error 
    $\delta^\lambda_h(s,a)$. Recall the definition in \eqref{definition of regularized bellman error}, we claim that
    \begin{align}
        0 \leq \delta^\lambda_h(s,a) \leq 2\Gamma_h(s,a) \label{bound of bellman error}
    \end{align}
    holds for $\forall (s,a,h) \in \cS \times \cA \times [H]$. For the LHS of \eqref{bound of bellman error}, first we notice that if $\la \bphi(s,a), \hat{\bw}_h^\lambda \ra - \Gamma_h(s, a)\leq 0$, the inequality holds trivially as $\hat{Q}_h^\lambda(s,a)= 0$. Next we consider the case where $\la \bphi(s,a), \hat{\bw}_h^\lambda \ra - \Gamma_h(s, a)\geq 0$. By the definition of $\hat{Q}_h^\lambda(s,a)$ and 
     the assumption in the lemma, we have
     \begin{align*}
         \delta^\lambda_h(s,a) &= \mathcal{T}_h^{\lambda} \hat{V}_{h+1}^{\lambda}(s, a) - \hat{Q}_h^\lambda(s,a) \\
         & =  \mathcal{T}_h^{\lambda} \hat{V}_{h+1}^{\lambda}(s, a) - \min \big \{\la \bphi(s,a), \hat{\bw}_h^\lambda \ra- \Gamma_h(s,a), H-h+1 \big \} \\
         & \geq \mathcal{T}_h^{\lambda} \hat{V}_{h+1}^{\lambda}(s, a) - \la \bphi(s,a), \hat{\bw}_h^\lambda \ra + \Gamma_h(s,a) \\
         & \geq 0.
     \end{align*}
     On the other hand, by the assumption in the lemma, we have 
     \begin{align*}
         \la \bphi(s,a), \hat{\bw}_h^\lambda \ra- \Gamma_h(s,a) \leq \mathcal{T}_h^{\lambda} \hat{V}_{h+1}^{\lambda}(s, a) \leq H-h+1.
     \end{align*}
     Hence, we can upper bound $\delta_h^\lambda(s,a)$ as 
     \begin{align*}
                  \delta^\lambda_h(s,a) &= \mathcal{T}_h^{\lambda} \hat{V}_{h+1}^{\lambda}(s, a) - \hat{Q}_h^\lambda(s,a) \\
         & =  \mathcal{T}_h^{\lambda} \hat{V}_{h+1}^{\lambda}(s, a) - \max \big \{\la \bphi(s,a), \hat{\bw}_h^\lambda \ra- \Gamma_h(s,a), 0 \big \} \\
         & \leq \mathcal{T}_h^{\lambda} \hat{V}_{h+1}^{\lambda}(s, a) - \la \bphi(s,a), \hat{\bw}_h^\lambda \ra +\Gamma_h(s,a) \\
         & \leq 2\Gamma_h(s,a).
     \end{align*}
    This concludes the claim. Now it remains to bound the empirical transition kernel $\hat{P}$. Noticing the fact that $ \forall h \in [H], (s,a) \in \cS \times \cA,$ 
    \begin{align}
        \lambda D(\hat{\mu}_{h,i} \|\mu^0_{h,i}) &\leq  \EE_{s' \sim \hat{\mu}_{h,i}}[\hat{V}^\lambda_{h+1}(s')]+ \lambda D(\hat{\mu}_{h,i} \|\mu^0_{h,i}) \nonumber\\
        & = \inf_{\mu_{h, i} \in \Delta(\cS)}\big[\EE_{s' \sim \mu_{h, i}}[\hat{V}_{h+1}(s')]+ \lambda D(\mu_{h, i} \|\mu_{h, i}^0)\big] \nonumber\\
        & \leq
        \inf_{\mu_{h, i} \in \Delta(\cS)}\big[\EE_{s' \sim \mu_{h, i}}[V_{h+1}^{\star, \lambda}(s')]+ \lambda D(\mu_{h, i} \|\mu^0_{h,i})\big] \label{D2.1}\\
        & \leq \EE_{s' \sim \mu_{h, i}^0}[V_{h+1}^{\star, \lambda}(s')] \nonumber
        \\ & \leq \max_{s \in \cS} V_{h+1}^{\star, \lambda}(s), \nonumber
    \end{align}
    where \eqref{D2.1} comes from the pessimism of value function, i.e $\hat{V}^\lambda_{h+1}(s) \leq V^{\star, \lambda}_h(s), \forall h \in [H]$. Hence, the empirical transition kernel $\hat{P}_h(\cdot | s,a)$ is contained in the set $\mathcal{U}^\lambda(P^0)$ defined in \eqref{def:robustly admissible set}. Hence, by \eqref{upper bound for subopt in prep} and \eqref{bound of bellman error}, we have
    \begin{align*}
            \text{SubOpt}(\hat{\pi},s,\lambda) &\leq  \sum_{h=1}^H \EE^{\pi^\star, \hat{P}}\big[\delta^\lambda_h(s_h,a_h)|s_1 = s\big] + \sum_{h=1}^H \EE^{\hat{\pi}, P^{\hat{\pi}}}\big[-\delta^\lambda_h(s_h,a_h)\|s_1 = s\big] \\
            & \leq  2 \sum_{h=1}^{H}\EE^{\pi^\star, \hat{P}}\big[\Gamma_h(s_h, a_h)|s_1 = s\big] \\
            &\leq 2 \sup_{P \in \mathcal{U}^\lambda(P^0)} \sum_{h=1}^{H}\EE^{\pi^\star, P}\big[\Gamma_h(s_h, a_h)|s_1 = s\big].
    \end{align*}
    This concludes the proof.
    \end{proof}
\subsection{Proof of \Cref{bound of weights in TV distance}}
\begin{proof}
For all $h \in [H]$, from the definition of $w_h$,     
    \begin{align*}
        \|\bw_h^\lambda\|_2 =  \|\EE_{s \sim \bmu^0_h}[\hat{V}^\lambda_{h+1}(s)]_{\alpha_{h+1}}\|_2 \leq H\sqrt{d},
    \end{align*}
    where the inequality follows from the fact that $\hat{V}_{h+1}^\lambda \leq H$, for all $h \in [H]$. Meanwhile, by the definition of $\hat{\bw}_h^{\lambda}$ in \Cref{alg:R2PVI-TV}, and the triangle inequality, 
    \begin{align}
        \|\hat{w}_h^{\lambda}\|_2 &= \Big\|\bLambda_h^{-1}\sum_{\tau=1}^K \bphi(s_h^{\tau}, a_h^{\tau})[\hat{V}^\lambda_{h+1}(s)]_{\alpha_{h+1}}\Big\|_2 \nonumber\\
        &\leq H\sum_{\tau=1}^K\|\bLambda_h^{-1}\bphi(s_h^{\tau}, a_h^{\tau})\|_2 \nonumber\\
        &= H\sum_{\tau = 1}^K \sqrt{\bphi(s_h^{\tau}, a_h^{\tau})^{\top} \bLambda_h^{-1/2} \bLambda_h^{-1} \bLambda_h^{-1/2} \bphi(s_h^{\tau}, a_h^{\tau})} \nonumber\\
        &\leq \frac{H}{\sqrt{\gamma}} \sum_{\tau=1}^{K}\sqrt{\bphi(s_h^{\tau}, a_h^{\tau})^{\top}\bLambda_h^{-1}\bphi(s_h^{\tau}, a_h^{\tau})} \label{support3}\\
        &\leq \frac{H\sqrt{K}}{\sqrt{\gamma}}
        \sqrt{\sum_{\tau=1}^{K}\bphi(s_h^{\tau}, a_h^{\tau})^{\top}\bLambda_h^{-1}\bphi(s_h^{\tau}, a_h^{\tau})} \label{support4}\\
        &=\frac{H\sqrt{K}}{\sqrt{\gamma}}
        \sqrt{\text{Tr}(\bLambda_h^{-1}(\bLambda_h - \gamma \mathbf{I}))} \nonumber\\
        &\leq \frac{H\sqrt{K}}{\sqrt{\gamma}}\sqrt{\text{Tr}(\mathbf{I})} \nonumber\\
        &
        =H\sqrt{\frac{Kd}{\gamma}},\nonumber
    \end{align}
    where \eqref{support3} follows from the fact that $\|\bLambda_h^{-1}\| \leq \gamma^{-1}$, \eqref{support4} follows from the Cauchy-Schwartz inequality. Then we conclude the proof.
\end{proof}
    \subsection{Proof of \Cref{lemma:bound of weights in KL distance}}
\begin{proof}
    By definition, we have
    \begin{align*}
       \big \|\bw_h^\lambda\|_2 = \Big\|\EE_{s \sim \bmu^{0}_{h}}\Big[e^{-\frac{\hat{V}^\lambda_{h+1}(s)}{\lambda}}\Big]\Big\|_2 = \int _{\mathcal{S}}\Big\|e^{-\frac{\hat{V}^\lambda_{h+1}(s)}{\lambda}}\bmu_h^0(s)\Big\|_2ds \leq \int_{\mathcal{S}}\|\bmu_h^0(s)\|_2ds \leq \sqrt{d},
    \end{align*}
    this concludes the proof of $\bw_h^\lambda$. For $\hat{\bw}_h^{\lambda}$,
    \begin{align}
        \|\hat{\bw}_h^{\lambda}\|_2 &= \Big\|\bLambda_h^{-1}\sum_{\tau=1}^K \bphi(s_h^{\tau}, a_h^{\tau})e^{-\frac{\hat{V}^\lambda_{h+1}(s_{h+1})}{\lambda}}\Big\|_2 \nonumber\\
        &\leq \sum_{\tau=1}^K\|\bLambda_h^{-1}\bphi(s_h^{\tau}, a_h^{\tau})\|_2 \nonumber\\
        & = \sum_{\tau = 1}^K \sqrt{\bphi(s_h^{\tau}, a_h^{\tau})^{\top} \bLambda_h^{-1/2} \bLambda_h^{-1} \bLambda_h^{-1/2} \bphi(s_h^{\tau}, a_h^{\tau})} \nonumber\\
        &\leq \frac{1}{\sqrt{\gamma}} \sum_{\tau=1}^{K}\sqrt{\bphi(s_h^{\tau}, a_h^{\tau})^{\top}\bLambda_h^{-1}\bphi(s_h^{\tau}, a_h^{\tau})} \label{support1}\\
        &\leq \frac{\sqrt{K}}{\sqrt{\gamma}}
    \sqrt{\sum_{\tau=1}^{K}\bphi(s_h^{\tau}, a_h^{\tau})^{\top}\bLambda_h^{-1}\bphi(s_h^{\tau}, a_h^{\tau})} \label{support2}
        \\
        &=\frac{\sqrt{K}}{\sqrt{\gamma}}
        \sqrt{\text{Tr}(\bLambda_h^{-1}(\bLambda_h -\gamma \mathbf{I}))} \nonumber\\
        &\leq \frac{\sqrt{K}}{\sqrt{\gamma}}\sqrt{\text{Tr}(\mathbf{I})}
        =\sqrt{\frac{Kd}{\gamma}}, \nonumber
    \end{align}
    \noindent
    where \eqref{support1} follows from the fact that $\|\bLambda_h^{-1}\| \leq \gamma^{-1}$, \eqref{support2} follows from the Cauchy-Schwartz inequality. Then we conclude the proof.
\end{proof}

\subsection{Proof of \Cref{lemma: covering number of function class in KL}}
\begin{proof}
Denote $\bA = \beta^2 \Lambda_h^{-1}$, then we have $\|\btheta\|_2 \leq L, \|\bA\|_2 \leq B^2 \gamma^{-1}$. For any two functions $V_1, V_2 \in \mathcal{V}$ with parameters $(\btheta_1, \bA_1), (\btheta_2, \bA_2)$, since both $\{ \cdot \}_{[0, H-h+1]}$ and $\max_a$ are contraction maps,
     \begin{align}
        &\text{dist}(V_1, V_2) \\
        &\leq \sup_{s,a}\Big|\bphi(s,a)^{\top}(\btheta_1 - \btheta_2) - \lambda \Big(\log\Big(1+ \sum_{i=1}^d\|\phi_i(s,a) \mathbf{1}_i\|_{\bA_1}\Big) - \log\Big(1+ \sum_{i=1}^d\|\phi_i(s,a) \mathbf{1}_i\|_{\bA_2}\Big) \Big) \Big|\nonumber \\
        &\leq \sup_{\bphi \in \RR^d, \|\bphi\| \leq 1}\Big|\bphi^{\top}(\btheta_1 - \btheta_2) - \lambda \log \frac{1+ \sum_{i=1}^d\|\phi_i(s,a) \mathbf{1}_i\|_{\bA_1}} {1+ \sum_{i=1}^d\|\phi_i(s,a) \mathbf{1}_i\|_{\bA_2}}\Big| \nonumber\\
        &\leq \sup_{\bphi \in \RR^d: \|\bphi\| \leq 1}|\bphi^{\top}(\btheta_1 - \btheta_2)| +\lambda \sup_{\bphi \in \RR^d: \|\bphi\| \leq 1} \Big|\log \frac{1+ \sum_{i=1}^d\|\phi_i(s,a) \mathbf{1}_i\|_{\bA_1}} {1+ \sum_{i=1}^d\|\phi_i(s,a) \mathbf{1}_i\|_{\bA_2}}\Big| , \label{G1}
     \end{align}
     we notice the fact that: for any $x >0, y>0$, 
     \begin{align*}
         \Big|\log\frac{1+x}{1+y}\Big| = \Big|\log\Big(\frac{x-y}{1+y} + 1\Big)\Big| \leq \log(|x-y| + 1)  \leq |x-y|.
     \end{align*}
     \noindent
     Therefore, \eqref{G1} can be bounded as:
     \begin{align}
        \eqref{G1} &\leq \sup_{\bphi \in \RR^d: \|\bphi\| \leq 1}|\bphi^{\top}(\btheta_1 - \btheta_2)| +\lambda \sup_{\bphi \in \RR^d: \|\bphi\| \leq 1} \Big|\sum_{i =1 }^d\|\phi_i(s,a) \mathbf{1}_i\|_{\bA_1} - \sum_{i =1 }^d\|\phi_i(s,a) \mathbf{1}_i\|_{\bA_2}\Big| \nonumber\\
        & = \sup_{\bphi \in \RR^d: \|\bphi\| \leq 1}|\bphi^{\top}(\btheta_1 - \btheta_2)| +\lambda \sup_{\bphi \in \RR^d: \|\bphi\| \leq 1} \Big|\sum_{i =1 }^d\sqrt{\phi_i\mathbf{1}_i^{\top}\bA_1\phi_i\mathbf{1}_i} - \sum_{i =1 }^d\sqrt{\phi_i\mathbf{1}_i^{\top}\bA_2\phi_i\mathbf{1}_i}\Big| \nonumber \\
        &\leq \|\btheta_1 - \btheta_2\|_2 + \lambda \sup_{\bphi \in \RR^d:\|\bphi\| \leq 1}\sum_{i=1}^d\sqrt{\phi_i\mathbf{1}_i^{\top}(\bA_1-\bA_2)\phi_i \mathbf{1}_i} \label{G2} \\
        &\leq \|\btheta_1 - \btheta_2\|_2 + \lambda \sqrt{\|\bA_1 - \bA_2\|}\sup_{\bphi \in \RR^d:\|\bphi\| \leq 1}\sum_{i=1}^d \|\phi_i \mathbf{1}_i\| \nonumber \\
        &\leq \|\btheta_1 - \btheta_2\|_2 + \lambda \sqrt{\|\bA_1 - \bA_2\|_{F}} \label{G3},
     \end{align}
     \noindent
     where the \eqref{G2} follows from the triangular inequality and the fact $|\sqrt{x}- \sqrt{y}| \leq \sqrt{|x-y|}$, and $\|\cdot\|_F$ denotes the Frobenius norm. We next define that $\mathcal{C_{\btheta}}$ is an $\epsilon/2$-cover of $\{ \btheta \in \RR^d| \|\btheta\|_2 \leq L \}$, and the $\mathcal{C}_{\bA}$ is an $\epsilon^2/{4\lambda^2}$-cover of $\{ \bA \in \RR^{d \times d}| \|\bA\|_F \leq d^{1/2}B^2\gamma^{-1} \}$. By \Cref{lemma: coverage number of euclidean ball}, we have that:
     \begin{align*}
         |\mathcal{C_{\btheta}}| \leq (1+ 4L/\epsilon)^d, |\mathcal{C}_{\bA}| \leq (1+ 8\lambda^2d^{1/2}B^2/\gamma \epsilon^2)^{d^2}.
     \end{align*}
     \noindent
     By (\ref{G3}), for any $V_1 \in \mathcal{V}$, there exists $\btheta_2 \in \mathcal{C}_{\btheta}$ and $\bA_2 \in \mathcal{C}_{\bA}$ s.t $V_2$ parametrized by $(\btheta_2, \bA_2)$ satisfying $\text{dist}(V_1, V_2) \leq \epsilon$. Therefore, we have the following:
     \begin{align*}
         \log|\mathcal{N}(\epsilon)| \leq \log|\mathcal{C_{\btheta}}| + \log|\mathcal{C}_{\bA}| \leq d\log(1+4L/\epsilon) + d^2\log(1+8\lambda^2d^{1/2}B^2/\gamma \epsilon^2).
     \end{align*}
     Hence we conclude the proof.
\end{proof}
\subsection{Proof of \Cref{lemma:bound of weights in chi2 distance}} 
\begin{proof}
By definition, we have that 
\begin{align}
            &\|\hat{\bw}_h^{\lambda}\|_2\notag\\
            &= \Big\|\Big[\max_{\alpha \in [(\hat{V}_{h+1}^{\lambda})_{\min},(\hat{V}_{h+1}^{\lambda})_{\max}]} \Big \{ \hat{\EE}^{\mu_{h,i}^0}[\hat{V}_{h+1}^{\lambda}(s)]_{\alpha} + \frac{1}{4\lambda}(\hat{\EE}^{\mu_{h,i}^0}[\hat{V}_{h+1}^{\lambda}(s)]_{\alpha})^2- \frac{1}{4\lambda}\hat{\EE}^{\mu_{h,i}^0}[\hat{V}_{h+1}^{\lambda}(s)]^2_{\alpha} \Big \}\Big]_{i \in [d]}\Big\|_2 \nonumber\\
            & \leq \Big\|\Big[H + \frac{H^2}{2\lambda}\Big]_{i \in [d]}\Big\|_2 \label{chi1}
             \\&= \sqrt{d}\Big(H + \frac{H^2}{2\lambda}\Big),\nonumber
         \end{align}
         \noindent
         where \eqref{chi1} follows by the fact that $\hat{\EE}^{\mu_{h,i}^0}[\hat{V}_{h+1}^{\lambda}(s)]_{\alpha} \in [0,H], \hat{\EE}^{\mu_{h,i}^0}[\hat{V}_{h+1}^{\lambda}(s)]^2_{\alpha} \in [0, H^2]$.
\end{proof}

\subsection{Proof of \Cref{lemma: bounding robust value function - hard instance}}
\begin{proof}
        We first proof the LHS of the lemma by induction from last stage $H$. From the definition of $V^{\pi, \lambda}_H$ and $\btheta_h$, we can learn that
        \begin{align*}
            V_H^{\pi, \lambda}(s_1) = r_H(s_1, \pi_H(s_1)) = \bphi(s_1, \pi(s_1))^{\top}\btheta_h = \frac{\delta}{2d}\Big(d + \sum_{i=1}^d\xi_{Hi}\EE^{\pi}a_{Hi}\Big).
        \end{align*}
        This is the base case. Now suppose the conclusion holds for stage $h+1$, that is to say, 
        \begin{align*}
            V^{\pi, \lambda}_{h+1}(s_1) \leq \frac{\delta}{2d}\sum_{j=h+1}^H ( 1- \epsilon)^{j-h-1} \Big( d + \Big(\sum_{i=1}^d\xi_{ji}\EE^{\pi}a_{ji}\Big)\Big). 
        \end{align*}
        Recall the regularized robust bellman equation in \Cref{prop:regularized Robust Bellman equation} and the regularized duality of the three divergences, we have 
        \begin{align}
            Q_{h}^{\pi, \lambda}(s_1, a) &= r_h(s_1, a) + \inf_{\bmu_{h} \in \Delta(\cS)^{d+2},P_h=\la \bphi,\bmu_h \ra}\big[\EE_{s'\sim P_h(\cdot|s,a)}\big[V_{h+1}^{\pi, \lambda}(s')\big]+ \lambda \la\bphi(s, a),\bD(\bmu_h||\bmu_h^0) \ra\big] \nonumber \\
            & \leq r_h(s_1,a) +\EE_{s' \sim P_h^0(\cdot |s_1,a)}[V_{h+1}^{\pi, \lambda}(s')] \\
            &= r_h(s_1,a) + (1-\epsilon) V_{h+1}^{\pi, \lambda}(s_1).
        \end{align}
        Then with regularized robust bellman equation in \Cref{prop:regularized Robust Bellman equation} and the inductive hypothesis, we have 
        \begin{align}
            V_{h}^{\pi, \lambda}(s_1) &= Q_h^{\pi, \lambda}(s_1, \pi(s_1))  \nonumber\\
            &\leq r_h(s_1, \pi_h(s_1)) + (1-\epsilon) V_{h+1}^{\pi, \lambda}(s_1) \label{for TV}
            \\& =\frac{\delta}{2d}\Big(d + \sum_{i=1}^d\xi_{hi}\EE^{\pi}a_{hi}\Big) + \frac{\delta}{2d}\sum_{j=h+1}^H ( 1- \epsilon)^{j-h} \Big( d + \Big(\sum_{i=1}^d\xi_{ji}\EE^{\pi}a_{ji}\Big)\Big) \nonumber\\
            & = \frac{\delta}{2d}\sum_{j=h}^H ( 1- \epsilon)^{j-h} \Big( d + \Big(\sum_{i=1}^d\xi_{ji}\EE^{\pi}a_{ji}\Big)\Big). \nonumber
        \end{align}
        Hence, by the induction argument, we conclude the proof of the RHS. Furthermore, for any $h \in [H]$, we can upper bound $V_{h}^{\pi, \lambda}(s)$ as
        \begin{align}
        V_{h}^{\pi, \lambda}(s) \leq \frac{\delta}{2d}\sum_{j=h}^H ( 1- \epsilon)^{j-h} \Big( d + \Big(\sum_{i=1}^d\xi_{ji}\EE^{\pi}a_{ji}\Big)\Big) \leq \delta(H-h) \leq \lambda (H-h)/H \leq \lambda,\label{ineq: upper bound for V}
        \end{align}
        where the third inequality holds by the definition of $\delta$. For the left, we prove by discussing the KL, $\chi^2$ and TV cases respectively. 
        \paragraph{Case I - TV.}The case for TV holds trivially as by \Cref{prop:the regularized duality under TV divergence}, we have
        \begin{align}
            Q_{h}^{\pi, \lambda}(s_1, a) &= r_h(s_1, a) + \inf_{\bmu_{h} \in \Delta(\cS)^{d+2},P_h=\la \bphi,\bmu_h \ra}\big[\EE_{s'\sim P_h(\cdot|s,a)}\big[V_{h+1}^{\pi, \lambda}(s')\big]+ \lambda \la\bphi(s, a),\bD(\bmu_h||\bmu_h^0) \ra\big] \nonumber \\
            & =r_h(s_1, a) +  \la \bphi(s_1,a),  \EE_{s' \sim \bmu_h^0}[V_{h+1}^{\pi, \lambda}(s')]_{\min_{s'}(V_{h+1}^{\pi, \lambda}(s')) + \lambda} \ra \\
            & = r_h(s_1, a) + \EE_{s' \sim P_h^0(\cdot |s,a)}[V_{h+1}^{\pi, \lambda}(s')]_{\min_{s'}(V_{h+1}^{\pi, \lambda}(s'))+ \lambda} \nonumber\\
            & = r_h(s_1,a) + (1-\epsilon) V_{h+1}^{\pi, \lambda}(s_1),\label{lower5}
        \end{align}
        where \eqref{lower5} holds by (\ref{ineq: upper bound for V}). Hence, the inequality in (\ref{for TV}) holds for equality. This concludes the proof for TV-divergence.
        \paragraph{Case II - KL.}  We prove by induction. The case holds trivially in last stage $H$. Suppose         
        \begin{align*}
            V^{\pi, \lambda}_{h+1}(s_1) \geq \frac{\delta}{2d}\sum_{j=h+1}^H ( 1- \epsilon)^{j-h-1} \Big( d + \Big(\sum_{i=1}^d\xi_{ji}\EE^{\pi}a_{ji}\Big)\Big)- (H-h)\lambda \epsilon (e-1). 
        \end{align*}
        Recall the duality form of \Cref{prop:the regularized duality under KL divergence}, the Q-function at stage $h$ can be upper bounded as:
        \begin{align}
            Q_{h}^{\pi, \lambda}(s_1, a) &= r_h(s_1, a) + \inf_{\bmu_{h} \in \Delta(\cS)^{d+2},P_h=\la \bphi,\bmu_h \ra}\big[\EE_{s'\sim P_h(\cdot|s,a)}\big[V_{h+1}^{\pi, \lambda}(s')\big]+ \lambda \la\bphi(s, a),\bD(\bmu_h||\bmu_h^0) \ra\big] \nonumber \\
            & = r_h(s_1,a) + \la \bphi(s_1,a), - \lambda \log \EE_{s' \sim \bmu_h^0} e^{-V_{h+1}^{\pi, \lambda}(s')/\lambda} \ra \nonumber\\
            & = r_h(s_1,a) - \lambda \log \big(\epsilon + (1-\epsilon)e^{-V_{h+1}^{\pi, \lambda}(s_1)/\lambda }\big) \nonumber\\
            & = r_h(s_1,a) + V_{h+1}^{\pi, \lambda}(s_1) - \lambda \log \big(\epsilon e^{V_{h+1}^{\pi, \lambda}(s_1)/\lambda} + (1-\epsilon)\big) \nonumber\\
            &\geq r_h(s_1,a) + V_{h+1}^{\pi, \lambda}(s_1) -\lambda \epsilon \big(e^{V_{h+1}^{\pi, \lambda}(s_1)/\lambda}-1\big) \label{lower3}\\
            & \geq r_h(s_1,a) + V_{h+1}^{\pi, \lambda}(s_1) - \lambda \epsilon (e - 1),
            \label{lower4}
        \end{align}        
        where \eqref{lower3} follows by the fact that $\log(1+x) \leq x, \forall x >0$, \eqref{lower4} follows by (\ref{ineq: upper bound for V}). Therefore, by the inductive hypothesis, we have
        \begin{align*}
            V_{h}^{\pi, \lambda}(s_1) &= Q_h^{\pi, \lambda}(s_1, \pi(s_1)) \\
            &\geq r_h(s_1, \pi_h(s_1)) + (1-\epsilon) V_{h+1}^{\pi, \lambda}(s_1) - \lambda \epsilon (e-1) 
            \\& = \frac{\delta}{2d}\sum_{j=h}^H ( 1- \epsilon)^{j-h} \Big( d + \Big(\sum_{i=1}^d\xi_{ji}\EE^{\pi}a_{ji}\Big)\Big) - (H-h) \lambda \epsilon (e-1).
        \end{align*}
        This finishes the KL setting. 
        \paragraph{Case III - $\chi^2$.}  Similar to the case in TV, KL, by the duality of $\chi^2$ in \Cref{prop:the regularized duality under x2 divergence}, we have
        \begin{align}
            Q_{h}^{\pi, \lambda}(s_1, a) &= r_h(s_1, a) + \inf_{\bmu_{h} \in \Delta(\cS)^{d+2},P_h=\la \bphi,\bmu_h \ra}\big[\EE_{s'\sim P_h(\cdot|s,a)}\big[V_{h+1}^{\pi, \lambda}(s')\big]+ \lambda \la\bphi(s, a),\bD(\bmu_h||\bmu_h^0) \ra\big] \nonumber \\
            & = r_h(s_1, a) + (1-\epsilon)\sup_{\alpha \in [V_{\min}, V_{\max}]} \Big \{ [V_{h+1}^{\pi, \lambda}(s_1)]_{\alpha} - \frac{\epsilon}{4\lambda}[V_{h+1}^{\pi, \lambda}(s_1)]_{\alpha}^2 \Big\} \nonumber\\
            &\geq r_h(s_1, a) + (1-\epsilon)\Big[V_{h+1}^{\pi, \lambda}(s_1) - \frac{\epsilon}{4\lambda}[V_{h+1}^{\pi, \lambda}(s_1)]^2 \Big] \nonumber\\
            & \geq r_h(s_1, a) + (1-\epsilon)V_{h+1}^{\pi, \lambda}(s_1) - \frac{\epsilon \lambda(1-\epsilon)}{4}, \label{lower2}
        \end{align}
        where \eqref{lower2} follows by (\ref{ineq: upper bound for V}). Hence, similar to \textbf{Case II}, by induction, we have
        \begin{align*}
            V_{h}^{\pi, \lambda}(s_1) \geq \frac{\delta}{2d}\sum_{j=h}^H ( 1- \epsilon)^{j-h} \Big( d + \Big(\sum_{i=1}^d\xi_{ji}\EE^{\pi}a_{ji}\Big)\Big) - (H-h) \frac{\epsilon \lambda(1-\epsilon)}{4}.
        \end{align*}
        This finishes the $\chi^2$ setting, and we complete the proof.
    \end{proof}

\section{Auxiliary Lemmas}
\begin{lemma}[Lemma D.3 of \citet{liu2024distributionally}]
        \label{lemma:covering number of function class}
    For any $h \in [H]$, let $\mathcal{V}_h$ denote a class of functions mapping from $\mathcal{S}$ to $\RR$ with the following form:
     \begin{align*}
        V_h(x; \btheta, \beta, \bLambda_h) = \max_{a \in \mathcal{A}} \Big\{ \bphi(s,a)^{\top} \btheta - \beta \sum_{i=1}^d\|\phi_i(\cdot, \cdot) \mathbf{1}_i\|_{\bLambda_h^{-1}} \Big\}_{[0, H-h+1]},
     \end{align*}
     \noindent
    the parameters $(\btheta, \beta, \bLambda_h)$ satisfy $\|\btheta\|_2 \leq L, \beta \in [0, B], \gamma_{\min}(\bLambda_h) \geq \gamma$. Let $\mathcal{N}_h(\epsilon)$ be the $\epsilon$-covering number of $\mathcal{V}$ with respect to the distance $\text{dist}(V_1, V_2) = \sup_x|V_1(x) - V_2(x)|$. Then
     \begin{align*}
        \log\mathcal{N}_h(\epsilon) \leq d\log(1+4L/\epsilon) + d^2\log(1+8d^{1/2}B^2/\gamma \epsilon^2).
     \end{align*}
 \end{lemma}
 \begin{lemma}[Corollary 4.2.11 of \citet{vershynin2018high}]
     \label{lemma:covering number of interval}
    Denote the $\epsilon$-covering number of the closed interval $[a, b]$ for some real number $b > a$ with respect to the distance metric $d(\alpha_1, \alpha_2) = |\alpha_1 - \alpha_2|$ as $\mathcal{N}_{\epsilon}([a,b])$, then we have $\mathcal{N}_{\epsilon}([a,b]) \leq 3(b-a)/{\epsilon}$.
 \end{lemma}

 \begin{lemma}[Lemma B.2 of \citet{jin2021pessimism}]
 \label{lemma:concentration bound}
    Let $f:\mathcal{S} \rightarrow [0, R-1]$ be any fixed function. For any $\delta \in (0, 1)$, we have
     \begin{align*}
        P\bigg(\|\sum_{\tau = 1}^{K}\bphi(s_h^{\tau}, a_h^{\tau})\eta_h^{\tau}(f)\|_{\Lambda_h^{-1}}^2 \geq R^2 (2\log(1/\delta) + d\log(1 + K/\gamma))\bigg) \leq \delta,
     \end{align*}
     where $\eta_h^{\tau}(f)= \EE_{s' \sim P_h^0(\cdot |s^{\tau}_h, a_h^{\tau})}[f(s')] - f(s_{h+1}^{\tau}) $.
 \end{lemma}

 \begin{lemma}[Lemma F.3 of \citet{liu2024minimax}]
        \label{lemma: bound of dataset}
        If $K \geq \mathcal{O}(d^6)$ and the feature map is define as \Cref{sec:Construction of Hard Instances}, then with probability at least $1- \delta$, we have for any transition P,
        \begin{align*}
            \sum_{h=1}^H\EE^{\pi^\star, P}\Big[\sum_{i=1}^d \|\phi_i(s,a)\mathbf{1}_i\|_{{\Lambda}_h^{-1}}|\mathbf{s}_1 = s_1\Big] \leq \frac{4d^{3/2}H}{\sqrt{K}}.
        \end{align*}
    \end{lemma}

    \begin{lemma}[Lemma 5.2 of \citet{vershynin2010introduction}]
        \label{lemma: coverage number of euclidean ball}
        For any $\epsilon > 0$, the $\epsilon$ -covering number of the Euclidean ball in $\RR^d$ with radius $R >0$ is upper bounded by $(1+2R/{\epsilon})^d$.
    \end{lemma}

\end{document}